\documentclass[10pt]{article}
\usepackage{balance}
\usepackage[top=2.5cm, bottom=2.9cm, left=2.5cm,
right=2.5cm]{geometry}
\usepackage{dblfloatfix}
\usepackage[export]{adjustbox} 
\usepackage{microtype}
\usepackage{graphicx} % Required for inserting images
\usepackage{pgfplots}
\pgfplotsset{compat=1.18}
\usepackage{tikz}
\usetikzlibrary{calc}
\usepgfplotslibrary{groupplots}

\usepackage{subcaption}
\usepackage{booktabs} % for professional tables
\usepackage{placeins}
\usepackage{hyperref}
\usepackage{bm}
\usepackage{enumitem,amsthm,amsmath,amsfonts,amssymb}
\usepackage{indentfirst}
\usepackage{float}
\usepackage{natbib}
\usepackage{mathtools}
\usepackage{algorithmic}
\usepackage{algorithm}
\usepackage{xcolor}
\usepackage{graphicx}
\usepackage{subcaption}
\usepackage{tikz}
\usepackage{multirow}
\usepackage{multicol}
\usepackage{hyperref}
\usepackage{soul}
\usepackage{mathrsfs, fancyhdr}
\usepackage{amscd}
\usepackage{booktabs} 
 
\newcommand{\Z}{\mathbb{Z}}

\newcommand{\N}{\mathbb{N}}

\newcommand{\E}{\mathbb{E}}

\newcommand{\R}{\mathbb{R}}

\def\bw{\mathbf{w}}
\def\O{\mathcal{O}}
\def\gep{\epsilon}

\def\X{\mathcal{X}}
\def\Y{\mathcal{Y}}

\def\bA{\mathbf{A}}
\def\Z{\mathcal{Z}}
\def\N{\mathcal{N}}
\def\L{\mathcal{L}}
\def\bx{\mathbf{x}}

\def\bw{\mathbf{w}}

\def\bu{\mathbf{u}}

\def\priv{\text{priv}}

\def\bb{\mathbf{b}}
\def\A{\mathcal{A}}

\def\gd{\delta}
\def\gep{\epsilon}

\def\bh{\mathbf{h}}

\def\ba{\mathbf{a}}

\def\bu{\mathbf{u}}

\def\0{\mathbf{0}}

\def\bv{\mathbf{v}}

\def\bv{\mathbf{v}}

\def\proj{\text{Proj}}
\def\bB{\mathbf{B}}

\def\bw{\mathbf{w}}

\def\O{\mathcal{O}}

\def\bb{\mathbf{b}}

\def\A{\mathcal{A}}
\def\Z{\mathcal{Z}}
\def\R{\mathbb{R}}

\def\X{\mathcal{X}}

\def\Y{\mathcal{Y}}
\def\Z{\mathcal{Z}}

\def\bD{\mathbf{D}}

\def\bfI{\mathbf{I}}

\def\ebb{\mathbb{E}}

\def\rbb{\mathbb{R}}

\def\bc{\mathbf{c}}

\usepackage{amsmath}
\usepackage{amssymb}
\usepackage{mathtools}
\usepackage{amsthm}
\usepackage{multirow}

\usepackage[capitalize,noabbrev]{cleveref}

\theoremstyle{plain}
\newtheorem{theorem}{Theorem}[section]
\newtheorem{proposition}[theorem]{Proposition}
\newtheorem{lemma}[theorem]{Lemma}
\newtheorem{corollary}[theorem]{Corollary}
\theoremstyle{definition}
\newtheorem{definition}[theorem]{Definition}
\newtheorem{assumption}[theorem]{Assumption}
\theoremstyle{remark}
\newtheorem{remark}[theorem]{Remark}

\usepackage[textsize=tiny]{todonotes}

\begin{document}

\title{Optimization,  Generalization and Differential Privacy Bounds for Gradient Descent on Kolmogorov–Arnold Networks\footnote{Published as a conference paper at ICML 2026}}
\author{Puyu Wang$^1$\quad Junyu Zhou$^2$\quad Philipp Liznerski$^1$ \quad Marius Kloft$^1$ \\ 
\smallskip \\
$^{1}$ RPTU Kaiserslautern-Landau, Kaiserslautern, Germany\\
$^{2}$ Catholic University of Eichstätt-Ingolstadt, Ingolstadt, Germany}

 \date{}

\maketitle

\begin{abstract}
 Kolmogorov--Arnold Networks (KANs) have recently emerged as a structured alternative to standard MLPs, yet a principled theory for their training dynamics, generalization, and privacy properties remains limited.
In this paper, we analyze gradient descent (GD) for training two-layer KANs and derive general bounds that characterize their training dynamics, generalization, and utility under differential privacy (DP).
As a concrete instantiation, we specialize our analysis to logistic loss under an NTK-separable assumption, where we show that \emph{polylogarithmic} network width suffices for GD to achieve an optimization rate of order $1/T$ and a generalization rate of order $1/n$, with $T$ denoting the number of GD iterations and $n$ the sample size. 
In the private setting, we characterize the noise required for $(\epsilon,\delta)$-DP and obtain a utility bound of order $\sqrt{d}/(n\epsilon)$ (with $d$ the input dimension), matching the classical lower bound for general convex Lipschitz problems. 
Our results imply that polylogarithmic width is not only sufficient but also \emph{necessary} under differential privacy, revealing a qualitative gap between non-private (sufficiency only) and private (necessity also emerges) training regimes. 
Experiments further illustrate how these theoretical insights can guide practical choices, including network width selection and early stopping.
\end{abstract}

\bigskip

%%%%%%%%%%%%%%%%%%%%%%%%%%%%%%%%%%%%%%%%%%%%%%%%%%%%%%%%%%%%%%%%%%%%%%
%%%%%%%%%%%%%%%%%%%%%%%%%%%%%%%%%%%%%%%%%%%%%%%%%%%%%%%%%%%%%%%%%%%%%
\parindent=0cm
%%%%%%%%%%%%%%%%%%%%%%%%%%%%%%%%%%%%%

\section{Introduction}
Kolmogorov--Arnold Networks (KANs) \cite{liu2025kan} have recently emerged as a structured alternative to standard multilayer perceptrons (MLPs), replacing fixed pointwise nonlinearities with learnable univariate edge functions. This architectural bias has led to strong empirical performance in tasks where interactions can be effectively represented through univariate components. Such settings arise in scientific computing, physics-informed learning, time series forecasting, and molecular and biological modeling \cite{LiMolecular,CherednichenkoP25,shukla2024comprehensive,patra2025physics,wang2025kolmogorov,vaca2024kolmogorov}. In genomic sequence classification, for instance, replacing MLP modules with KAN layers has been shown to improve predictive performance on several benchmark datasets (see Figure~\ref{fig:kan}).

At the same time, KANs raise new theoretical questions: their training dynamics depend on a large collection of coupled univariate function parameters, and standard analyses developed for MLPs do not directly apply. In particular, the interaction between width, optimization dynamics, and statistical generalization in KANs remains poorly understood. 

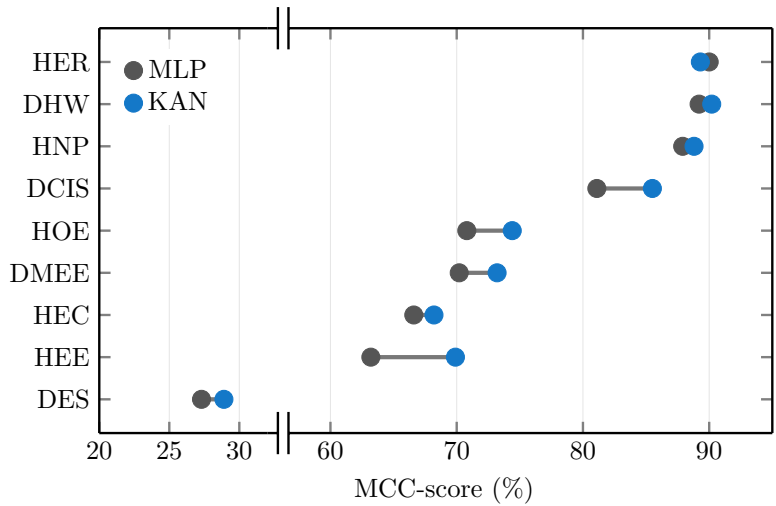
\begin{figure}[t]
\centering
\begin{tikzpicture}

\definecolor{mlpC}{RGB}{85,85,85}
\definecolor{kanC}{RGB}{20,120,200}
\definecolor{linkC}{RGB}{120,120,120}

\newcommand{\AxisTickFont}{\normalsize}
\newcommand{\AxisLabelFont}{\normalsize}
\newcommand{\AxisTitleFont}{\normalsize}
\newcommand{\AxisLegendFont}{\normalsize}

\begin{groupplot}[
  group style={group size=2 by 1, horizontal sep=0pt},
  height=0.5\linewidth,
  y dir=reverse,
  symbolic y coords={HER,DHW,HNP,DCIS,HOE,DMEE,HEC,HEE,DES},
  ytick={HER,DHW,HNP,DCIS,HOE,DMEE,HEC,HEE,DES},
  xmajorgrids,
  grid style={draw=gray!20},
  axis line style={line width=0.9pt},
  tick style={line width=0.9pt},
  tick label style={font=\AxisTickFont},
  label style={font=\AxisLabelFont},
  title style={font=\AxisTitleFont},
  clip=false,
  axis lines=box
]

% ---------------- Left panel: 20..35 ----------------
\nextgroupplot[
  width=0.37\linewidth,
  xmin=20, xmax=32,
  xtick={20,25,30},
  %xticklabel style={xshift=-4pt}, % avoid seam label clash
  xlabel={},
  legend style={font=\AxisLegendFont, draw=none, at={(0.1,0.95)}, anchor=north west},
  % built-in discontinuity marks (top+bottom at the right edge)
  % remove y ticks/axis on the inner seam
  ytick pos=left,
  axis y line*=left,
]

% Force full y-span
\addplot[draw=none, forget plot] coordinates {
  (20,HER) (20,DHW) (20,HNP) (20,DCIS) (20,HOE)
  (20,DMEE) (20,HEC) (20,HEE) (20,DES)
};

% DES only
\addplot[linkC, line width=1.6pt, forget plot] coordinates {(27.3,DES) (28.9,DES)};
\addplot[only marks, mark=*, mark size=3.4pt, draw=mlpC, fill=mlpC, forget plot]
  coordinates {(27.3,DES)};
\addplot[only marks, mark=*, mark size=3.4pt, draw=kanC, fill=kanC, forget plot]
  coordinates {(28.9,DES)};

% Marker-only legend
\addlegendimage{only marks, mark=*, mark size=3.4pt, draw=mlpC, fill=mlpC}
\addlegendentry{MLP}
\addlegendimage{only marks, mark=*, mark size=3.4pt, draw=kanC, fill=kanC}
\addlegendentry{KAN}

% ---------------- Right panel: 60..95 ----------------
\nextgroupplot[
  width=0.5\linewidth,
  xmin=55, xmax=95,
  xtick={60,70,80,90},
  %xticklabel style={xshift=4pt}, % avoid seam label clash
  xlabel={  MCC-score (\%)},
  yticklabel=\empty,
  % built-in discontinuity marks (top+bottom at the left edge)
  axis x discontinuity=parallel,
  % remove y ticks/axis on the inner seam
  ytick pos=right,
  axis y line*=right,
      every axis x label/.append style={at=(ticklabel cs:0.35)}
]

% Force full y-span
\addplot[draw=none, forget plot] coordinates {
  (60,HER) (60,DHW) (60,HNP) (60,DCIS) (60,HOE)
  (60,DMEE) (60,HEC) (60,HEE) (60,DES)
};

% Connecting lines
\addplot[linkC, line width=1.6pt, forget plot] coordinates {(90.0,HER) (89.3,HER)};
\addplot[linkC, line width=1.6pt, forget plot] coordinates {(89.2,DHW) (90.2,DHW)};
\addplot[linkC, line width=1.6pt, forget plot] coordinates {(87.9,HNP) (88.8,HNP)};
\addplot[linkC, line width=1.6pt, forget plot] coordinates {(81.1,DCIS) (85.5,DCIS)};
\addplot[linkC, line width=1.6pt, forget plot] coordinates {(70.8,HOE) (74.4,HOE)};
\addplot[linkC, line width=1.6pt, forget plot] coordinates {(70.2,DMEE) (73.2,DMEE)};
\addplot[linkC, line width=1.6pt, forget plot] coordinates {(66.6,HEC) (68.2,HEC)};
\addplot[linkC, line width=1.6pt, forget plot] coordinates {(63.2,HEE) (69.9,HEE)};

% Dots
\addplot[only marks, mark=*, mark size=3.4pt, draw=mlpC, fill=mlpC, forget plot]
coordinates {
  (90.0,HER) (89.2,DHW) (87.9,HNP) (81.1,DCIS)
  (70.8,HOE) (70.2,DMEE) (66.6,HEC) (63.2,HEE)
};
\addplot[only marks, mark=*, mark size=3.4pt, draw=kanC, fill=kanC, forget plot]
coordinates {
  (89.3,HER) (90.2,DHW) (88.8,HNP) (85.5,DCIS)
  (74.4,HOE) (73.2,DMEE) (68.2,HEC) (69.9,HEE)
};

\end{groupplot}
\end{tikzpicture}
\vspace{-1mm}
\caption{Comparison of MLP and KAN performance on genomic sequence classification. Each point pair corresponds to a benchmark dataset from Table~1 of \citet{CherednichenkoP25}, reporting Matthews correlation coefficient (higher is better) for a baseline MLP model and its KAN-based counterpart. Most KAN points lie to the right of their MLP counterparts, indicating improved predictive performance  when replacing MLP modules with KAN layers.}
\vspace{-1mm}
\label{fig:kan}
\end{figure}

Most existing theoretical analyses of KANs are algorithm-independent, focusing on expressivity and approximation-theoretic guarantees rather than the behavior of concrete training algorithms \cite{zhang2025generalization,liu2025rate,li2025generalization,wang2025expressiveness}. As a result, they do not yield guarantees for the iterates produced by gradient descent (GD), nor do they explain how optimization dynamics interact with generalization. A notable exception is \cite{gao2025convergence}, who show convergence of GD for two-layer KANs under strong positive-definiteness assumptions on the associated neural tangent kernel (NTK), which require network widths that scale polynomially in problem parameters. However, corresponding generalization guarantees for GD-trained KANs remain largely unexplored.

In addition to optimization and generalization, privacy considerations introduce an additional layer of complexity for training KANs. As models are increasingly trained on sensitive data in domains such as biology and medicine, it becomes essential to understand how privacy-preserving training procedures affect their performance. Differential privacy (DP) \cite{dwork2014algorithmic} provides a rigorous framework for limiting the influence of individual data points, and gradient-based methods with additive noise, such as differentially private gradient descent (DP-GD), are standard tools for achieving such guarantees \cite{abadi2016deep,song2013stochastic,wang2022differentially}. However, existing theoretical analyses of private gradient methods for neural networks predominantly focus on MLPs \cite{romijnders2024convex,wang2025optimaldp,shi2025towards}.

Establishing theoretical guarantees for GD and its differentially private variant (DP-GD) on KANs is technically challenging. Along the optimization trajectory, the curvature of the loss depends sensitively on how far the parameters move from initialization, while DP-GD introduces stochastic perturbations at every iteration over a nonconvex landscape. To address these challenges, we develop a unified framework that combines self-consistent curvature control for GD with a trajectory-wise sensitivity and stability analysis of the projected DP-GD recursion. This framework enables guarantees for (DP-)GD optimization, generalization, and privacy utility in KAN training. Our main contributions can be summarized as follows.

\begin{itemize}[leftmargin=*, topsep=0pt, itemsep=1pt, parsep=0pt, partopsep=0pt]
\item \emph{Reference-point framework for GD and DP-GD.}
We develop a unified reference-point analysis for training two-layer KANs with GD and DP-GD, yielding bounds on the optimization risk (training loss), the population risk, and the trajectory-averaged population risk under $(\epsilon,\delta)$-DP (Theorems~\ref{thm:opt-error}, \ref{thm:risk}, and \ref{thm:utility}).
\item \emph{Optimization and fast-rate generalization.}
Under NTK separability and logistic loss, we show that \emph{polylogarithmic width suffices} for GD to achieve an optimization risk $\widetilde{\O}(1/T)$ (Theorem~\ref{thm:ntk}) and a fast-rate population risk bound $\widetilde{\O}(1/n)$ (Theorem~\ref{thm:ntk-gen}).
\item \emph{Differential privacy and utility.}
We provide an explicit Gaussian noise calibration ensuring $(\epsilon,\delta)$-DP for DP-GD on KANs and establish a privacy--utility tradeoff with an admissible width regime. Under NTK separability, this yields a utility bound $\widetilde{\O} (\!\sqrt{d}/(n\epsilon) )$ for the trajectory-averaged population risk, matching the classical convex Lipschitz lower bound  (Theorem~\ref{thm:dp-risk-ntk}).
\item \emph{Sharp width characterization.}
In the NTK-separable setting, we show that \emph{polylogarithmic width} is not only sufficient but essentially necessary for DP-GD to attain the desired utility, revealing a qualitative gap between non-private GD and DP-GD.
 
\item \emph{Theory-guided practical implications, validated empirically.}
Our theory provides guidance for selecting the width $m$ and the iteration number $T$.
For GD, it predicts diminishing returns from increasing $m$ and $T$ beyond moderate values.
In contrast, under a fixed privacy budget, it suggests that DP-GD benefits from moderate widths and early stopping.
These qualitative trends are supported by the experiments in Section~\ref{sec:experiments}. 

\end{itemize}

\textbf{Organization:}
Section~\ref{sec:primer} introduces KANs and the model studied in this work.
Sections~\ref{sec:overview} and \ref{sec:main-results} present our main results.  Section~\ref{sec:experiments} reports experimental findings.
Related work and conclusion are given in Sections~\ref{sec:related-work} and \ref{sec:conclu}.

\section{Kolmogorov--Arnold Networks: A Primer}\label{sec:primer}
This section provides background on KANs and introduces the concrete model studied in this work.

%\vspace{-1mm}
\subsection{Motivation: Kolmogorov--Arnold representation}
\vspace{-1mm}

Kolmogorov--Arnold type representation results \citep{kolmogorov1963representation,arnol1957functions,braun2009constructive} show that broad classes of multivariate functions admit representations as superpositions of univariate functions combined with simple aggregation operations. This perspective motivates neural architectures that place learnable univariate components on network edges, rather than relying on a fixed pointwise nonlinearity shared across units. KANs, introduced by \cite{liu2025kan}, embody this idea by explicitly parameterizing edge functions, thereby inducing a structural bias that differs from standard MLPs.

Formally, a KAN is a feedforward network that maps an input vector $\bx=(x_1,\dots,x_{n_0})\in\R^{n_0}$ to a sequence of hidden representations $\bx_{\ell}=(x_{\ell,1},\dots,x_{\ell,n_\ell})\in\R^{n_\ell}$ across layers $\ell=0,\dots,L$, with $\bx_0=\bx$ and $n_\ell$ denoting the width of layer $\ell$. Each edge $(\ell,i)\to(\ell+1,j)$ carries a learnable univariate function $\phi_{\ell,j,i}:\R\to\R$, and the hidden units are computed as
\vspace{-2mm}
\begin{align}\label{eq:KAN-liu}
x_{\ell+1,j}=\sum_{i=1}^{n_\ell}\phi_{\ell,j,i}\big(x_{\ell,i}\big),\qquad \ell=0,\dots,L-1.\vspace{-2mm}
\end{align}
In practice, the edge functions are parameterized using finite-dimensional function classes such as basis expansions. We defer the specific parameterization to Section~\ref{subsec:two-layer-kan}.

\vspace{-1mm}
\subsection{Empirical Context and Applications}
We briefly summarize empirical observations and application settings that motivate the study of KANs.

\vspace{-1mm}
\paragraph{Architectural inductive bias and interpretability.}
KANs introduce an architectural inductive bias by replacing fixed pointwise activations with learnable univariate edge functions, which can be particularly effective when the target mapping exhibits low-dimensional structure, smoothness, or compositional interactions. This design aligns naturally with settings in which multivariate relationships admit structured decompositions into simpler components. In addition, the explicit parameterization of edge functions can support interpretability in practice, as individual components may be visualized and inspected directly \citep{erdmann2025kan,ranasinghe2024ginn}. 

\vspace{-2mm}
\paragraph{Application landscape.}
Building on this architectural inductive bias, empirical studies report competitive performance of KAN-based models on a variety of structured learning tasks \citep{liu2025kan,somvanshi2025survey}; see Figure~\ref{fig:kan} for an illustrative example in genomic sequence classification. In practice, KAN variants have often been explored as modular replacements for MLP components within larger architectures. Representative examples include convolutional KAN layers for vision tasks \citep{bodner2024convolutional,ferdaus2024kanice}, temporal KAN designs for time-series forecasting \citep{genet2024temporal}, and Transformer architectures in which the feedforward subnetwork is replaced by KAN layers \citep{wang2025kolmogorov}. KAN components have also been combined with diffusion models in several recent works \citep{xiong2025conditional,su2025kan,qiu2025finding}.

\vspace{-1mm}
\subsection{Two-layer KANs with B-splines}\label{subsec:two-layer-kan}
We study a two-layer KAN with a single hidden layer of width $m$ and a scalar output, in which the univariate edge functions are parameterized using B-spline bases.

\vspace{-2mm}
\paragraph{Parameter representation.}
For analysis, we collect all trainable coefficients into a single parameter vector. For each hidden unit $j\in[m]$, the first-layer spline coefficients $\{a_{i,j,k}\}_{i\in[d],\,k\in[p]}$ are arranged into a vector $\ba_j\in\R^{dp}$ using a fixed ordering of the index pair $(i,k)$, and we set $\ba=(\ba_1,\ldots,\ba_m)\in\R^{mdp}$. Similarly, the second-layer coefficients $\{c_{j,k}\}_{k\in[p]}$ are collected into vectors $\bc_j\in\R^{p}$ and stacked as $\bc=(\bc_1,\ldots,\bc_m)\in\R^{mp}$. The complete parameter vector is thus $\Theta=(\ba,\bc)\in\R^{mp(d+1)}$.

\vspace{-1mm}
\paragraph{Two-layer KAN model.}
%We study a two-layer KAN with a single hidden layer of width $m$ and a scalar output.
Each univariate edge function is parameterized using a B-spline basis $\{b_k\}_{k=1}^p$, and a bounded activation function $\sigma:\R\to\R$ (e.g., $\tanh$ or sigmoid) is applied after the first layer. The spline basis functions are defined on an interval covering all values encountered during training and satisfy the boundedness conditions in Assumption~\ref{ass:sigma}. For an input $\bx_0=\bx\in\R^d$, with input dimension $d$ and hidden width $m$, the two-layer KAN model is given by
\vspace{-1mm}
\begin{equation}\label{eq:KAN}
f_{\Theta}(\bx)= \frac{1}{\sqrt{m}} \sum_{j=1}^m\, \sum_{k=1}^{p} \, c_{j,k}\, b_k\big(x_{1,j}\big),\vspace{-1mm}
\end{equation}
where the hidden units take the form
\vspace{-1mm}
\[
x_{1,j}
= \sigma \Big(\frac{1}{\sqrt{d}} \sum_{i=1}^d \sum_{k=1}^{p} a_{i,j,k}\, b_k(x_{0,i}) \Big),\quad j\in[m].
\]
This model closely follows the two-layer KAN of \cite{gao2025convergence}, adding an explicit $1/\sqrt{d}$ normalization in the first layer to stabilize pre-activations as $d$ grows and simplify the Hessian analysis. The $1/\sqrt{m}$ factor is the standard output normalization to control the scale as $m$ increases.

\vspace{-1mm}
\subsection{Toward a Theoretical Understanding}\label{subsec:theory-gap}
For the two-layer KAN model defined above, key theoretical questions remain open. In particular, it is unclear when gradient-based methods optimize the model reliably, how the resulting predictors generalize, and how these guarantees interact with privacy constraints. This motivates our theoretical study of optimization, generalization, and differential privacy for KANs.

\vspace{-1mm}
\section{Overview of Our Results}\label{sec:overview}
Our formal results are presented in Section~\ref{sec:main-results} in a general form. In this section, we specialize to the logistic loss and an NTK-separable regime to highlight the main implications and intuitions. 
%Throughout the paper, we refer to the training loss as the \emph{optimization risk} and the generalization as the \emph{population risk}.
%In this terminology, our \emph{generalization bounds} are bounds on the population risk.
%In the private setting, a \emph{utility bound} is an upper bound on the population risk guaranteed under $(\epsilon,\delta)$-DP.
We use $\widetilde \O(\cdot)$ to suppress logarithmic factors and write $a \asymp b$ if $a$ and $b$ are of the same order.

%In this section, we provide an informal overview of our main results. For concreteness, we focus on the logistic loss $\ell(u)=\log(1+\exp(-u))$ and consider a separable setting where the data are separated with margin $\gamma$ by the neural tangent kernel (NTK) features at initialization (see Assumption~\ref{ass:ntk}). The \textit{general and formal} results for broader losses and settings are deferred to Section~\ref{sec:main-results}. 

\vspace{-1mm}
\subsection{Optimization Bound}
We run GD with a constant step size $\eta>0$ for $T$ iterations.  The following theorem provides the convergence behavior of GD for KANs.  
For readability, we suppress the explicit dependence of the required width on problem parameters and defer these details to Section~\ref{sec:main-results}.
\begin{theorem}[Informal version of Theorem~\ref{thm:ntk}]
Under the NTK separability assumption with margin $\gamma>0$, 
if the network width is polylogarithmic,
i.e., $m \ge \mathrm{polylog}(n,T)$, then with high probability (w.h.p.) over the random initialization, a two-layer KAN trained by GD achieves an optimization risk (training loss) at most
\vspace{-1mm}\[\widetilde{\O} \Big(\frac{1}{\gamma^2 \eta T}\Big).\vspace{-1mm}\]
\end{theorem} 
Remarkably, a \emph{polylogarithmic} network width already \emph{suffices} for GD to attain an $\O(1/T)$ optimization rate.

 \vspace{-1mm}
\subsection{Generalization Bound}\label{subsec:overview-gen-risk}
Building on our optimization guarantee, we establish the first algorithm-dependent generalization bound specifically for GD-trained KANs, achieving a fast $\mathcal{O}(1/n)$ rate up to the separability margin and logarithmic factors.
\begin{theorem}[Informal version of Theorem~\ref{thm:ntk-gen}]
Under the NTK separability assumption with margin $\gamma>0$, 
if the network width is polylogarithmic,
 and $ \eta T \gtrsim n$,
then w.h.p. over the random initialization, a two-layer KAN trained by GD achieves the expected population risk at most 
\vspace{-1mm}
\[ \widetilde{ \O} \Big(\frac{1}{\gamma^4 n}\Big), \vspace{-1mm} \]
where the expectation is taken with respect to (w.r.t.) the draw of the training data.
\end{theorem}
Moreover, our result reveals an implicit regularization bias toward solutions that remain close to initialization. See Section~\ref{subsec:gen-gd} for further discussion.

\subsection{Differentially Private Gradient Descent}
Finally, we study DP-GD, where privacy imposes additional width constraints, and give a sharp utility bound.
\begin{theorem}[Informal version of Theorem~\ref{thm:dp-risk-ntk}]
Under the NTK separability assumption with margin $\gamma>0$,
for $\eta T \asymp \frac{\gamma^2 n\epsilon}{\sqrt{d}}$ and a polylogarithmic width $m \asymp \mathrm{polylog} ( n  )$,
w.h.p. over the random initialization, the $(\epsilon,\delta)$-DP variant of GD algorithm (Algorithm~\ref{alg1}) achieves an expected population risk, averaged over $T$ iterates, bounded by
\vspace{-1mm}
\[\widetilde{\O} \Big(\frac{\sqrt{d}}{\gamma^4 n\epsilon}\Big).\vspace{-1mm}\]
\end{theorem}
Our result also characterizes an admissible width regime with an essentially matching lower bound, showing that a polylogarithmic width is necessary for DP-GD to achieve the desired utility bound. See Section~\ref{sec:DP-GD} for details.

\vspace{-1mm}
\subsection{Comparison to Prior Work}
We refer to Theorems~\ref{thm:ntk}, \ref{thm:ntk-gen} and \ref{thm:dp-risk-ntk} for the precise statements of our results and highlight here the main points of departure from existing theory.

\paragraph{Optimization bounds.}
The work most closely related to our optimization analysis is \cite{gao2025convergence}, which studies GD for two-layer B-spline KANs in regression. 
Under a positive-definiteness assumption on the expected NTK Gram matrix $G^\infty$ (i.e., $\lambda_{\min}(G^\infty)>0$), they prove global linear convergence $(1-\frac{\eta }{2}\lambda_{\min}(G^\infty))^T$ but require network widths that scale polynomially in the problem parameters.
In contrast, we analyze GD  for classification under an NTK-separability condition.
As noted by \cite{nitanda2019gradient}, this assumption is weaker than a Gram-matrix positive-definiteness condition and can be satisfied in many cases due to the universality of the neural tangent
models.
This yields sublinear convergence $\widetilde{\O}(\frac{1}{\eta T\gamma^2})$ already in a polylogarithmic-width regime.

\paragraph{Generalization bounds.}
Most existing generalization results for KANs are \emph{algorithm-independent} \cite{zhang2025generalization,li2025generalization,liu2025rate}. 
By comparison, we provide the \emph{first algorithm-dependent} generalization bound $\widetilde{\O}(1/\gamma^4 n)$ for GD-trained KANs under NTK separability.

\paragraph{Utility bounds for DP-GD.}
To the best of our knowledge, this is the first utility analysis of DP gradient-based methods for KANs.
A minimax lower bound tailored to KANs under DP is not currently available. 
We therefore compare to the classical  lower bound for general convex Lipschitz problems \citep{bassily2019private},
which is widely used as a generic baseline in the DP optimization literature.
Our utility bound $\widetilde{\O}(\frac{\sqrt{d}}{n\epsilon})$ matches this baseline up to the separability margin and logarithmic factors.
%Our utility bound $\widetilde{\O}(\frac{\sqrt{d}}{n\epsilon})$ matches the classical lower bound for general convex Lipschitz problems \citep{bassily2019private}, up to the separability margin and logarithmic factors.
%This benchmark is for a broader function class and may not be tight for KANs.

 \vspace{-1mm}
\section{Preliminaries and Main Results}\label{sec:main-results}
We begin by describing our problem setup and then present our main results. Proofs are provided in Appendix~\ref{sec:appe-opt}--\ref{sec:appen-dp}.

\vspace{-1mm}
\subsection{Preliminaries}
We consider the following empirical minimization problem for a two-layer KAN classifier $f_{\Theta}$ parameterized by $\Theta\in \R^{mp(d+1)}$ (see \eqref{eq:KAN}): 
\vspace{-1mm}
\[\min_{\Theta\in \R^{mp(d+1)}} \mathcal{L}_S(\Theta) = \frac{1}{n} \sum_{i=1}^n  \ell  (y_i f_\Theta(\bx_i) ) \vspace{-1mm}\] 
given a training  data set  $S=\{z_i=(\bx_i, y_i)\}_{i=1}^n$ where each $z_i$ is independently drawn from the population distribution $\mathcal{P}$. 
Here, we call $\mathcal{L}_S(\Theta)$ the optimization risk (training loss).  
Denote $\|\cdot\|_2$ the standard Euclidean norm for vectors and the operator norm for matrices. Throughout the paper, we assume $\|\bx_i\|_2\le 1$, the binary labels $y_i\in \{-1,+1\}$ and $\ell(\cdot)$ be a nonnegative convex loss function.   Define the population risk  $\mathcal{L}(\Theta) = \mathbb{E}_{(\bx,y)\sim\mathcal{P}}  [\ell  (y f_\Theta(\bx)) ].$
%We use \emph{optimization risk} to refer to the  training risk $\mathcal{L}_S(\Theta)$ and \emph{generalization risk} to refer to the  test risk $\mathcal{L}(\Theta)$.

For notational convenience, we identify the concatenation $[\ba^\top, \bc^\top]^\top$ with the ordered pair $(\ba,\bc)$. 
We train $\Theta(k)=(\ba(k),\bc(k))$ by GD with a step size  $\eta>0$:
$$\ba(k+1)=\ba(k)-\eta \, \partial_\ba \mathcal{L}_S(\Theta(k)) ,$$
$$\bc(k+1)=\bc(k)-\eta \, \partial_\bc \mathcal{L}_S(\Theta(k))  .$$  
Following \cite{gao2025convergence}, we use the standard Gaussian initialization:
\begin{equation}\label{eq:init-W}
   \ba(0)  \sim  \N(0, \mathbf{I}_{mdp}) \, \text{ and } \, \bc(0)  \sim  \N(0, \mathbf{I}_{mp}) . 
\end{equation}
We impose two standard assumptions \citep{gao2025convergence,taheri2024sharper}.
Basis functions with degree exceeding three (e.g., cubic B-spline basis) and commonly used  transformation functions
(e.g., sigmoid and hyperbolic tangent) can be chosen so that the following assumption holds.
\begin{assumption}\label{ass:sigma}
Assume $\sigma$ satisfies $|\sigma(u)|\le B_\sigma$, $|\sigma'(u)|\le B'_\sigma$, and $|\sigma''(u)|\le B''_\sigma$ for all $u\in\R$.
Further, assume $\{b_k\}_{k=1}^p$ satisfy $|b_k(v)|\le B_b$, $|b_k'(v)|\le B'_b$, and $|b_k''(v)|\le B''_b$ for all $v\in\mathrm{range}(\sigma)$.
\end{assumption}
\begin{assumption} \label{ass:loss}
    Assume the  nonnegative convex loss $\ell:\R \rightarrow \R_+$ satisfies
     $|\ell'(u)|\le B'_\ell,$ and $\ell''(u) \le B''_\ell $ for all $u\in \R$. 
    Further, assume $\ell$ is self-bounded with $\alpha_\ell >0$, i.e., $|\ell'(u)|\le \alpha_\ell  \ell(u)$ for all $u\in\R$.
\end{assumption}
The logistic loss naturally satisfies Assumption~\ref{ass:loss} with $B'_\ell=1$,  $B''_\ell=1/4$ and $\alpha_\ell=1$. Unless otherwise specified, we set $B'_\ell=B''_\ell=\alpha_\ell=1$ in the sequel. Throughout the paper, we assume  that Assumptions~\ref{ass:sigma} and \ref{ass:loss} hold.

\subsection{Optimization Guarantees of GD}\label{subsec:opt}
We now present our main results on the optimization bounds. Our bounds are stated relative to a reference point $\Theta^*$ (e.g., an interpolating solution). 
For any $\Theta^* \in \R^{mp(d+1)}$,  define its \textit{reference-point complexity} by \[\mathfrak{C}_S(\Theta^*)=2\eta T \mathcal{L}_S(\Theta^*) +    \|  \Theta(0)-\Theta^* \|_2^2, \]
and further denote its expected counterpart by $\mathfrak{C}(\Theta^*)=\E_S[\mathfrak{C}_S(\Theta^*)]$ and write $\Lambda_{\Theta^*}=\|\Theta(0) - \Theta^{*} \|_2$ for the distance from the initialization. 
Let $C_{\sigma,b} \ge 0$ be a generic constant depending only on the quantities specified in Assumption~\ref{ass:sigma}.%, its value may vary from line to line.
Define $\rho_\ell := C_{\sigma,b}\,p^3$, which plays the role of a (trajectory-dependent) smoothness parameter for $\mathcal{L}_S$.  An explicit expression for $\rho_\ell$ is provided in the appendix.
We use $a \gtrsim b$ (resp.\ $a \lesssim b$) to denote $a \ge C b$ (resp.\ $a \le C b$) for an absolute constant $C>0$. Since $p$ is fixed in our parameterization, we suppress its dependence in our results. 
%In the following theorem, we establish conditions for the network width and the reference point that guarantee the optimization risk decays to zero if the network can interpolate the training set. Let $\delta \in (0,1)$  denote a failure probability. 

Let $\{\Theta(k)\}_{k=0}^T$ be the iterates produced by GD with step size $\eta$, and let $\Theta(T)$ be the output after $T$ iterations. The following theorem gives the convergence behavior of GD. 
\begin{theorem}[Optimization -- General bound]\label{thm:opt-error} 
Let $\Theta^* $ be a reference point. Assume  $\eta \le \min\{1/\rho_\ell,1\}$ and $\Lambda_{\Theta^*}^2 \ge  4 \max  \{  \eta T \mathcal{L}_S(\Theta^*),  \eta \mathcal{L}_S(\Theta(0)) \} $.  
If $m\gtrsim  ( \log(m/\delta) + \Lambda_{\Theta^*}^2 )\Lambda_{\Theta^*}^4$, then with probability at least $1-\delta$ over the randomness of the initialization, it holds that
\vspace{-1mm}
    \begin{align*}%\label{eq:opt-rate}
          \mathcal{L}_S(\Theta(T ) ) 
        \le \frac{1}{T}\sum_{k=1}^T\L_S(\Theta(k)) \le \frac{\mathfrak{C}_S(\Theta^*)}{\eta T}.  \vspace{-1mm}
    \end{align*}
        Furthermore, for all $k\in[T-1]$, with probability at least $1-\delta$ over initialization, it holds that
        \vspace{-1mm}\[ \|\Theta(k\!+\!1) \!-\! \Theta^*\|_2 \! \le \!\sqrt{2}\Lambda_{\Theta^*}, \ \ \|\Theta(k\!+\!1) \!-\! \Theta(0)\|_2 \! \le\! 3\Lambda_{\Theta^*}. \vspace{-1mm}\]
\end{theorem}
\textit{Discussion of Results.}
Beyond the $\O(\mathfrak{C}_S(\Theta^*)/T)$ convergence rate, Theorem~\ref{thm:opt-error} also shows that the GD iterates remain in a controlled neighborhood of both the initialization $\Theta(0)$ and the reference point $\Theta^*$.
As we show below, this \emph{stay-in-a-ball} property is the cornerstone of our analysis: it allows us to establish that the training loss $\mathcal{L}_S$ is locally smooth and almost  convex along the GD trajectory, which in turn yields the convergence guarantees.

\begin{remark}[Key proof idea]
A main challenge in proving Theorem~\ref{thm:opt-error} is that the curvature of the training loss $\mathcal{L}_S$ is \textit{trajectory-dependent}:
explicit Hessian calculations show that both the minimum and maximum eigenvalues depend on the evolving deviation $\|\Theta(k)-\Theta(0)\|_2$.
To make the curvature control self-consistent, we develop a \emph{double induction} that simultaneously bounds
(i) the cumulative training loss $\eta\sum_{s=1}^k \mathcal{L}_S(\Theta(s))$ and
(ii) the distances $\|\Theta(k)-\Theta^*\|_2$ and $\|\Theta(k)-\Theta(0)\|_2$.
These coupled bounds ensure that all iterates stay inside a region where the Hessian eigenvalues are uniformly controlled, thereby certifying local smoothness and almost-convexity along the trajectory. 
\end{remark}

\paragraph{Risks under Realizability.} 
To clarify the dependency of $\Theta^*$ in Theorem ~\ref{thm:opt-error}, we impose a realizability assumption, which posits the existence of a model near the initialization that achieves arbitrarily small training error.
Similar assumptions have been considered in \citep{schliserman2022stability,taheri2024generalization}.
We verify this assumption in Theorem~\ref{thm:ntk}.
\begin{assumption}[Realizability]\label{ass:realizability}
Assume that for almost all draws of $S\sim\mathcal{P}^n$ and for any sufficiently small $\epsilon>0$, the set
$\{\Theta:\mathcal{L}_S(\Theta)\le \epsilon\}$ is non-empty, and define 
$g(\epsilon):=\inf\big\{\|\Theta-\Theta(0)\|_2:\ \mathcal{L}_S(\Theta)\le \epsilon\big\}.$
Assume moreover that the infimum is attained, i.e., there exists $\Theta^\epsilon$ such that
$\mathcal{L}_S(\Theta^\epsilon)\le \epsilon$ and $g(\epsilon)=\|\Theta^\epsilon-\Theta(0)\|_2$.
\end{assumption}

\begin{remark}[Milder condition]
In Lemma~\ref{lem:reali}, we show that Assumption~\ref{ass:realizability} can be relaxed as follows: for any sufficiently small $\epsilon>0$, there exists $\Theta^\epsilon$ such that $\mathcal{L}_S(\Theta^\epsilon)\le \epsilon$.
%Equivalently, the sublevel set $\{\Theta:\mathcal{L}_S(\Theta)\le \epsilon\}$ is non-empty.
This requirement is strictly weaker than Assumption~\ref{ass:realizability}.
\end{remark}

\begin{theorem}[Optimization under Realizability]\label{thm:risk-reali}
Suppose Assumption \ref{ass:realizability} holds.
Let $\eta \lesssim  \min\{ g^2(1),$  $ g^2(1) (\mathcal{L}_S(\Theta(0)))^{-1}\}$  be a constant, and assume $m\gtrsim  {  \big( \log(  {m }/{\delta})   +  g^2( {1}/{T})  \big)} g^4( {1}/{T})$.  Then, with probability at least $1-\delta$ over the randomness of the initialization, \vspace{-1mm}
\[\mathcal{L}_S(\Theta(T)) \le \frac{1}{T}\sum_{k=1}^T\L_S(\Theta(k)) \lesssim \frac{ \eta + g^2\big(\frac{1}{T}\big)}{\eta T}.\vspace{-1mm}\]
\end{theorem}
The above optimization performance depends explicitly on the realizability profile $g$.
If $g(\epsilon)\lesssim \log(1/\epsilon)$, then choosing $\eta T \asymp n$
% (a minimal choice, though larger $\eta T$ is also allowed)
yields $\mathcal{L}_S(\Theta(T)) \lesssim \frac{\log^2(n)}{n}$ under $m\gtrsim \log^6(n)$ (up to logarithmic factors in $1/\delta$).

\paragraph{Connection to NTK Separability.} 
We next connect realizability to a separability condition formulated in terms of the NTK features at initialization. As shown in \cite{nitanda2019gradient}, this separability assumption is weaker than the positivity assumption on the Gram-matrix of NTK considered in the literature \cite{arora2019fine,du2019gradient,gao2025convergence}. 
Let $\langle\cdot,\cdot\rangle$ denote the dot product.
\begin{assumption}[NTK separability\label{ass:ntk}]
Let $\gamma\in(0,1]$.
Assume there exists $\Theta_0\!\in\rbb^{mp(d+1)}$ with $\|\Theta_0\|_2=1$ such that  
$y_i\big\langle\nabla f_{\Theta(0)}(\bx_i) ,\Theta_0\big\rangle\geq\gamma, \forall i\in[n]. $
\end{assumption}
For concreteness, we specialize to the logistic loss in the following theorem, which shows that with polylogarithmic width, GD attains a $\widetilde{\O}(1/T)$ optimization rate.
\begin{theorem}[Optimization under NTK separability]\label{thm:ntk}
    Suppose Assumption~\ref{ass:ntk} holds. Let $\ell$ be the logistic loss. 
    Let $\eta \lesssim  1$  be a constant, assume $m \gtrsim \log(m/\delta)\big(\log^6(T) + \log^3(n/\delta)\big)/\gamma^6$.
    Then, with probability at least $1-\delta$ over the randomness of the initialization, it holds that 
    \vspace{-1mm}
\[\mathcal{L}_S(\Theta(T)) \le \frac{1}{T}\sum_{k=1}^T\L_S(\Theta(k)) \lesssim \frac{\log^2(T) + \log\big(\frac{n}{\delta}\big)}{\gamma^2\eta T}.\vspace{-1mm}\]
\end{theorem}

\subsection{Generalization Guarantees of GD}\label{subsec:gen-gd}
To derive generalization (population risk) guarantees for GD, it suffices to control the generalization gap $\mathcal{L}(\Theta(T))-\mathcal{L}_S(\Theta(T))$ and then combine it with our optimization results (Theorem~\ref{thm:opt-error}).
We bound the generalization gap via on-average argument stability \citep{lei2020fine}.
Without loss of generality, we assume $ \Lambda_{\Theta^*}\ge 1$ (otherwise we may replace it by $1$).
\begin{theorem}[Generalization gap via on-average argument stability]\label{thm:generalization}
Let $\widetilde{S}$ be an i.i.d.\ copy of $S$, and let $\Theta^*$ be a reference point independent of $S$. Assume $\eta \le \min\{1/2\rho_\ell,1\}$, $\Lambda_{\Theta^*}^2 \ge  8 \max\big\{ \eta T \big(\mathcal{L}_S(\Theta^*) + \mathcal{L}_{\widetilde{S}}(\Theta^*)\big),\ 
\eta \big(\mathcal{L}_S(\Theta(0)) + \mathcal{L}_{\widetilde{S}}(\Theta(0))\big) \big\},$  and  $m\gtrsim \big(\log\big(\tfrac{m}{\delta}\big)+\Lambda_{\Theta^*}^2\big)\Lambda_{\Theta^*}^4.$ 
Then, with probability at least $1-\delta$ over the random initialization,
\vspace{-1mm}
\[\E_{S} \big[ \mathcal{L}(\Theta(T)) - \mathcal{L}_S(\Theta(T)) \big]
\lesssim \frac{\eta \Lambda_{\Theta^*}^2}{n}\, \E_{S}\Big[\sum_{t=0}^{T} \mathcal{L}_S(\Theta(t))\Big]. \vspace{-2mm}\]
Here, the expectation is taken over the draw of  $S$.
\end{theorem}
Theorem~\ref{thm:generalization} shows that the generalization gap is controlled by the cumulative training loss along the GD trajectory.
Consequently, whenever GD achieves small training loss, it also enjoys a small generalization gap.

Combining Theorem~\ref{thm:generalization} with the optimization bound in Theorem~\ref{thm:opt-error} yields the following generalization guarantee.
\begin{theorem}[Generalization -- General  bound]\label{thm:risk}
Suppose the assumptions of Theorems~\ref{thm:generalization} hold. 
Then, with probability at least $1-\delta$ over the randomness of the initialization, \vspace{-1mm} 
\[   \ebb_S\big[\mathcal{L}(\Theta(T))\big]  \lesssim   \big( \frac{ \Lambda_{\Theta^*}^2}{n}  +  \frac{  1}{\eta T}\big)  \mathfrak{C}(\Theta^*).\vspace{-1mm}\]
\end{theorem}
\textit{Discussion of Results.}
Theorem~\ref{thm:risk} has two immediate implications.
First, we can certify a suitable reference point $\Theta^*$ such that
$\Lambda_{\Theta^*}^2 \le \mathfrak{C}_S(\Theta^*) \lesssim \big(\log^2(T)+\log(n/\delta)\big)/\gamma^2$, for instance under the NTK separability assumption.
Consequently, taking $\eta T \asymp n$
% (the smallest order that suffices)
and $m \gtrsim \mathrm{polylog}(n,\delta^{-1})$, a polylogarithmic width already \emph{suffices} to achieve the fast $\widetilde{\O}(1/n)$ rate for the population risk.
Second, our theorem highlights an \emph{implicit regularization} effect of GD.
Since the bound is stated in terms of $\mathfrak{C}(\Theta^*)$, we may choose
$\Theta^* \in \arg\min_{\Theta}\big\{\mathcal{L}(\Theta)+\frac{1}{2\eta T}\|\Theta-\Theta(0)\|_2^2\big\}$.
Thus, our analysis is informative whenever
$\inf_{\Theta}\big\{\mathcal{L}(\Theta)+\frac{1}{2\eta T}\|\Theta-\Theta(0)\|_2^2\big\}$
is small, favoring predictors with low generalization risk that remain close to initialization, consistent with prior observations on implicit bias in overparameterized learning \citep{oymak2019overparameterized}.

We instantiate Theorem~\ref{thm:risk} under NTK separability, where an fast rate in $n$ up to logarithmic factors is obtained.
We also provide population risk bound under realizability assumption in Theorem~\ref{thm:risk-reali-gen}.   
\begin{theorem}[Generalization under NTK separability]\label{thm:ntk-gen}
Suppose Assumption~\ref{ass:ntk} holds. Let $\ell$ be logistic loss. 
Let $\eta \lesssim 1$ be a constant step size.
Assume that $\eta T \gtrsim n$ and $m \gtrsim \log(\tfrac{m}{\delta})\big(\log^6(T) + \log^3(n/\delta)\big)/\gamma^6$.
Then, with probability at least $1-\delta$ over the randomness of the random initialization, it holds that
$\ebb_S\big[\mathcal{L}(\Theta(T))\big]
    \lesssim \frac{\log^4(n) + \log^2 ({n}/{\delta} )}{\gamma^4 n}.$ 
\end{theorem}

\subsection{Differentially Private Gradient Descent}\label{sec:DP-GD}
To enable training KANs on sensitive datasets, we study a differentially private variant of GD (DP-GD).
We first present DP-GD with an explicit choice of the Gaussian noise variance ensuring $(\epsilon,\delta)$-DP, and then establish its privacy and utility (population risk) guarantees.

We begin by recalling DP \cite{dwork2006calibrating}. 
Two datasets $S$ and $S'$ are neighboring if they differ in one data point.
\begin{definition}[Differential privacy] \label{def:DP}
We say that a randomized algorithm $\A$ satisfies $(\gep, \delta)$-DP if, for any two neighboring datasets $S$ and $S'$  and any measurable set $E$ in the output space of $\A$,  it holds 
$\mathbb{P}(\A (S)\in E) \le e^\gep \mathbb{P}(\A(S')\in E)+ \gd.$ 
We say $\A$ satisfies $\gep$-DP if $\gd=0$. 
\end{definition} 
We ensure DP by adding Gaussian noise to the gradients at each iteration and projecting the noisy updates onto bounded domains.
The noise variances are chosen according to the $\ell_2$-sensitivity (Definition~\ref{def:sensitivity}) of the per-iteration gradients.
%We control this sensitivity by ensuring all iterates remain in bounded sets, which yields uniform bounds along the DP-GD trajectory.

Specifically, let $\ba(0)\in\R^{mdp}$ and $\bc(0)\in\R^{mp}$ be the initialization of DP-GD generated according to \eqref{eq:init-W}.
For $R>0$, let $\mathcal{B}(\bar{u},R)$ denote the closed Euclidean ball centered at $\bar{u}$ with radius $R$.
We define the parameter domains $\Omega_{\ba}= \mathcal{B}(\ba(0),R_1)$ and $\Omega_{\bc}= \mathcal{B}(\bc(0),R_2)$ for some constants $R_1,R_2 >0$, and let $\proj_\Omega(\cdot)$ denote the Euclidean projection onto $\Omega$. Rather than using gradient clipping, we control sensitivity by constraining the iterates via projection. 
Given dataset $S$, step size $\eta$, and privacy parameters $(\epsilon,\delta)$.
For $k=0,\dots,T-1$, let $\widetilde{\Theta}(k)=(\ba(k),\bc(k))$, DP-GD updates $\ba(k)$ and $\bc(k)$ via 
\begin{equation}\label{eq:update-A}
   \! \!\ba(k + 1)= \proj_{\Omega_\ba}\!  \big(\ba(k) - \eta \big( \partial_\ba \mathcal{L}_S (\widetilde{\Theta}(k))   +   \bb_1(k)\big)\big),
\end{equation}
\begin{equation}\label{eq:update-c}
\!\bc(k+1) = \proj_{\Omega_\bc} \big(\bc(k) - \eta \big( \partial_{\bc} \mathcal{L}_S(\widetilde{\Theta}(k))   +  \bb_2(k)\big)\big).
\end{equation}
Here, $\bb_1(k)\in \R^{mdp}$ and $\bb_2(k)\in\R^{mp}$ are independent Gaussian vectors given by
\begin{equation*}%\label{eq:noise-a}
    \bb_{1}(k) \overset{\mathrm{i.i.d.}}{\sim}  \N(0, \sigma_1^2\, \mathbf{I}_{mdp} ) \ \text{ and }  \ 
    \bb_{2}(k) \overset{\mathrm{i.i.d.}}{\sim} \N(0,  \sigma_2^2\,  \mathbf{I}_{mp}) .
\end{equation*}
with $\sigma_1^2 = C_1\, \tilde{\sigma}^2$ and $\sigma_2^2 = C_2\, \tilde{\sigma}^2$.
Here, $\tilde{\sigma}^2=T \big(1+\frac{\log( 2T/{\delta})}{\epsilon}\big) (n^2 \epsilon)^{-1}$,
$C_1=8(B'_{\ell}B'_\sigma B'_b B_b)^2 p^2\big(4\sqrt{p} + \frac{2\sqrt{ \log( {2}/{\delta})} + R_2}{\sqrt{m}}\big)^2$,
and $C_{2} =8(B'_\ell B_b)^2p$. 
The detailed procedure is summarized in Algorithm~\ref{alg1} (see Appendix~\ref{sec:appen-dp}), and its privacy guarantee is stated below.

\begin{theorem}[Privacy guarantee]\label{thm:privacy}
Algorithm~\ref{alg1} satisfies $(\epsilon,\delta)$-DP.
\end{theorem}

Unlike non-private GD, DP-GD injects additive Gaussian noise into the gradient at every iteration.
Consequently, the iterates follow a stochastic trajectory and monotonic decrease of the training loss is generally not guaranteed in this nonconvex setting.
We therefore study the averaged population risk along the DP-GD trajectory.% which is the quantity appearing in Theorem~\ref{thm:utility}. 

Fix any reference point $\Theta^*=(\ba^*,\bc^*)$.
We assume $\ba^*\in\Omega_{\ba}$, $\bc^*\in\Omega_{\bc}$, and that $\Theta^*$ is independent of $S$.
% Define $\widetilde{\mathfrak{C}}_S(\Theta^*)=4\eta T \mathcal{L}_S(\Theta^*) + 4\|\widetilde{\Theta}(0)-\Theta^*\|_2^2,$ 
% Define $\widetilde{\mathfrak{C}}(\Theta^*)=\E_{S,\A}[\widetilde{\mathfrak{C}}_S(\Theta^*)]$ and 
$\widetilde{\Lambda}_{\Theta^*}=\|\widetilde{\Theta}(0)-\Theta^*\|_2$.
% Here, the expectation $\E_{S,\A}$ is taken over the draw of the training sample $S$ and the internal randomness of Algorithm~\ref{alg1}.
Let $\bar{\rho} = C_{\sigma,b}\, p^2\big(p+\frac{R_2^2}{m}\big)$ and $R=R_1+R_2$.
\begin{theorem}[Utility guarantee -- General bound]\label{thm:utility}
Let $\{\widetilde{\Theta}(k)\}_{k=1}^T$ be produced by Algorithm~\ref{alg1}. 
Assume  $\|\widetilde{\Theta}(0)-\Theta^*\|_2^2
\ge
C\max \{
\eta T\big(\mathcal L_S(\Theta^*)+\mathcal L_{\widetilde S}(\Theta^*)\big),
\eta\big(\mathcal L_S(\widetilde{\Theta}(0))
+\mathcal L_{\widetilde S}(\widetilde{\Theta}(0))\big)
 \}$ and $\eta\le \min\{1/3\bar\rho,1\}$. If $m\gtrsim
 \big(\log (\frac{m}{\delta} )+R^2\big)
\max\big\{
R^4+\widetilde{\Lambda}_{\Theta^*}^4,\,
 ( \frac{ \eta T\log(T/\delta)}{n\epsilon}  )^2 \big\} $
 and \vspace{-1mm} \[m \lesssim (n\epsilon)^4\big(d^2(\eta T)^4 (\log( {m}/{\delta}) + R^2) \log^2( {T}/{\delta})\big)^{-1}, \]
then with probability at least $1-\delta$ over the randomness of the initialization, $ \frac{1}{T} \sum_{k=1}^T   \E_{S,\A}\big[ \mathcal{L}(\widetilde{\Theta}(k)) \big] $ is controlled by
\vspace{-2mm}
\begin{align*} 
  \big(\frac{1}{\eta T} + \frac{m^{\frac{1}{4}}}{n}\big)\widetilde{\Lambda}_{\Theta^*}^2  +    \big(1 + \frac{\eta T}{n}\big)     \frac{ m\eta T  d   \log( \frac{T}{\delta} ) }{n^2\epsilon^2} +\delta.\vspace{-2mm}
\end{align*}
\end{theorem}

\textit{Discussion of Results.}
A key contribution of our DP analysis is a \emph{width characterization} for private KAN training: we prove matching upper and lower bounds on an admissible width regime.
In particular, under the NTK separability setting (see Theorem~\ref{thm:dp-risk-ntk}), a polylogarithmic scaling of $m$ is not only sufficient but also essentially necessary for DP-GD to attain a utility rate ${\O}( {\sqrt{d}}/{n\epsilon})$, matching the classical convex Lipschitz lower bound in its dependence on $d$, $n$, and $\epsilon$ in \cite{bassily2019private}.

\begin{figure*}[!t]
\centering
\setlength{\tabcolsep}{4pt}
\renewcommand{\arraystretch}{1}

% ================= Subfigure: GD =================
\begin{subfigure}[t]{0.43\textwidth}
\centering
\captionsetup{margin={5em,0pt}}
\caption{\text{GD}}

\begin{tabular}{@{}c c c c@{}}
% & \multicolumn{1}{c}{\footnotesize Train/test acc. vs. $m$}
% & \multicolumn{1}{c}{\footnotesize Train/test acc. vs. $T$} \\
% [-0em]

\adjustbox{valign=c}{\rotatebox{90}{\small Synthetic} \hspace{.2em}} &
\multirow{3}{*}{\rotatebox{90}{\small Training and test accuracy}} &
\adjustbox{valign=c}{\hspace{-2mm}\includegraphics[width=0.48\linewidth]{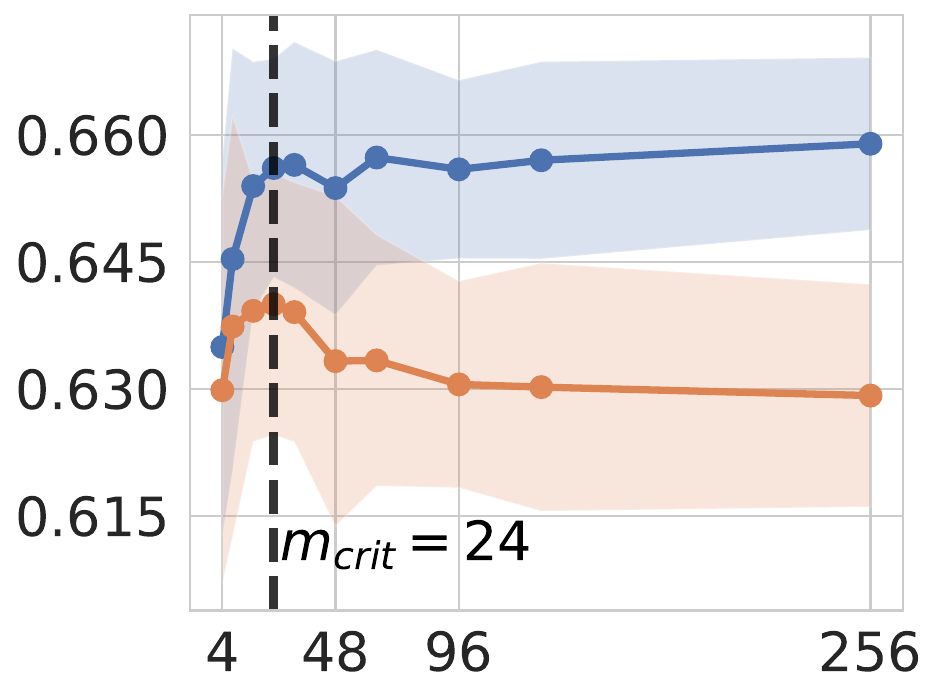}} &
\adjustbox{valign=c}{\includegraphics[width=0.48\linewidth]{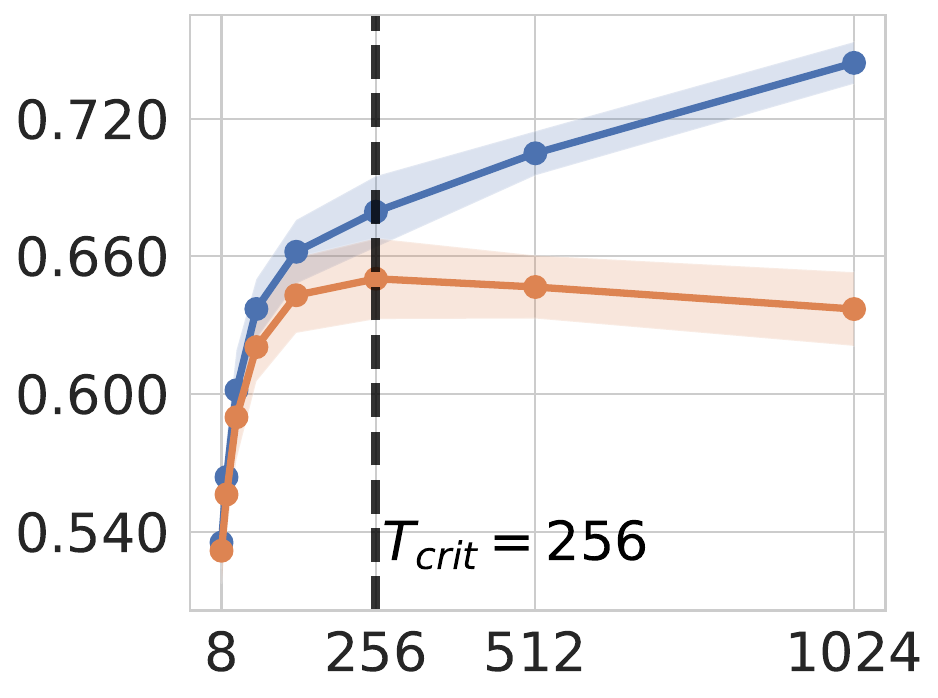}}

\\[4em]

\adjustbox{valign=c}{\rotatebox{90}{\footnotesize MNIST} \hspace{.2em}} &
&
\adjustbox{valign=c}{\hspace{-2mm}\includegraphics[width=0.48\linewidth,height=0.33\linewidth]{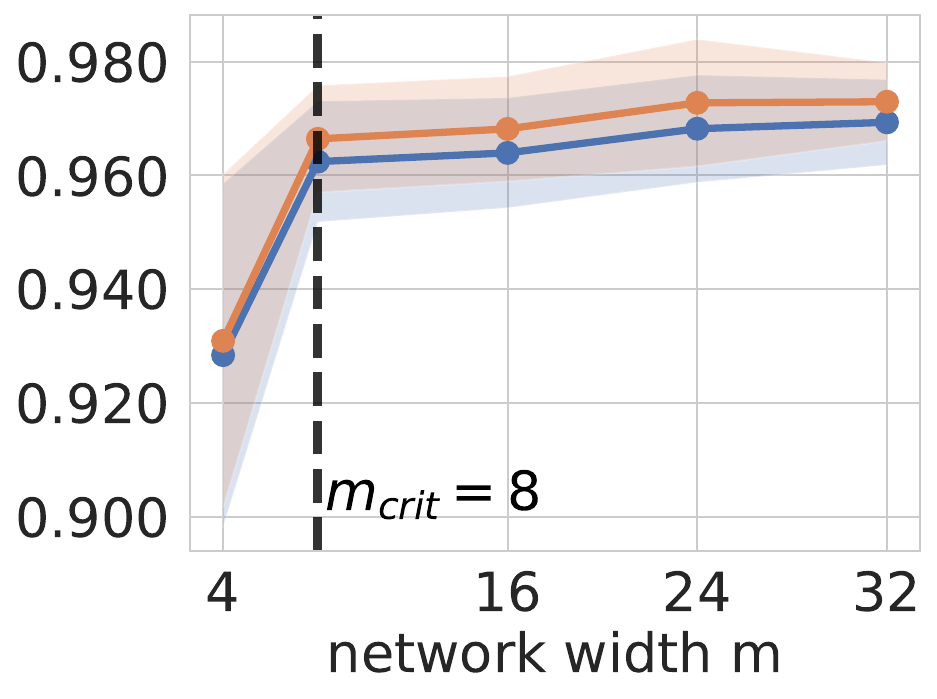}} &
\adjustbox{valign=c}{\includegraphics[width=0.48\linewidth,height=0.33\linewidth]{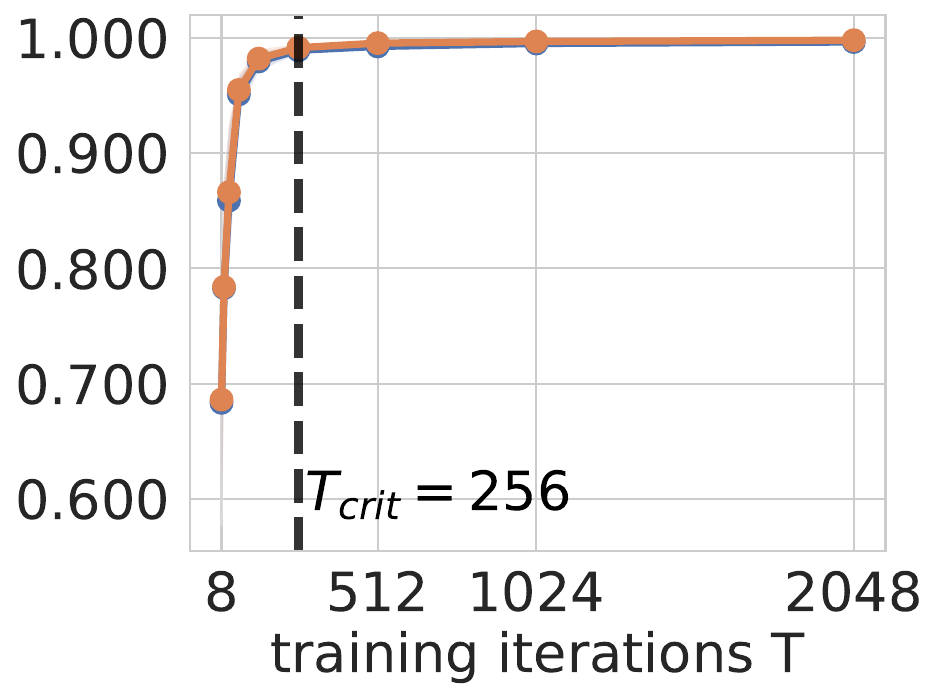}} 

\\[4em]

\multicolumn{4}{c}{\centering \includegraphics[height=1.5em]{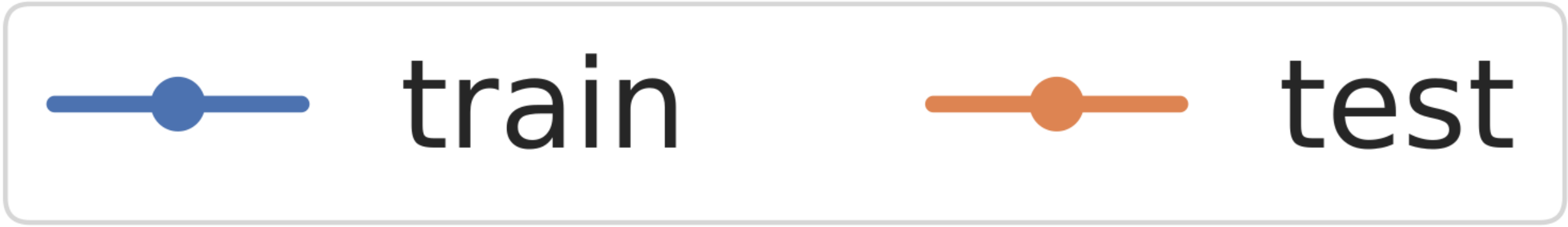}}
\end{tabular}
\end{subfigure}
\hspace{4em}
% ================= Subfigure: DPGD =================
\begin{subfigure}[t]{0.43\textwidth}
\centering
\caption{\text{DP-GD}}

\begin{tabular}{@{}c c c@{}}
%  \multicolumn{1}{c}{\footnotesize Utility vs. $m$}
% & \multicolumn{1}{c}{\footnotesize Utility vs. $T$} \\
% [-0em]

\multirow{3}{*}{\rotatebox{90}{\small \hspace{-5.5em} Private utility}} \hspace{-1em}
 &
\adjustbox{valign=c}{\includegraphics[width=0.48\linewidth]{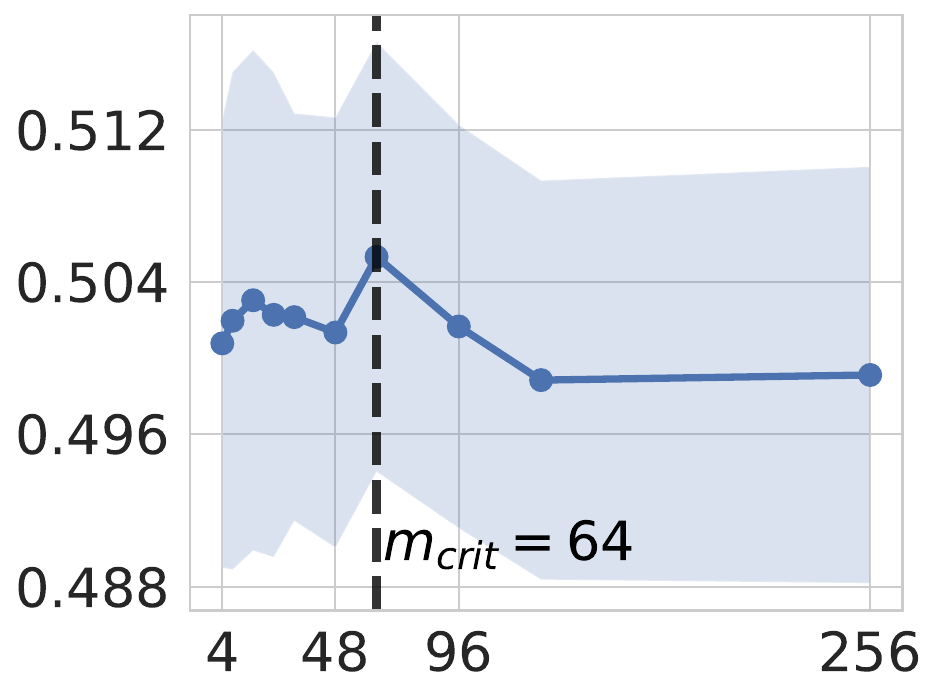}} &
\adjustbox{valign=c}{\includegraphics[width=0.48\linewidth]{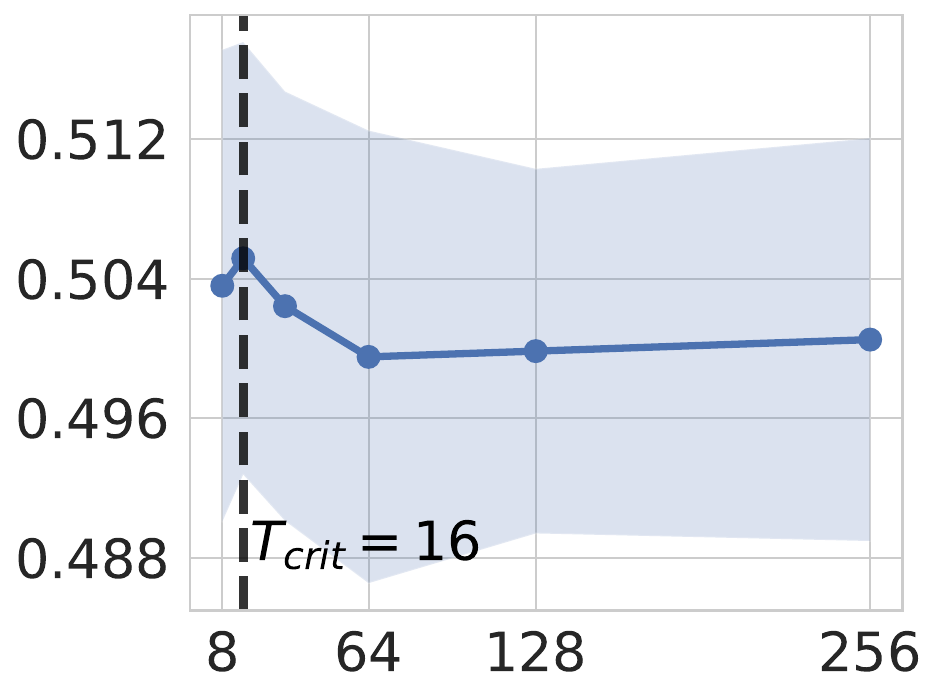}}
\\[4em]

 &
\adjustbox{valign=c}{\includegraphics[width=0.48\linewidth,height=0.33\linewidth]{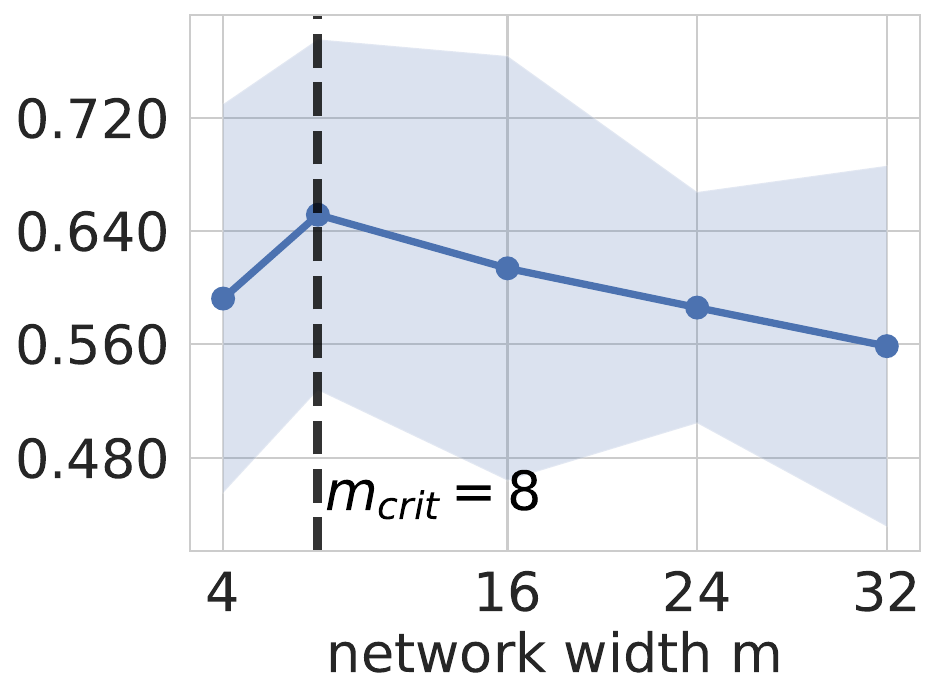}} &
\adjustbox{valign=c}{\includegraphics[width=0.48\linewidth,height=0.33\linewidth]{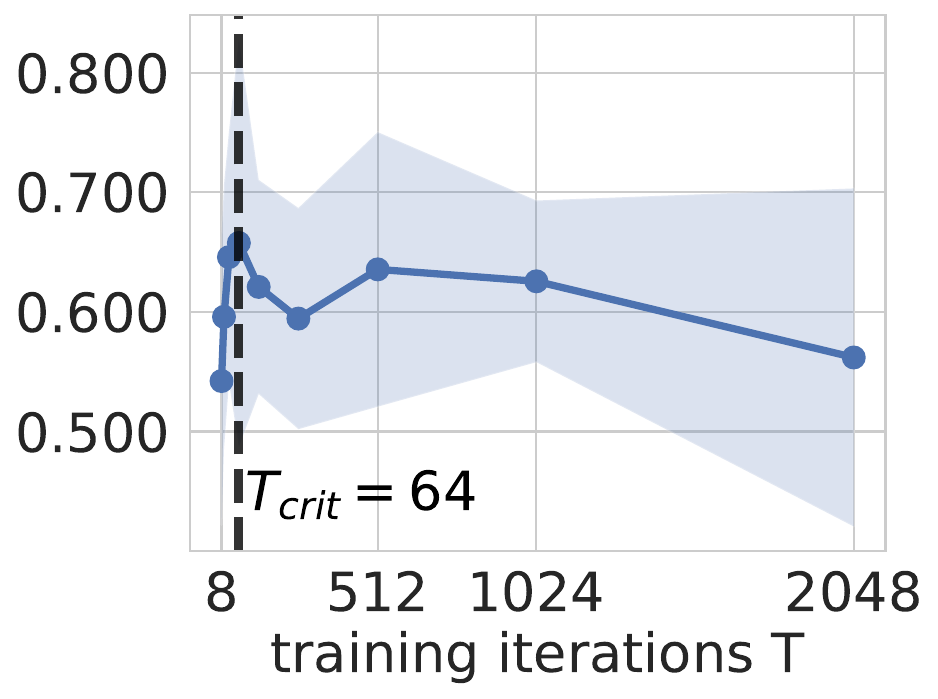}}

\\[4em]

\multicolumn{3}{c}{\centering \includegraphics[height=1.5em]{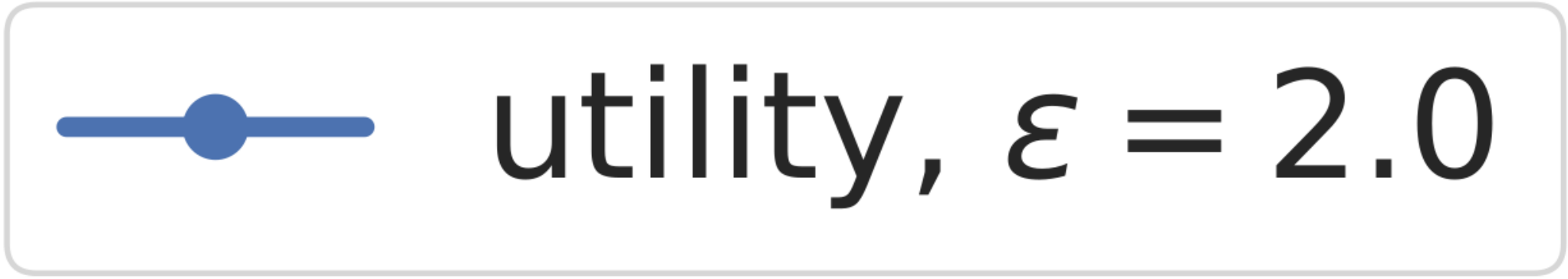}}
\end{tabular}
\end{subfigure}

\vspace{-1.5mm}
\caption{\footnotesize
(a) Training and test accuracy as a function of the width \(m\) (with \(T\) fixed) and the number of iterations \(T\) (with \(m\) fixed) for GD, and (b) private utility as a function of \(m\) and \(T\) for DP-GD, on synthetic logistic data (top row) and MNIST (bottom row).
In (a), the vertical dashed lines indicate empirically observed change points beyond which increasing the width \(m\) yields diminishing or flat training and test accuracy, while increasing the number of training iterations \(T\) primarily improves training accuracy but leads to diminishing or flat test accuracy, consistent with the theoretical bounds.
In (b), the dashed lines indicate observed turning points beyond which private utility degrades due to the amplification or accumulation of privacy noise, in line with the private utility bound.
Together, these results suggest selecting width and training duration within the admissible regimes identified by the theory.}
\label{fig:main-synth-mnist}
\vspace{-2mm}
\end{figure*}

We next specialize Theorem~\ref{thm:utility} to the NTK-separable setting, obtaining a utility rate of order $\mathcal{O}(\sqrt{d}/(n\epsilon))$. 
\begin{theorem}[Utility under NTK Separability]\label{thm:dp-risk-ntk}
Let Assumption~\ref{ass:ntk} holds. Assume $\eta \lesssim 1$ be a constant, $\delta
\lesssim \frac{\sqrt d}{n\epsilon }$, $\eta T \asymp \frac{\gamma^2n\epsilon}{\sqrt{d}\log^{5/2}(n/\delta)}$ and $m \asymp \log^6(n/\delta)/\gamma^6$.
Then, with probability at least $1-\delta$ over initialization,
\vspace{-1mm}
\begin{align*}
    \frac{1}{T}\sum_{k=1}^T   \E_{S,\A}\big[ \mathcal{L}(\widetilde{\Theta}(k))&\big]    \lesssim   \log^6\big(\frac{n}{\delta}\big)  \frac{\sqrt{d}}{\gamma^4n\epsilon}.  
\end{align*}
\end{theorem}
 
\section{Practical Implications and Experiments}\label{sec:experiments}
In this section, we empirically validate our theoretical results and translate them into concrete practical guidance for achieving optimal inference performance.

\vspace{-1mm}
\subsection{Practical Implications for Training KANs Using GD}
\vspace{-1mm}
Our theory provides explicit recommendations for choosing the network width \(m\) and the number of iterations \(T\), identifying critical regimes for both GD and DP-GD.

\vspace{-2mm}
\paragraph{Training and test accuracy saturate beyond critical width.}
Both the optimization bound (Theorem~\ref{thm:ntk}) and the generalization bound (Theorem~\ref{thm:ntk-gen}) indicate that increasing the width \(m\) improves accuracy only up to a critical threshold $m_{\textrm{crit}}$ (typically on the order of \(\log (n)\)). Beyond this regime, both training and test accuracy plateau and may exhibit fluctuations,  suggesting that further increasing the network width is unnecessary and offers limited benefit.

\vspace{-2mm}
\paragraph{Training accuracy may continue to improve with training iterations.}
The optimization bound furthermore predicts that training accuracy may continue to improve as the number of GD iterations \(T\) increases. However, this suggests that additional iterations primarily benefit optimization rather than generalization.

\vspace{-2mm}
\paragraph{Test accuracy saturates with increasing training iterations.}
In contrast, the generalization bound indicates that test accuracy saturates beyond a critical threshold $T_{\textrm{crit}}$ and may even deteriorate. Consequently, running GD for substantially longer yields at best marginal additional generalization gains.

\vspace{-2mm}
\paragraph{Private utility is maximized when width and training iterations lie in an  admissible range.}
The utility bound of Theorem~\ref{thm:dp-risk-ntk} provides explicit guidance for choosing the network width \(m\) and the number of training iterations \(T\) under a fixed privacy budget. In particular, the bound suggests using a polylogarithmic width and an effective training horizon \(T\) on the order of \(n\epsilon / \sqrt{d}\) (up to logarithmic factors and problem-dependent constants) to achieve favorable utility. This implies that the width \(m\) should be sufficiently large to enable learning, but not so large that it excessively amplifies the injected noise. Similarly, the number of iterations \(T\) should be large enough to benefit optimization, yet small enough to avoid excessive accumulation of privacy noise, thereby motivating early stopping.

\subsection{Empirical Validation}
Figure~\ref{fig:main-synth-mnist} reports the accuracy of GD and DP-GD on synthetic logistic data and MNIST, supporting the practical implications of our theory. Experimental details and hyperparameter settings are provided in Appendix~\ref{sec:appen-experiment}.

Figure~\ref{fig:main-synth-mnist}(a) shows that, on both datasets, increasing the width \(m\) yields rapid gains in training and test accuracy in the small-width regime (up to approximately \(m_{\text{crit}} \approx 24\)), after which accuracy diminishes and the performance curves become largely flat, consistent with the theoretical bounds. As the number of training iterations \(T\) increases, training accuracy improves (or remains near its peak), whereas test accuracy stabilizes much earlier. In particular, once \(T\) reaches a moderate scale (around \(T_{\text{crit}} \approx 256\)), further iterations provide little additional improvement in test accuracy.

Figure~\ref{fig:main-synth-mnist}(b) shows that for fixed privacy budget $\epsilon=2$, private utility (test accuracy at the last iterate) increases with either \(m\) or \(T\) up to an observed turning point (up to approximately \(m_{\text{crit}} \approx 64\) and \(T_{\text{crit}} \approx 64\)), after which performance degrades, highlighting the importance of moderate widths and early stopping. On MNIST, private utility peaks around \(m \approx 8\) and $T \approx 64$. Overall, the empirical trends align well with our theoretical implications.

\begin{table}[t]
\centering
\caption{Empirical NTK margin on nonlinear synthetic data.}
\label{tab:ntk-margin}
\small
\begin{tabular}{llccc}
\toprule
Model & Data & $m=8$ & $m=64$ & $m=512$ \\
\midrule
ReLU & Multi-int. & 0.0151 & 0.0191 & 0.0197 \\
KAN  & Multi-int. & \textbf{0.0400} & \textbf{0.0883} & \textbf{0.1027} \\
ReLU & XOR        & 0.0226 & 0.0285 & 0.0284 \\
KAN  & XOR        & \textbf{0.0243} & \textbf{0.0503} & \textbf{0.0599} \\
ReLU & Checker.   & 0.0060 & 0.0104 & 0.0111 \\
KAN  & Checker.   & \textbf{0.0206} & \textbf{0.0451} & \textbf{0.0538} \\
\bottomrule
\end{tabular}
\end{table}

\subsection{Separability Validation}
To further examine the practical achievability of Assumption~\ref{ass:ntk}, we conduct an additional empirical study on synthetic datasets that are not linearly separable in the input space.
For KAN, we use the same two-layer model as in \eqref{eq:KAN}. As a baseline, we consider a parameter-matched two-layer ReLU network of the form
\[
f_{\mathrm{ReLU}}(x)=\frac{1}{\sqrt{h}}\sum_{j=1}^h c_j \, \mathrm{ReLU} \Big(\frac{\langle \ba_j,\bx\rangle}{\sqrt d}\Big),
\]
with Gaussian initialization \(\ba_j \sim \mathcal{N}(0,\bfI_d)\) and \(c_j \sim \mathcal{N}(0,1)\). For a KAN with hidden width \(m\), we set the ReLU width to \(h=pm\), where \(p\) is the number of spline basis functions, so that the two models have the same number of parameters.

We consider three structured nonlinear datasets. In the \emph{multi-interval one-dimensional} dataset, the label is determined by a non-monotone rule on a single coordinate: points with \(x_1\) in two outer intervals are assigned one label, while points in a middle interval are assigned the other. In the \emph{XOR-gap} dataset, the label is given by the sign pattern of two active coordinates, with a margin gap excluding points near the coordinate axes. In the \emph{checkerboard-gap} dataset, the label is determined by the parity of the signs of several active coordinates, again with a gap around the decision boundaries. In all cases, the remaining coordinates are sampled as nuisance dimensions.

For each model and width, we compute the empirical NTK margin at random initialization by solving the corresponding max-margin problem in the tangent feature space,
\[
\max_{\|\Theta\|_2 \le 1}\min_{i\in[n]} y_i \langle \nabla f_{\Theta(0)}(x_i), \Theta \rangle .
\]
We repeat the experiment over multiple random initializations and report the mean margin. Table~\ref{tab:ntk-margin} summarizes the results for representative widths \(m\in\{8,64,512\}\). On all three nonlinear datasets, the empirical NTK margin of KAN is positive and increases with width. Moreover, it is consistently larger than that of the parameter-matched ReLU network, with the gap becoming more pronounced at larger widths. While these results do not provide a complete characterization of Assumption~\ref{ass:ntk}, they offer additional evidence that NTK separability is practically plausible for structured nonlinear data aligned with the KAN architecture.

\section{Related Work}\label{sec:related-work}
There are only few theoretical results analyzing the training dynamics of KAN under gradient methods. In this section, we discuss related work on KANs and differentially private optimization, and position our results relative to existing guarantees. Additional discussion on analyses for MLPs is deferred to Appendix~\ref{sec:appen-relatedMLP}.

\paragraph{Gradient-based analysis for KANs.} Besides \cite{gao2025convergence}, \cite{wang2025expressiveness} studies expressiveness and GD dynamics via spectral bias in a simplified setting, which is orthogonal to our guarantees. Most existing generalization results for KANs are \emph{algorithm-independent}, e.g., \cite{zhang2025generalization,li2025generalization,liu2025rate} rely on uniform-convergence-style arguments and therefore do not yield bounds tailored to GD iterates.

\paragraph{Utility analysis for DP-GD on MLPs.} Rigorous utility analysis for DP-GD and DP-SGD on MLPs remain limited \cite{bu2023convergence,romijnders2024convex,shi2025towards,wang2025optimaldp}. \cite{wang2025optimaldp} derive utility bounds for smooth three-layer MLPs in an  over-scaled regime, and explicitly leave the standard $1/\sqrt{m}$ scaling as an open problem.
Other works rely on convex reformulations or alternative threat models \cite{romijnders2024convex}, or provide task-specific evidence that noise can act as an implicit regularizer \citep{shi2025towards}.
There are also works on DP optimization for nonconvex objectives, which mainly focus on stationarity or empirical objective gaps \cite{zhang2017efficient,wang2019differentially,bassily2021differentially,lowy2024make}.

\vspace{-1mm}
\section{Conclusion and Limitations}\label{sec:conclu}
We provide optimization, generalization, and private utility bounds for training KANs with (DP-)GD, identifying regimes of width and training duration beyond which performance saturates or degrades. These insights offer concrete guidance for choosing model size and training horizons and are supported by experiments on synthetic data and MNIST. Several limitations point to future work: extending the analysis beyond two-layer KANs, relaxing the smoothness assumption on $\sigma$ to cover non-smooth activations such as ReLU, and generalizing our guarantees from GD to SGD.

\section*{Impact Statement}
This paper advances the theoretical understanding of training Kolmogorov--Arnold Networks under differential privacy. 
Our results provide explicit $(\epsilon,\delta)$-DP guarantees and corresponding utility bounds, which can support the principled use of gradient-based training on sensitive datasets (e.g., in biology and medicine) by limiting the influence of any single individual’s record. 
At the same time, our bounds make explicit the privacy--utility trade-off and the dependence of performance on design choices such as width and early stopping. 
As with any DP method, these guarantees mitigate but do not eliminate privacy risks, and practitioners should select privacy parameters $(\epsilon,\delta)$ and validate models in the intended deployment context.

\section*{Acknowledgment}PW acknowledges support by the Alexander-von-Humboldt Foundation through a Humboldt Research Fellowship. MK acknowledges support by the DFG through FOR 5359 (ID 459419731), TRR 375 (ID 511263698), SPP 2298 (ID 464252197), and SPP 2331 (ID 441958259), by the Carl-Zeiss Foundation through the initiative AI-Care, and by the BMFTR award 01IS24071A.

\nocite{langley00}

\bibliography{learning}

@article{LiMolecular,
author = {Li, Longlong and Zhang, Yipeng and Wang, Guanghui and Xia, Kelin},
year = {2025},
month = {08},
pages = {1346-1354},
title = {Kolmogorov–Arnold graph neural networks for molecular property prediction},
volume = {7},
journal = {Nature Machine Intelligence},
doi = {10.1038/s42256-025-01087-7}
}

@article{CherednichenkoP25,
  author       = {Oleksandr Cherednichenko and
                  Maria S. Poptsova},
  title        = {Kolmogorov-Arnold networks for genomic tasks},
  journal      = {Briefings Bioinform.},
  volume       = {26},
  number       = {2},
  year         = {2025},
  doi          = {10.1093/BIB/BBAF129}
}

@inproceedings{liu2025kan,
  title={Kan: Kolmogorov-arnold networks},
  author={Liu, Ziming and Wang, Yixuan and Vaidya, Sachin and Ruehle, Fabian and Halverson, James and Solja{\v{c}}i{\'c}, Marin and Hou, Thomas Y and Tegmark, Max},
   booktitle={International Conference on Learning Representations},
  year={2025}
}

@inproceedings{zhang2017efficient,
  title={Efficient private ERM for smooth objectives},
  author={Zhang, Jiaqi and Zheng, Kai and Mou, Wenlong and Wang, Liwei},
  booktitle={Proceedings of the 26th International Joint Conference on Artificial Intelligence},
  pages={3922--3928},
  year={2017}
}

@inproceedings{qiu2025finding,
  title={Finding Local Diffusion Schrodinger Bridge using Kolmogorov-Arnold Network},
  author={Qiu, Xingyu and Yang, Mengying and Ma, Xinghua and Li, Fanding and Liang, Dong and Luo, Gongning and Wang, Wei and Wang, Kuanquan and Li, Shuo},
  booktitle={Proceedings of the Computer Vision and Pattern Recognition Conference},
  pages={23227--23236},
  year={2025}
}

@inproceedings{su2025kan,
  title={KAN-DDPM: Kolmogorov-Arnold networks with diffusion denoising probabilistic models for MRI-to-CT synthesis},
  author={Su, Vanessa and Yang, Xiaofeng},
  booktitle={Medical Imaging 2025: Imaging Informatics},
  volume={13411},
  pages={227--233},
  year={2025},
  organization={SPIE}
}

@inproceedings{xiong2025conditional,
  title={A Conditional KAN Diffusion Network for Human Activity Recognition with Missing Sensor Signal Series},
  author={Xiong, Hao and Gong, Jiayi and Luo, Haiyong and Zhao, Fang and Gao, Yang and Chen, Runze and Xiao, Mingyu},
  booktitle={ICASSP 2025-2025 IEEE International Conference on Acoustics, Speech and Signal Processing (ICASSP)},
  pages={1--5},
  year={2025},
  organization={IEEE}
}

@article{genet2024temporal,
  title={A temporal kolmogorov-arnold transformer for time series forecasting},
  author={Genet, Remi and Inzirillo, Hugo},
  journal={arXiv preprint arXiv:2406.02486},
  year={2024}
}

@inproceedings{ferdaus2024kanice,
  title={Kanice: Kolmogorov-arnold networks with interactive convolutional elements},
  author={Ferdaus, Md Meftahul and Abdelguerfi, Mahdi and Ioup, Elias and Dobson, David and Niles, Kendall N and Pathak, Ken and Sloan, Steven},
  booktitle={Proceedings of the 4th International Conference on AI-ML Systems},
  pages={1--10},
  year={2024}
}

@article{bodner2024convolutional,
  title={Convolutional kolmogorov-arnold networks},
  author={Bodner, Alexander Dylan and Tepsich, Antonio Santiago and Spolski, Jack Natan and Pourteau, Santiago},
  journal={arXiv preprint arXiv:2406.13155},
  year={2024}
}

@article{ranasinghe2024ginn,
  title={Ginn-kan: Interpretability pipelining with applications in physics informed neural networks},
  author={Ranasinghe, Nisal and Xia, Yu and Seneviratne, Sachith and Halgamuge, Saman},
  journal={arXiv preprint arXiv:2408.14780},
  year={2024}
}

@article{kolmogorov1963representation,
  title={On the representation of continuous functions of many variables by superposition of continuous functions of one variable and addition},
  author={Kolmogorov, Andre{\u{\i}} Nikolaevich},
  journal={Translations American Mathematical Society},
  volume={2},
  number={28},
  pages={55--59},
  year={1963}
}

@inproceedings{arnol1957functions,
  title={On functions of three variables},
  author={Arnol'd, Vladimir Igorevich},
  booktitle={Doklady Akademii Nauk},
  volume={114},
  number={4},
  pages={679--681},
  year={1957},
  organization={Russian Academy of Sciences}
}

@article{erdmann2025kan,
  title={KAN We Improve on HEP Classification Tasks? Kolmogorov--Arnold Networks Applied to an LHC Physics Example},
  author={Erdmann, Johannes and Mausolf, Florian and Sp{\"a}h, Jan Lukas},
  journal={Computing and Software for Big Science},
  volume={9},
  number={1},
  pages={9},
  year={2025},
  publisher={Springer}
}

@article{somvanshi2025survey,
  title={A survey on kolmogorov-arnold network},
  author={Somvanshi, Shriyank and Javed, Syed Aaqib and Islam, Md Monzurul and Pandit, Diwas and Das, Subasish},
  journal={ACM Computing Surveys},
  volume={58},
  number={2},
  pages={1--35},
  year={2025},
  publisher={ACM New York, NY}
}

@article{braun2009constructive,
  title={On a constructive proof of Kolmogorov’s superposition theorem},
  author={Braun, J{\"u}rgen and Griebel, Michael},
  journal={Constructive approximation},
  volume={30},
  number={3},
  pages={653--675},
  year={2009},
  publisher={Springer}
}

@article{vaca2024kolmogorov,
  title={Kolmogorov-arnold networks (kans) for time series analysis},
  author={Vaca-Rubio, Cristian J and Blanco, Luis and Pereira, Roberto and Caus, M{\`a}rius},
  journal={arXiv preprint arXiv:2405.08790},
  year={2024}
}

@article{shukla2024comprehensive,
  title={A comprehensive and FAIR comparison between MLP and KAN representations for differential equations and operator networks},
  author={Shukla, Khemraj and Toscano, Juan Diego and Wang, Zhicheng and Zou, Zongren and Karniadakis, George Em},
  journal={Computer Methods in Applied Mechanics and Engineering},
  volume={431},
  pages={117290},
  year={2024},
  publisher={Elsevier}
}

@article{patra2025physics,
  title={Physics Informed Kolmogorov-Arnold Neural Networks for Dynamical Analysis via Efficient-KAN and WAV-KAN},
  author={Patra, Subhajit and Panda, Sonali and Parida, Bikram Keshari and Arya, Mahima and Jacobs, Kurt and Bondar, Denys I and Sen, Abhijit},
  journal={Journal of Machine Learning Research},
  volume={26},
  number={233},
  pages={1--39},
  year={2025}
}

@article{wang2025kolmogorov,
  title={Kolmogorov--Arnold-Informed neural network: A physics-informed deep learning framework for solving forward and inverse problems based on Kolmogorov--Arnold Networks},
  author={Wang, Yizheng and Sun, Jia and Bai, Jinshuai and Anitescu, Cosmin and Eshaghi, Mohammad Sadegh and Zhuang, Xiaoying and Rabczuk, Timon and Liu, Yinghua},
  journal={Computer Methods in Applied Mechanics and Engineering},
  volume={433},
  pages={117518},
  year={2025},
  publisher={Elsevier}
}

@inproceedings{wang2025expressiveness,
  title={On the expressiveness and spectral bias of KANs},
  author={Wang, Yixuan and Siegel, Jonathan W and Liu, Ziming and Hou, Thomas Y},
  booktitle={International Conference on Learning Representations},
  year={2025}
}

@inproceedings{oymak2019overparameterized,
  title={Overparameterized nonlinear learning: Gradient descent takes the shortest path?},
  author={Oymak, Samet and Soltanolkotabi, Mahdi},
  booktitle={International Conference on Machine Learning},
  pages={4951--4960},
  year={2019},
  organization={PMLR}
}

@inproceedings{dwork2006calibrating,
  title={Calibrating noise to sensitivity in private data analysis},
  author={Dwork, Cynthia and McSherry, Frank and Nissim, Kobbi and Smith, Adam},
  booktitle={Theory of cryptography conference},
  pages={265--284},
  year={2006},
  organization={Springer}
}

@inproceedings{bassily2019private,
  title={Private stochastic convex optimization with optimal rates},
  author={Bassily, Raef and Feldman, Vitaly and Talwar, Kunal and Guha Thakurta, Abhradeep},
  booktitle={Advances in neural information processing systems},
  volume={32},
  year={2019}
}

@article{wang2025generalization,
  title={Generalization guarantees of gradient descent for shallow neural networks},
  author={Wang, Puyu and Lei, Yunwen and Wang, Di and Ying, Yiming and Zhou, Ding-Xuan},
  journal={Neural Computation},
  volume={37},
  number={2},
  pages={344--402},
  year={2025},
  publisher={MIT Press 255 Main Street, 9th Floor, Cambridge, Massachusetts 02142, USA~…}
}

@article{bu2023convergence,
  title={On the convergence and calibration of deep learning with differential privacy},
  author={Bu, Zhiqi and Wang, Hua and Dai, Zongyu and Long, Qi},
  journal={Transactions on machine learning research},
  volume={2023},
  pages={https--openreview},
  year={2023}
}

@inproceedings{wang2019differentially,
  title={Differentially private empirical risk minimization with non-convex loss functions},
  author={Wang, Di and Chen, Changyou and Xu, Jinhui},
  booktitle={International Conference on Machine Learning},
  pages={6526--6535},
  year={2019},
  organization={PMLR}
}

@inproceedings{song2013stochastic,
  title={Stochastic gradient descent with differentially private updates},
  author={Song, Shuang and Chaudhuri, Kamalika and Sarwate, Anand D},
  booktitle={2013 IEEE global conference on signal and information processing},
  pages={245--248},
  year={2013},
  organization={IEEE}
}

@inproceedings{bassily2021differentially,
  title={Differentially private stochastic optimization: New results in convex and non-convex settings},
  author={Bassily, Raef and Guzm{\'a}n, Crist{\'o}bal and Menart, Michael},
  booktitle={Advances in Neural Information Processing Systems},
  volume={34},
  pages={9317--9329},
  year={2021}
}

@article{wang2022differentially,
  title={Differentially private SGD with non-smooth losses},
  author={Wang, Puyu and Lei, Yunwen and Ying, Yiming and Zhang, Hai},
  journal={Applied and Computational Harmonic Analysis},
  volume={56},
  pages={306--336},
  year={2022},
  publisher={Elsevier}
}

@article{dwork2014algorithmic,
  title={The algorithmic foundations of differential privacy},
  author={Dwork, Cynthia and Roth, Aaron and others},
  journal={Foundations and Trends{\textregistered} in Theoretical Computer Science},
  volume={9},
  number={3--4},
  pages={211--407},
  year={2014},
  publisher={Now Publishers, Inc.}
}

@inproceedings{schliserman2022stability,
  title={Stability vs implicit bias of gradient methods on separable data and beyond},
  author={Schliserman, Matan and Koren, Tomer},
  booktitle={Conference on Learning Theory},
  pages={3380--3394},
  year={2022},
  organization={PMLR}
}

@article{gao2025convergence,
  title={On the convergence of (stochastic) gradient descent for kolmogorov--arnold networks},
  author={Gao, Yihang and Tan, Vincent YF},
  journal={IEEE Transactions on Information Theory},
  year={2025},
  publisher={IEEE}
}

@article{zou2020gradient,
  title={Gradient descent optimizes over-parameterized deep ReLU networks},
  author={Zou, Difan and Cao, Yuan and Zhou, Dongruo and Gu, Quanquan},
  journal={Machine learning},
  volume={109},
  pages={467--492},
  year={2020},
  publisher={Springer}
}

@inproceedings{taheri2024sharper,
  title={Sharper Guarantees for Learning Neural Network Classifiers with Gradient Methods},
  author={Taheri, Hossein and Thrampoulidis, Christos and Mazumdar, Arya},
  booktitle={International Conference on Learning Representations},
  year={2025}
}

@article{nitanda2019gradient,
  title={Gradient descent can learn less over-parameterized two-layer neural networks on classification problems},
  author={Nitanda, Atsushi and Chinot, Geoffrey and Suzuki, Taiji},
  journal={arXiv preprint arXiv:1905.09870},
  year={2019}
}

@article{frei2023random,
  title={Random feature amplification: Feature learning and generalization in neural networks},
  author={Frei, Spencer and Chatterji, Niladri S and Bartlett, Peter L},
  journal={Journal of Machine Learning Research},
  volume={24},
  number={303},
  pages={1--49},
  year={2023}
}

@article{zhou2024generalization,
  title={Generalization analysis with deep ReLU networks for metric and similarity learning},
  author={Zhou, Junyu and Wang, Puyu and Zhou, Ding-Xuan},
  journal={arXiv preprint arXiv:2405.06415},
  year={2024}
}

@article{taheri2024generalization,
  title={Generalization and Stability of Interpolating Neural Networks with Minimal Width},
  author={Taheri, Hossein and Thrampoulidis, Christos},
  journal={Journal of Machine Learning Research},
  volume={25},
  number={156},
  pages={1--41},
  year={2024}
}

@inproceedings{chen2021much,
  title={How much over-parameterization is sufficient to learn deep ReLU networks?},
  author={Chen, Zixiang and Cao, Yuan and Zou, Difan and Gu, Quanquan},
  booktitle={International Conference on Learning Representation},
  year={2021}
}

@inproceedings{nitanda2021optimal,
  title={Optimal rates for averaged stochastic gradient descent under neural tangent kernel regime},
  author={Nitanda, Atsushi and Taiji, Suzuki},
  booktitle={International Conference on Learning Representations},
  year={2021}
}

@article{nguyen2024many,
  title={How many neurons do we need? A refined analysis for shallow networks trained with gradient descent},
  author={Nguyen, Mike and M{\"u}cke, Nicole},
  journal={Journal of Statistical Planning and Inference},
  pages={106169},
  year={2024},
  publisher={Elsevier}
}

@inproceedings{ji2019polylogarithmic,
  title={Polylogarithmic width suffices for gradient descent to achieve arbitrarily small test error with shallow relu networks},
  author={Ji, Ziwei and Telgarsky, Matus},
  booktitle={ International Conference on Learning Representations},
  year={2020}
}

@book{wainwright2019high,
  title={High-dimensional statistics: A non-asymptotic viewpoint},
  author={Wainwright, Martin J},
  volume={48},
  year={2019},
  publisher={Cambridge university press}
}

@inproceedings{lowy2024make,
  title={How to make the gradients small privately: Improved rates for differentially private non-convex optimization},
  author={Lowy, Andrew and Ullman, Jonathan and Wright, Stephen J},
  booktitle={International Conference on Machine Learning},
  year={2024}
}

@article{jacot2018neural,
  title={Neural tangent kernel: Convergence and generalization in neural networks},
  author={Jacot, Arthur and Gabriel, Franck and Hongler, Cl{\'e}ment},
  journal={Advances in Neural Information Processing Systems},
  volume={31},
  year={2018}
}

@inproceedings{allen2019convergence,
  title={A convergence theory for deep learning via over-parameterization},
  author={Allen-Zhu, Zeyuan and Li, Yuanzhi and Song, Zhao},
  booktitle={International Conference on Machine Learning},
  pages={242--252},
  year={2019},
  organization={PMLR}
}

@inproceedings{du2019gradient,
  title={Gradient descent finds global minima of deep neural networks},
  author={Du, Simon and Lee, Jason and Li, Haochuan and Wang, Liwei and Zhai, Xiyu},
  booktitle={International Conference on Machine Learning},
  pages={1675--1685},
  year={2019},
  organization={PMLR}
}

@article{bartlett2017spectrally,
  title={Spectrally-normalized margin bounds for neural networks},
  author={Bartlett, Peter L and Foster, Dylan J and Telgarsky, Matus J},
  journal={Advances in Neural Information Processing Systems},
  volume={30},
  year={2017}
}

@inproceedings{arora2019fine,
  title={Fine-grained analysis of optimization and generalization for overparameterized two-layer neural networks},
  author={Arora, Sanjeev and Du, Simon and Hu, Wei and Li, Zhiyuan and Wang, Ruosong},
  booktitle={International Conference on Machine Learning},
  pages={322--332},
  year={2019},
  organization={PMLR}
}

@inproceedings{mironov2017renyi,
  title={R{\'e}nyi differential privacy},
  author={Mironov, Ilya},
  booktitle={2017 IEEE 30th computer security foundations symposium (CSF)},
  pages={263--275},
  year={2017},
  organization={IEEE}
}

@inproceedings{li2025generalization,
  title={Generalization Bounds for Kolmogorov-Arnold Networks (KANs) and Enhanced KANs with Lower Lipschitz Complexity},
  author={Li, Pengqi and Ding, Lizhong and Fu, Jiarun and Wang, Guoren and Yuan, Ye and others},
  booktitle={The Thirty-ninth Annual Conference on Neural Information Processing Systems},
  year={2025}
}

@article{zhou2026fine,
  title={Fine-Grained Analysis of Nonparametric Estimation for Pairwise Learning},
  author={Zhou, Junyu and Huang, Shuo and Feng, Han and Wang, Puyu and Zhou, Ding-Xuan},
  journal={IEEE Transactions on Neural Networks and Learning Systems},
  year={2026},
  publisher={IEEE}
}

@inproceedings{shi2025towards,
  title={Towards Understanding Generalization in DP-GD: A Case Study in Training Two-Layer CNNs},
  author={Shi, Zhongjie and Wang, Puyu and Zhang, Chenyang and Cao, Yuan},
  booktitle={The Fortieth AAAI Conference on Artificial Intelligence},
  year={2026}
}

@inproceedings{romijnders2024convex,
  title={Convex Approximation of Two-Layer ReLU Networks for Hidden State Differential Privacy},
  author={Romijnders, Rob and Koskela, Antti},
  booktitle={Advances in Neural Information Processing Systems},
  year={2025}
}

@inproceedings{wang2025optimaldp,
  title={Optimal Utility Bounds for Differentially Private Gradient Descent in Three-Layer Neural Networks},
  author={Wang, Puyu and Lei, Yunwen and Kloft, Marius and Ying, Yiming},
  booktitle={2025 IEEE 12th International Conference on Data Science and Advanced Analytics (DSAA)},
  pages={1--8},
  year={2025},
  organization={IEEE}
}

@article{liu2025rate,
  title={On the rate of convergence of Kolmogorov-Arnold Network regression estimators},
  author={Liu, Wei and Chatzi, Eleni and Lai, Zhilu},
  journal={arXiv preprint arXiv:2509.19830},
  year={2025}
}

@inproceedings{zhang2025generalization,
  title={Generalization bounds and model complexity for kolmogorov-arnold networks},
  author={Zhang, Xianyang and Zhou, Huijuan},
  booktitle={The Thirteenth International Conference on Learning Representations},
  year={2025}
}

@inproceedings{cao2019generalization,
  title={Generalization bounds of stochastic gradient descent for wide and deep neural networks},
  author={Cao, Yuan and Gu, Quanquan},
  booktitle={Advances in Neural Information Processing Systems},
  volume={32},
  year={2019}
}

@inproceedings{lei2022stability,
   title={Stability and Generalization Analysis of Gradient Methods for Shallow Neural Networks},
  author={Lei, Yunwen and Jin, Rong and Ying, Yiming},
 booktitle={Advances in Neural Information Processing Systems},
  volume={35},
  year={2022},
   organization={PMLR}
}

@inproceedings{richards2021stability,
  title={Stability \& Generalisation of Gradient Descent for Shallow Neural Networks without the Neural Tangent Kernel},
  author={Richards, Dominic and Kuzborskij, Ilja},
 booktitle={Advances in Neural Information Processing Systems},
  volume={34},
  year={2021},
   organization={PMLR}
}

@article{lei2026optimization,
  title={Optimization and generalization of gradient descent for shallow relu networks with minimal width},
  author={Lei, Yunwen and Wang, Puyu and Ying, Yiming and Zhou, Ding-Xuan},
  journal={Journal of Machine Learning Research},
  volume={27},
  number={34},
  pages={1--35},
  year={2026}
}

@inproceedings{li2025optimal,   title={Optimal Rates for Generalization of Gradient Descent for Deep ReLU Classification},   author={Li, Yuanfan and Lei, Yunwen and Guo, Zheng-Chu and Ying, Yiming},   booktitle={Advances in Neural Information Processing Systems},   year={2025} }

@inproceedings{abadi2016deep,
  title={Deep learning with differential privacy},
  author={Abadi, Martin and Chu, Andy and Goodfellow, Ian and McMahan, H Brendan and Mironov, Ilya and Talwar, Kunal and Zhang, Li},
  booktitle={Proceedings of the 2016 ACM SIGSAC Conference on Computer and Communications Security},
  pages={308--318},
  year={2016}
}

@book{vershynin2018high,
  title={High-dimensional probability: An introduction with applications in data science},
  author={Vershynin, Roman},
  year={2018},
  publisher={Cambridge university press}
}

@inproceedings{lei2020fine,
  title={Fine-Grained Analysis of Stability and Generalization for Stochastic Gradient Descent},
  author={Lei, Yunwen and Ying, Yiming},
  booktitle={International Conference on Machine Learning},
  pages={5809--5819},
  year={2020}
}

@article{mnist,
  title={The {MNIST} database of handwritten digit images for machine learning research},
  author={Deng, Li},
  journal={{IEEE} Signal Processing Magazine},
  volume={29},
  number={6},
  pages={141--142},
  year={2012},
  publisher={IEEE}
}
\bibliographystyle{plain}

%%%%%%%%%%%%%%%%%%%%%%%%%%%%%%%%%%%%%%%%%%%%%%%%%%%%%%%%%%%%%%%%%%%%%%%%%%%%%%%
%%%%%%%%%%%%%%%%%%%%%%%%%%%%%%%%%%%%%%%%%%%%%%%%%%%%%%%%%%%%%%%%%%%%%%%%%%%%%%%
% APPENDIX
%%%%%%%%%%%%%%%%%%%%%%%%%%%%%%%%%%%%%%%%%%%%%%%%%%%%%%%%%%%%%%%%%%%%%%%%%%%%%%%
%%%%%%%%%%%%%%%%%%%%%%%%%%%%%%%%%%%%%%%%%%%%%%%%%%%%%%%%%%%%%%%%%%%%%%%%%%%%%%%
\newpage
\appendix
\onecolumn

\begin{center}
   {\Large \bf  Appendix}
\end{center}
\section{Further Related work for MLPs}\label{sec:appen-relatedMLP}  
Most existing optimization and generalization results on MLPs focuses on \textit{fully-connected} architectures. 
A central tool is the NTK \cite{jacot2018neural}, which has enabled global convergence analyses and, in the lazy-training regime, sharp statistical rates in various settings \cite{arora2019fine,cao2019generalization,allen2019convergence,zou2020gradient,chen2021much,nguyen2024many,nitanda2021optimal}, often requiring widths that scale polynomially with $n$.
In parallel, generalization has been studied via uniform convergence using capacity measures such as Rademacher complexity and covering numbers 
\cite{bartlett2017spectrally,frei2023random,ji2019polylogarithmic,nitanda2019gradient,lei2026optimization,li2025optimal,zhou2024generalization,zhou2026fine}.
More recently, there has been growing interest in generalization analysis based on algorithmic stability \cite{richards2021stability,lei2022stability,taheri2024sharper,taheri2024generalization,wang2025generalization}, which provides algorithm-dependent generalization bounds directly tied to the training dynamics.

%\section{Two Examples}\label{sec:app-example}In this section, we give two example where the condition $g(\epsilon)\lesssim \log(1/\epsilon)$ holds. \textit{Example 1 (Linearly separable data).} Consider the logistic loss $\ell(a)=\log(1+\exp(-a))$, tanh transformation function $\sigma(u)=\frac{e^u-e^{-u}}{e^u+e^{-u}}$ and a linearly separable data with margin $\gamma>0$, i.e., there exists a unit-norm vector $\bv^*\in\rbb^d$ such that $\min_{i\in[n]}y_i\bx_i^\top\bv^*\geq\gamma$. We now show Assumption~\ref{ass:realizability} holds under this linear separability assumption. (this part may be moved to appendix; to be done) 

\section{Useful Lemmas}
In this section, we introduce several useful lemmas that will be used in the proofs.
\begin{lemma}[Concentration of the norm of Sub-Gaussian vectors \cite{wainwright2019high}]\label{lem:subGaussian}
    Let $\bu\in\R^s$ be a centered Sub-Gaussian vector with parameter $\sigma^2$, i.e., $\ebb[u_i] = 0$ and $\ebb[e^{\lambda u_i}] \le e^{\lambda^2\sigma^2/2}$ for all $\lambda>0$ and $i\in[s]$.
    Then with probability at least $1-\delta$ for $\delta\in(0,1)$, it holds that
    \[ \|\bu\|_2 \le 4\sigma\sqrt{s} + 2\sigma\sqrt{\log(\frac{1}{\delta})}. \]
\end{lemma}

\begin{corollary}\label{cor:c}
    With probability at least $1 - \delta$ over initialization $\bc(0)$, it holds that
    \begin{align}\label{eq:bound_c0}
        \|\bc(0)\|_2 \le 4\sqrt{pm} +  2\sqrt{\log(\frac{2}{\delta})} \qquad \text{and} \qquad \max_{i\in[m]}\|\bc_i(0)\|_2 \le 4\sqrt{p} + 2\sqrt{\log(\frac{2m}{\delta})}.
    \end{align}
\end{corollary}
\begin{proof}
    Note $\bc(0)\sim\N(0, \bfI_{mp})$.
    Applying Lemma~\ref{lem:subGaussian} twice with $\bu = \bc(0)$, $s=mp$ and $\sigma^2=1$ and with $\bu=\bc_i(0)$, $s=p$ and $\sigma^2=1$ over all $i\in[m]$, respectively, we can obtain the desried results.
\end{proof}

\begin{lemma}\label{lem:block_matrix}
    Let $s\in\mathbb{N}$ and $\bA_i\in\R^{m_i\times n_i}$ for $i\in[s]$.
    Define the diagonal block matrix $$\bB = \begin{bmatrix}
        \bA_1 & \cdots & \mathbf{0}\\
        \vdots & \ddots & \vdots\\
        \mathbf{0} & \cdots & \bA_s
    \end{bmatrix} \in \R^{\sum_{i=1}^sm_i \times \sum_{i=1}^sn_i}$$
    Then, it holds that
    \begin{align*}
        \|\bB\|_2 = \max_{i\in[s]} \|\bA_i\|_2.
    \end{align*}
\end{lemma}
\begin{proof}
    Since the operator norm of any PSD matrix $\bA$ is equal to the associated largest eigenvalue, i.e., $\|\bA\|_2 = \lambda_{\max}(\bA)$.  Then, it holds that
    \begin{align*}
        \|\bB\|_2 = \sqrt{\big\|\bB^\top\bB\|_2} = \sqrt{\lambda_{\max}\big(\bB^\top\bB\big)} = \max_{i\in[s]}\sqrt{\lambda_{\max}\big(\bA_i^\top\bA_i\big)} = \max_{i\in[s]}\sqrt{\big\|\bA_i^\top\bA_i\big\|_2} = \max_{i\in[s]}\|\bA_i\|_2,
    \end{align*}
    which completes the proof.
\end{proof}

Denote $[\Theta_1,\Theta_2] = \{\alpha\Theta_1 + (1-\alpha) \Theta_2 : \alpha\in[0,1]\}$ as the line segment between $\Theta_1$ and $\Theta_2$.
\begin{lemma}[Local quasi-convexity property \cite{taheri2024generalization}]\label{lem:quasi-convexity}
    Suppose $G:\R^d \rightarrow \R$ be a second-order differentiable function satisfying $ \lambda_{\min}(\nabla^2 G(\Theta)) \ge -\kappa G(\Theta) $. Let $\Theta_1, \Theta_2\in \R^d$ be two arbitrary points with distance $\|\Theta_1-\Theta_2\|_2\le D\le \sqrt{2/\kappa}$. Let $\tau:=(1-D^2\kappa/2)^{-1}$.
    Then, 
    \[ \max_{\mathcal{V} \in [\Theta_1,\Theta_2]} G(\mathcal{V}) \le \tau \max\big\{ G(\Theta_1) , G(\Theta_2)  \big\}. \]
\end{lemma}

\section{Proofs for Optimization.}\label{sec:appe-opt}
{\bf Matrix form of $f_\Theta$}:
For any scalar $v\in\R$, we denote $\bh(v) = [b_1(v), \ldots, b_{p}(v)]^\top\in\R^{p}$.
For $s\in\mathbb{N}$ and a vector $\bu = [u_1,\ldots,u_s]^\top \in \R^{s}$, we denote $\bh(\bu) = [\bh(u_1)^\top,\ldots,\bh(u_s)^\top]^\top\in\R^{sp}.$

To proof Proposition~\ref{pro:smooth}, we need the following results for gradients and Hessians.
To better understand the gradients and Hessians, we introduce the matrix multiplication form of $f_\Theta$.
Denote $\bA = [\ba_1,\ldots,\ba_m]^\top\in\R^{m\times dp}$, $f_\Theta$ can be rewritten as
\begin{align*}
    f_\Theta(\bx) = \frac{1}{\sqrt{m}} \bc^\top\bh\Big(\sigma\big(\frac{1}{\sqrt{d}}\bA\bh(\bx)\big)\Big).
\end{align*}
Let $C_{\sigma,b}$  be a constant that depend solely on $\sigma, b$.  
\begin{lemma}[Gradient and Hessian]\label{lem:hessian}
    Let $\delta\in(0,1)$.
    Suppose \eqref{eq:bound_c0} and Assumptions~\ref{ass:sigma} and \ref{ass:loss} hold. 
    It holds for any $\Theta=(\ba, \bc)$ and any $\bx \in \X$ that
    \[ \big\|\nabla f_\Theta(\bx)\big\|_2 \le C_{\sigma,b} \, p  \, \Big( \sqrt{p} +\sqrt{\frac{\log({1}/{\delta})}{m}}+ \max_{i\in[m]}\|\bc_i(0) - \bc_i\|_2 \Big) \]
    and
    \[\big\|\nabla^2 f_{\Theta}(\bx)\big\|_2  \le \frac{ C_{\sigma, b} \, p^{\frac{ 3}{2}}\big(\sqrt{p} +  \sqrt{\log({m}/{\delta})}+ \max_{i\in[m]}  \big\|\bc_i(0)-\bc_i\big\|_2\big)}{\sqrt{m}}.\]
\end{lemma}
\begin{proof}
We will control $\|\nabla  f_\Theta(\bx)\|_2$ and $\| \nabla^2  f_\Theta(\bx)\|_2$ separately. 
    It's obvious that
\begin{align}\label{eq:partial_c_norm}
    \big\|\partial_\bc f_{\Theta}(\bx)\big\|_2=  \frac{1}{\sqrt{m}}\big\|\bh\big( \sigma\big(\frac{1}{\sqrt{d}} \bA \bh(\bx)\big)\big) \big\|_2 \le  \sqrt{p} B_{b}. 
\end{align}

For all $v\in\R$, we denote $\bh'(v) = [b_1'(v),\ldots,b_p'(v)]^\top \in\R^p$ and  $\bh''(v) = [b_1''(v),\ldots,b_p''(v)]^\top \in\R^p$. 
Define 
\begin{align}\label{eq:bu}
    \bu(\bx) = \sigma\big(\frac{1}{\sqrt{d}}\bA\bh(\bx)\big) \ \text{ and } \ \bD(\bx) = \text{diag}\big(\sigma'(\frac{1}{\sqrt{d}}\ba_i^\top\bh(\bx))\big)_{i=1}^m \in \R^{m\times m}.
\end{align}

Note that $f_\Theta (\bx) = \frac{1}{\sqrt{m}} \bc^\top \bh(\bu(\bx))$, 
\begin{align*}
    \frac{\partial \bh(\bu(\bx))}{\partial \bu(\bx)} = \begin{bmatrix}
        \bh'(u_1(\bx)) & \mathbf{0} &\cdots &  \mathbf{0}\\
        \mathbf{0} & \bh'(u_2(\bx)) &\cdots &  \mathbf{0}\\
        \vdots & \vdots & \ddots & \vdots\\
        \mathbf{0} & \mathbf{0} & \mathbf{0} & \bh'(u_m(\bx))
    \end{bmatrix}\in\R^{mp\times m}
\end{align*}
and
\begin{align*}
    \frac{\partial \bu(\bx)}{\partial \ba_i} = \begin{bmatrix}
        \mathbf{0}\\
        \vdots\\
        \frac{1}{\sqrt{d}}\sigma'(\frac{1}{\sqrt{d}}\ba_i^\top \bh( \bx))\bh(\bx)^\top\\
        \vdots\\
        \mathbf{0}
    \end{bmatrix}\in\R^{m\times pd}.
\end{align*}

According to the chain rule, for any $i\in[m]$, it holds that
\begin{align*}
    \partial_{\ba_i} f_{\Theta}(\bx) &= \frac{\partial f_{\Theta}(\bx)}{\partial \bh(\bu(\bx))} \ \frac{\partial \bh(\bu(\bx))}{\partial \bu(\bx)} \ \frac{\partial \bu(\bx)}{\partial \ba_i} =\frac{1}{\sqrt{m d}} \big\langle \bc_i, \bh'(u_i(\bx)) \big\rangle   \sigma'\big(\frac{1}{\sqrt{d}}\ba_i^\top\bh(\bx)\big)\bh(\bx)^\top.
\end{align*}

% Then, we know
% \begin{align*}
%     \partial_{\bA} f_{\Theta}(\bx) = \frac{1}{\sqrt{m}}\bD(\bx) \begin{bmatrix}
%         \big\langle \bc_1, \bh'(u_1(\bx))\big\rangle\\
%         \vdots\\
%         \big\langle \bc_m, \bh'(u_m(\bx))\big\rangle
%     \end{bmatrix}\bh(\bx)^\top \in \R^{m\times pd}.
% \end{align*}
Recall that $\ba = \text{Vec}\big(\{\ba_i\}_{i=1}^m\big)\in\R^{mpd}$ is the vectorization of $\ba$.
It holds that
$$\partial_{\ba} f_{\Theta}(\bx) = \frac{1}{\sqrt{m d}}\text{Vec}\Big(\big\{\sigma'\big(\frac{1}{\sqrt{d}}\ba_i^\top\bh(\bx)\big)\langle \bc_i, \bh'(u_i(\bx))\rangle \bh(\bx)\big\}_{i=1}^m\Big) \in \R^{mpd},$$
and its associated $\|\cdot\|_2$-norm can be controlled by
\begin{align}
    \|\partial_{\ba} f_{\Theta}(\bx)\|_2 &= \frac{1}{\sqrt{md}}\Big(\sum_{i=1}^m \sigma'\big(\frac{1}{\sqrt{d}}\ba_i^\top\bh(\bx)\big)^2 \, \big|\langle \bc_i, \bh'(u_i(\bx))\rangle\big|^2 \|\bh(\bx)\|_2^2\Big)^{\frac{1}{2}} \nonumber\\
    &\le B_\sigma'B_b\sqrt{\frac{p }{m}}\Big(\sum_{i=1}^m \big|\langle \bc_i, \bh'(u_i(\bx))\rangle\big|^2\Big)^{\frac{1}{2}} \nonumber\\
    &\le B_\sigma'B_bB_b'\sqrt{\frac{p^2 }{m}}\Big(\sum_{i=1}^m \|\bc_i\big\|_2^2\Big)^{\frac{1}{2}} %= B_\sigma'B_bB_b'\sqrt{\frac{p^2d}{m}}\|\bc\|_2
    \nonumber\\
    &\le B_\sigma'B_bB_b'\sqrt{\frac{p^2 }{m}}\big(\|\bc(0)\|_2 + \|  \bc(0)-\bc\|_2\big)\label{eq:partial_a_norm_1}\\
    &\le B_\sigma'B_bB_b' p \, \Big(4\sqrt{p} + 2\sqrt{\frac{\log(\frac{2}{\delta})}{m}} + \max_{i\in[m]}\| \bc_i(0)-\bc_i\|_2\Big)\label{eq:partial_a_norm},
\end{align}
where the last inequality used \eqref{eq:bound_c0}.
Combining the estimates for $\|\partial_\bc f_\Theta(\bx)\|_2$ (see \eqref{eq:partial_c_norm}) and $\|\partial_\ba f_\Theta(\bx)\|_2$ (see \eqref{eq:partial_a_norm}) and the fact $\nabla f_\Theta(\bx) = [\partial_\bc f_\Theta(\bx)^\top, \partial_\ba f_\Theta(\bx)^\top]^\top$, it holds that
\begin{align}\label{eq:partial_nabla_norm}
    \big\|\nabla f_\Theta(\bx)\big\|_2  \le C_{\sigma,b} \, p  \, \Big( \sqrt{p} +\sqrt{\frac{\log(\frac{1}{\delta})}{m}}+ \max_{i\in[m]}\|\bc_i(0) - \bc_i\|_2 \Big).
\end{align}

Now, we turn to estimate the Hessian of $f_\Theta(\bx)$.
Note that
\begin{align*}
    \partial_{\ba}^2 f_\Theta(\bx) = \begin{bmatrix}
        \partial_{\ba_1}^2 f_\Theta(\bx) & \cdots & \mathbf{0}\\
        \vdots & \ddots & \vdots\\
        \mathbf{0} & \cdots & \partial_{\ba_m}^2 f_\Theta(\bx)
    \end{bmatrix}\in\R^{mpd\times mpd},
\end{align*}
where 
\begin{align*}
    \partial_{\ba_i}^2 f_\Theta(\bx) = & \frac{1}{d\sqrt{m}}\Big(\sigma''\big(\frac{1}{\sqrt{d}}\ba_i^\top\bh(\bx)\big)\big\langle \bc_i, \bh'\big(\sigma\big(\frac{1}{\sqrt{d}}\ba_i^\top\bh(\bx)\big)\big)\big\rangle \nonumber\\&+\big(\sigma'\big(\frac{1}{\sqrt{d}}\ba_i^\top\bh(\bx)\big)^2\big\langle\bc_i, \bh''\big(\sigma\big(\frac{1}{\sqrt{d}}\ba_i^\top\bh(\bx)\big)\big)\big\rangle\big)\Big)\bh(\bx)\bh(\bx)^\top \in \R^{pd\times pd}.
\end{align*}
Denoting $\bw_i = \sigma''\big(\frac{1}{\sqrt{d}}\ba_i^\top\bh(\bx)\big)\bh'\big(\sigma\big(\frac{1}{\sqrt{d}}\ba_i^\top\bh(\bx)\big)\big) + \sigma'\big(\frac{1}{\sqrt{d}}\ba_i^\top\bh(\bx)\big)^2 \,\bh''\big(\sigma\big(\frac{1}{\sqrt{d}}\ba_i^\top\bh(\bx)\big)\big)$. We can rewrite $\partial_{\ba_i}^2 f_\Theta(\bx)$ as
\[\partial_{\ba_i}^2 f_\Theta(\bx) = \frac{1}{d\sqrt{m}}\langle\bc_i, \bw_i\rangle\bh(\bx)\bh(\bx)^\top.\]
Note $\partial_{\ba}^2 f_\Theta(\bx)$ is the block diagonal matrix.
From Lemma \ref{lem:block_matrix} we know
\begin{align}
    \big\|\partial_{\ba}^2 f_\Theta(\bx)\big\|_2 &= \max_{i\in[m]}\big\|\partial_{\ba_i}^2 f_\Theta(\bx)\big\|_2 = \max_{i\in[m]}\sup_{\|\bv\|_2=1} \big|\bv^\top \, \partial_{\ba_i}^2 f_\Theta(\bx) \, \bv\big|\nonumber\\
    &= \frac{1}{d\sqrt{m}}\max_{i\in[m]} \sup_{\|\bv\|_2=1}\big|\langle\bc_i,\bw_i\rangle \langle \bh(\bx), \bv\rangle^2 \big| \le \frac{1}{d\sqrt{m}}\max_{i\in[m]} \big\|\bc_i\big\|_2 \, \big\|\bw_i\big\|_2 \, \big\| \bh(\bx)\big\|_2^2\nonumber\\
    &\le \frac{1}{d\sqrt{m}} \big\|\bh(\bx)\big\|_2^2\max_{i\in[m]} \Big[\big\|\bw_i\big\|_2 \big( \big\| \bc_i(0)\big\|_2 + \big\|\bc_i(0)-\bc_i\big\|_2\big)\Big] \nonumber\\
    &\le B_b^2\frac{p }{\sqrt{m}}\max_{i\in[m]} \Big[\big\|\bw_i\big\|_2 \big( \big\| \bc_i(0)\big\|_2 + \big\|\bc_i(0)-\bc_i\big\|_2\big)\Big] \nonumber\\
    &\le B_b^2\frac{p }{\sqrt{m}}\max_{i\in[m]} \big\|\bw_i\big\|_2 \Big( 4\sqrt{p} + 2\sqrt{\log(\frac{m}{\delta})} + \max_{i\in[m]}  \big\|\bc_i(0)-\bc_i\big\|_2\Big) \nonumber\\
    &\le B_b^2(B_{\sigma}''B_{b}' + B_{\sigma}'^2B_{b}'')\frac{p^{\frac{ 3}{2}}}{\sqrt{m}}\Big( 4\sqrt{p} + 2\sqrt{\log(\frac{m}{\delta})} + \max_{i\in[m]}  \big\|\bc_i(0)-\bc_i\big\|_2\Big),
\end{align}
where the second equality used the fact that $\partial_{\ba_i}^2f_\Theta(\bx)$ is symmetric, the first inequality used Cauchy-Schwarz inequality, the third inequality used Assumption \ref{ass:sigma} that $\sup_{t\in\R}|b(t)| \le B_b$, the last second inequality used \eqref{eq:bound_c0}, and in the last inequality we controlled $\|\bw_i\|_2$ by using Assumption \ref{ass:sigma} as follows
\begin{align*}
    \|\bw_i\|_2 &= \big\|\sigma''\big(\frac{1}{\sqrt{d}}\ba_i^\top\bh(\bx)\big) \, \bh'\big(\sigma\big(\frac{1}{\sqrt{d}}\ba_i^\top\bh(\bx)\big)\big) + \sigma'\big(\frac{1}{\sqrt{d}}\ba_i^\top\bh(\bx)\big)^2 \,\bh''\big(\sigma\big(\frac{1}{\sqrt{d}}\ba_i^\top\bh(\bx)\big)\big)\big\|_2 \\&\le \|\sigma''\|_\infty\sqrt{p}\|b'\|_\infty + \|\sigma'\|_\infty^2\sqrt{p}\|b''\|_\infty \le \sqrt{p}(B_{\sigma}''B_{b}' + B_{\sigma}'^2B_{b}'').
\end{align*}

Now, we turn to estimate $\|\partial_\bc\partial_\ba f_\Theta(\bx)\|_2$.
Note $\partial_\bc\partial_\ba f_\Theta(\bx)$ is a block diagonal matrix which has the form
\begin{align*}
    \partial_\bc\partial_\ba f_\Theta(\bx) = \begin{bmatrix}
        \partial_{\bc_1}\partial_{\ba_1} f_\Theta(\bx) & \cdots & \mathbf{0}\\
        \vdots & \ddots & \vdots\\
        \mathbf{0} & \cdots & \partial_{\bc_m}\partial_{\ba_m} f_\Theta(\bx)
    \end{bmatrix}\in\R^{mpd\times mp},
\end{align*}
where $\partial_{\bc_i}\partial_{\ba_i} f_\Theta(\bx) = \frac{1}{\sqrt{md}}\sigma'\big(\frac{1}{\sqrt{d}}\ba_i^\top\bh(\bx)\big) \, \bh(\bx)\bh'(u_i(\bx))^\top\in\R^{pd\times p}$.
From Lemma~\ref{lem:block_matrix}, we know
\begin{align}\label{eq:nabla^2-ac}
    \big\|\partial_\bc\partial_\ba f_\Theta(\bx)\big\|_2 &= \max_{i\in[m]}\big\|\partial_{\bc_i}\partial_{\ba_i} f_\Theta(\bx)\big\|_2\nonumber\\
    &= \frac{1}{\sqrt{md}}\max_{i\in[m]}\sup_{\|\bu\|_2=\|\bv\|_2=1} \sigma'\big(\frac{1}{\sqrt{d}}\ba_i^\top\bh(\bx)\big) \big\langle \bh(\bx), \bu\big\rangle \, \big\langle \bh'(u_i(\bx)), \bv\big\rangle\nonumber\\
    &\le \frac{1}{\sqrt{md}}\|\sigma'\|_\infty\|\bh(\bx)\|_2\max_{i\in[m]}\|\bh'(u_i(\bx))\|_2 \le \frac{B_\sigma'B_bB_b'\, p }{\sqrt{m}},
\end{align}
where the last inequality used Assumption~\ref{ass:sigma}.

Note the Hessian of $f_\Theta(\bx)$ has the form
$$\nabla^2 f_{\Theta}(\bx) = \begin{bmatrix}
    \partial^2_\bc f_\Theta(\bx) & \partial_\bc\partial_\ba f_\Theta(\bx)^\top\\
    \partial_\bc\partial_\ba f_\Theta(\bx) & \partial^2_\ba f_\Theta(\bx)
\end{bmatrix} = \begin{bmatrix}
    \mathbf{0} & \partial_\bc\partial_\ba f_\Theta(\bx)^\top\\
    \partial_\bc\partial_\ba f_\Theta(\bx) & \partial^2_\ba f_\Theta(\bx)
\end{bmatrix}.$$
Combining the above estimates for $\|\partial_\bc\partial_\ba f_\Theta(\bx)\|_2 $ and $\|\partial^2_\ba f_\Theta(\bx)\|_2$ and the fact that $\nabla^2 f_{\Theta}(\bx)$ is symmetric, we know
\begin{align}%\label{eq:Hessian_f}
    \big\|\nabla^2 f_{\Theta}(\bx)\big\|_2 &= \sup_{\|\bv\|_2=1} \big|\bv^\top \nabla^2 f_{\Theta}(\bx)\bv\big| = \sup_{\|(\bv_1,\bv_2)\|_2=1} \big|2\bv_1^\top \, \partial_\bc\partial_\ba f_\Theta(\bx) \, \bv_2 + \bv_2^\top \, \partial^2_\ba f_\Theta(\bx) \, \bv_2\big|\nonumber\\
    &\le 2\big\|\partial_\bc\partial_\ba f_\Theta(\bx)\big\|_2 + \big\|\partial^2_\ba f_\Theta(\bx)\big\|_2 \label{eq:nabla^2}\\
    &\le  \frac{ C_{\sigma, b} \, p^{\frac{ 3}{2}}\big(\sqrt{p} +  \sqrt{\log(\frac{m}{\delta})}+ \max_{i\in[m]}  \big\|\bc_i(0)-\bc_i\big\|_2\big)}{\sqrt{m}}\nonumber .  
\end{align} 
The proof is complete.
\end{proof}
Building on Lemma~\ref{lem:hessian}, we establish in the following upper and lower bounds on the largest and smallest eigenvalues of $\nabla^2  \mathcal{L}_S( \Theta )$, respectively.
These bounds depend on the quantity $\max_{i\in[m]}  \|\bc_i(0) -\bc_i \|_2$. We will show that this term remains controlled along the gradient descent trajectory. Consequently, the loss function $\ell(y  f_\Theta(\bx))$ is weakly convex and smooth with respect to the variable $\Theta$. 
\begin{proposition}[Smoothness and Curvature]\label{pro:smooth}
    Let $\delta\in(0,1)$.
    Suppose \eqref{eq:bound_c0} and Assumptions~\ref{ass:sigma} and \ref{ass:loss} hold.
    Assume $m\gtrsim \max\{\log(m/\delta),p\}$.
    It holds for any $\Theta$ and any training dataset $S$, that
    $$\lambda_{\min} \big( \nabla^2  \mathcal{L}_S( \Theta ) \big)   \ge  -  \frac{ C_{\sigma , b} \, p^{\frac{3}{2}} \big( \sqrt{\log(\frac{m}{\delta})}  +  \sqrt{p}  +  R_{\bc}\big)}{\sqrt{m}}  \mathcal{L}_S( \Theta )$$ 
    and
    $$\lambda_{\max}\big(\nabla^2 \mathcal{L}_S( \Theta )\big)  \le   C_{\sigma,b } \,  p^2  \big(  p + R_{\bc}^2 \big),$$
where $R_{\bc}= \max_{i\in[m]}  \|\bc_i(0) -\bc_i \|_2$.
\end{proposition}
\begin{proof}
The gradient of loss is given as
\[ \nabla  \ell(y f_\Theta(\bx)) = \ell'(y f_\Theta(\bx))y  \nabla  f_\Theta(\bx) .\]
% Note that $|y|=1$.
% It holds that
% \begin{align}
%     \big\|\nabla  \ell(y f_\Theta(\bx))\big\|_2\le \big|\ell'(y f_\Theta(\bx))\big| \big\|  \nabla  f_\Theta(\bx)\big\|_2 .
% \end{align}
For the Hessian of loss, note that
\[ \nabla^2  \ell(y f_\Theta(\bx)) = \ell''(y f_\Theta(\bx))   \nabla  f_\Theta(\bx)\nabla  f_\Theta(\bx)^\top + \ell'(y f_\Theta(\bx)) y  \nabla^2  f_\Theta(\bx). \]
Since $\ell$ is convex, $\ell''(a) \ge 0$ for all $a\in\R$.
Then, $\ell''(y f_\Theta(\bx))   \nabla  f_\Theta(\bx)\nabla  f_\Theta(\bx)^\top$ is a PSD matrix.
By further noting that $|\ell'(a)| \le G_\ell=1$ and $|\ell''(a)|\le L_\ell=1$ for all $a\in\R$, we have
\begin{align}\label{eq:Hessian_loss}
      -|\ell'(y f_\Theta(\bx))|\big\|\nabla^2  f_\Theta(\bx)\big\|_2 \le \lambda_{\min}\big(\nabla^2  \ell(y f_\Theta(\bx))\big) \le  \lambda_{\max}\big(\nabla^2  \ell(y f_\Theta(\bx))\big) \le
      \big\| \nabla  f_\Theta(\bx)\big\|_2^2 +  \big\| \nabla^2  f_\Theta(\bx)\big\|_2.
\end{align}

% Hence,
% \begin{align}
%   \big\|  \nabla^2  \ell(y f_\Theta(\bx))\big\|_2 &\le \big| \ell''(y f_\Theta(\bx))\big| \big\| \nabla  f_\Theta(\bx)\big\|_2^2 + \big|\ell'(y f_\Theta(\bx))\big| \big\| \nabla^2  f_\Theta(\bx)\big\|_2^2\nonumber\\
%   &\le  
% \end{align}
% If we further use selfboundess, it holds

Plugging the estimates of $\big\| \nabla  f_\Theta(\bx)\big\|_2$ and $\big\| \nabla^2  f_\Theta(\bx)\big\|_2$ (see Lemma~\ref{lem:hessian}) back into \eqref{eq:Hessian_loss} and noting that $|\ell'(yf_\Theta(\bx))|\le  \ell(yf_\Theta(\bx))$, we know
\[\lambda_{\min}\big(\nabla^2  \ell(y f_\Theta(\bx))\big)  \ge -  \frac{C_{\sigma, b} \, p^{\frac{3}{2}} \big(\sqrt{\log(\frac{ m}{\delta})} + \sqrt{p} + R_{\bc}\big)}{\sqrt{m}}  \ell(yf_\Theta(\bx)),\]
and
    \[\lambda_{\max}\big(\nabla^2  \ell(y f_\Theta(\bx))\big)  \le
     C_{\sigma,b } \, p^2  \big(p +  \sqrt{\frac{\log(\frac{m}{\delta})}{m}} + R_{\bc}^2\big)  
\]
with $R_{\bc}= \max_{i\in[m]}\| \bc_i(0)-\bc_i\|_2$. 

By applying the same argument to each sample loss and using that the empirical loss is their average, we obtain,
\[  \lambda_{\min}\big(\nabla^2  \mathcal{L}_S( \Theta )\big)  \ge -  \frac{C_{\sigma, b} \, p^{\frac{3}{2}} \big(\sqrt{\log(\frac{ m}{\delta})} + \sqrt{p} + R_{\bc}\big)}{\sqrt{m}}  \mathcal{L}_S( \Theta ),\]
and from the condition $m\gtrsim \max\{\log(m/\delta),p\}$ we have
\[\lambda_{\max}\big(\nabla^2  \mathcal{L}_S( \Theta )\big) \le
     C_{\sigma,b } \, p^2 \big(p + R_{\bc}^2\big).\]
The proof is complete.
\end{proof}

In the subsequent proof, we work with vectorized quantities, and therefore $\Theta \in \R^{mp(d+1)}$.  
Let $\{\Theta(k)\}_{k\in[T]}$ be produced by GD with $T$ iterations. By using self-bounding property of the loss $\ell$, we have following estimates of $\ba$ and $\bc$.% Define $R_{\bc,T}=\max_{k\in[T],i\in[m]}\{ \|\bc_i(0) - \bc_i(k)\|_2^2 \}$  and $\rho_\ell=C_{\sigma,b } \, p^2d \,  (p + R_{\bc,T}^2 )$.  Proposition~\ref{pro:smooth} implies that $\mathcal{L}_S(\Theta)$ is $\rho_\ell$-strongly smooth through the trajectory of GD. The following theorem provides the optimization error bounds for GD with KANs.  

\begin{lemma}\label{lem:update-ca}
    Let $\delta\in(0,1)$.
     Suppose Assumptions~\ref{ass:sigma} and \ref{ass:loss} hold.
     Let $\{\Theta(k)\}_{k\in[T]}$ be produced by GD. The following statements hold true.
     \begin{enumerate}[label=(\alph*), leftmargin=*]
       \item For all $k\in [T-1]$, we have \[  \max_{j\in[m]}\big\| \bc_j (k+1) - \bc_j (0)  \big\|_{2 } \le \frac{\eta B_b\sqrt{p}}{\sqrt{m}} \sum_{t=0}^k \mathcal{L}_S(\Theta(t)).\]
     \item Let $\delta\in(0,1)$, assume $m\gtrsim \log(m/\delta)$ and  \eqref{eq:bound_c0} holds. Then, for all $k\in [T-1]$, we have
      \[ \max_{j\in[m]}\big\| \ba_j(k+1) - \ba_j(0)  \big\|_{2 } \le \frac{\eta B'_\sigma B_b B'_b p }{\sqrt{m}} \Big(4\sqrt{p} + 2\sqrt{\log(\frac{2m}{\delta})} + \frac{\eta B_b\sqrt{p}}{\sqrt{m}} \sum_{s=0}^T \mathcal{L}_S(\Theta(s))\Big) \sum_{t=0}^k \mathcal{L}_S(\Theta(t))  .\]
    \end{enumerate}
\end{lemma}
\begin{proof}
We first prove part (a) of the lemma. 
For any $j\in[m]$, by the GD update rule and the self-bounding property of the loss $\ell$, we have
\begin{align}
     \big\| \bc_j (k+1) - \bc_j (k) \big\|_2
     &=\eta \Big\| \frac{\partial \mathcal{L}_S(\Theta(k))}{\partial \bc_j } \Big\|_2 \le  \frac{\eta}{n} \sum_{i=1}^n \big|  \ell'(y_i f_{\Theta(k)}(\bx_i)) y_i  \big|  \Big\|  \frac{\partial f_{\Theta(k)}(\bx_i)}{\partial \bc_j }\Big\|_2\nonumber\\
     &=\frac{\eta}{\sqrt{m}}\frac{1}{n} \sum_{i=1}^n \big|  \ell'(y_i f_{\Theta(k)}(\bx_i)) \big|  \Big\| \bh\Big(\frac{1}{\sqrt{d}}\sigma\big(\ba_j^\top\bh(\bx)\big)\Big)\Big\|_2\nonumber\\
     &\le \frac{\eta B_b\sqrt{p}}{\sqrt{m}}  \frac{1}{n} \sum_{i=1}^n \Big|  \ell'(y_i f_{\Theta(k)}(\bx_i)) \Big| \le \frac{\eta B_b\sqrt{p}}{\sqrt{m}} \mathcal{L}_S(\Theta(k)),
\end{align}
where the second inequality used $\sup_{v\in\R}|b_k(v)|\le B_b$, and  the last inequality follows from the self-bounding property $|\ell'(v)| \le \ell(v)$ for all $v\in\R$.

Furthermore, by summing over iterations, we obtain
\begin{align}\label{eq:ck-c0}
     \big\| \bc_j (k+1) - \bc_j (0) \big\|_2
     &\le \sum_{t=0}^k \big\| \bc_j (t+1) - \bc_j (t) \big\|_2  \le  \frac{\eta B_b\sqrt{p}}{\sqrt{m}} \sum_{t=0}^k \mathcal{L}_S(\Theta(t)).
\end{align}

Now we consider part (b) of the lemma.
Using the same notation as \eqref{eq:bu}, we denote $u_j(\bx) = \sigma(\frac{1}{\sqrt{d}} \ba_j^\top\bh(\bx ))$ for any $j\in[m]$ and $\bx\in\X$.
According to the update rule, $|\ell'(v)|\le \ell(v)$ and the above inequality, we know for all $j\in[m]$ and $k=0,\ldots,T-1$, it holds that
\begin{align*}
    \big\| \ba_j(k+1) - \ba_j(k) \big\|_2
    &\le \frac{\eta}{\sqrt{md}} \Big\|  \frac{1}{n} \sum_{i=1}^n \ell'(y_i f_{\Theta(k)}(\bx_i)) y_i  \sigma'\big(\frac{1}{\sqrt{d}}\ba_j^\top\bh(\bx_i)\big)  \,  \langle \bc_j, \bh'(u_j(\bx_i))\rangle  \bh'(u_j(\bx_i)) \Big\|_2\\
    &\le \frac{\eta}{\sqrt{md}}\frac{1}{n}\sum_{i=1}^n \big|\ell'(y_i f_{\Theta(k)}(\bx_i))\big|\|\sigma'\|_\infty\|\bc_j\|_2\|\bh'(u_j(\bx_i))\|_2\|\bh'(u_j(\bx_i))\|_2\\
    &\le \frac{\eta B'_\sigma B_b B'_b p }{\sqrt{m}} \|\bc_j(k)\|_2 \mathcal{L}_S(\Theta(k))\\
    &\le  \frac{\eta B'_\sigma B_b B'_b p  }{\sqrt{m}} \big(\|\bc_j(0)\|_2+ \|\bc_j(k)-\bc_j(0)\|_2\big) \mathcal{L}_S(\Theta(k))\\
    &\le \frac{\eta B'_\sigma B_b B'_b p }{\sqrt{m}} \Big(4\sqrt{p} + 2\sqrt{\log(\frac{2m}{\delta})} + \frac{\eta B_b\sqrt{p}}{\sqrt{m}} \sum_{t=0}^k \mathcal{L}_S(\Theta(t))  \Big) \mathcal{L}_S(\Theta(k)),
\end{align*}
where we have used Cauchy-Schwarz inequality in the second inequality, and in the third inequality we have used Assumption~\ref{ass:sigma} and the self-boundedness property $|\ell'(v)|\le\ell(v)$, and in the last inequality we have used \eqref{eq:bound_c0} and \eqref{eq:ck-c0}.

Hence, it holds for all $j\in[m]$ that
\begin{align*}\label{eq:ak-a0}
     \big\| \ba_j (k+1) - \ba_j (0) \big\|_2
     &\le \sum_{t=0}^k \big\| \ba_j (t+1) - \ba_j (t) \big\|_2\\
     &\le \frac{\eta B'_\sigma B_b B'_b p  }{\sqrt{m}} \Big(4\sqrt{p} + \sqrt{\log\big(\frac{2m}{\delta}\big)} + \frac{\eta B_b\sqrt{p}}{\sqrt{m}} \sum_{s=0}^T \mathcal{L}_S(\Theta(s))\Big) \sum_{t=0}^k \mathcal{L}_S(\Theta(t)).
\end{align*}
The proof is complete. 
\end{proof}

For any reference point $\Theta^* \in \R^{mp(d+1)}$, let $c_{\max}=C_{\sigma, b} \, p^{\frac{3}{2}} \big(\sqrt{\log(\frac{m }{\delta})} + \sqrt{p} + 3\big\|\Theta(0) - \Theta^*\big\|_2 \big)$, and $\rho_\ell=C_{\sigma,b } \, p^2  \big(p  +  \max\{\frac{\eta^2 B^2_b p}{m}  \|\Theta(0) - \Theta^{*} \|_2^4, 3\|\Theta(0) - \Theta^{*}\|_2^2   )\}\big)$.

Observe that the only source of randomness in Theorem~\ref{thm:opt-error} comes from the initialization, and the conclusion holds on the event~\eqref{eq:bound_c0}. 
Therefore, we equivalently restate Theorem~\ref{thm:opt-error} on the event~\eqref{eq:bound_c0}, replacing the statement ``with probability at least $1-\delta$” by the condition~\eqref{eq:bound_c0}.
\begin{theorem}[Restatement of Theorem~\ref{thm:opt-error}]\label{thm:restate_opt-error}
    Let $\delta\in(0,1)$.
    Suppose \eqref{eq:bound_c0} and Assumptions~\ref{ass:sigma} and \ref{ass:loss} hold.
    For any reference point $\Theta^* \in \R^{mp(d+1)}$ satisfies $ \|\Theta(0) - \Theta^{*} \|_2^2 \ge  4 \max  \{  \eta T \mathcal{L}_S(\Theta^*),  \eta \mathcal{L}_S(\Theta(0)) \} $, assume $m\ge C_{\sigma, b}\, p^3 ( \log(m/
    \delta) + p + \|\Theta(0)-\Theta^*\|_2^2 )\|\Theta(0)-\Theta^*\|_2^4$, and $\eta \le \min\{1/\rho_\ell,1\}$, then we have
    \[ \mathcal{L}_S(\Theta(T ) ) 
        \le \frac{1}{T} \sum_{k=1}^{T }  \mathcal{L}_S(\Theta(k ) ) 
        \le  2 \mathcal{L}_S(\Theta^*) + \frac{1}{\eta T}  \big\|  \Theta(0)-\Theta^*\big\|_2^2.  \]
        Furthermore, for all $k\in[T-1]$, it holds that
        \[ \|\Theta(k+1) - \Theta^*\|_2  \le \sqrt{2}\|\Theta(0)-\Theta^*\|_2  \quad \text{ and }\quad \|\Theta(k+1) - \Theta(0)\|_2  \le 3\|\Theta(0)-\Theta^*\|_2 . \]
\end{theorem}
\begin{proof}
 
We use the induction strategy to prove the following statements.
It holds for any $k\in[T]$ that
\[\eta \sum_{s=1}^{k}  \mathcal{L}_S(\Theta(s ) ) 
         \le \frac{7}{5} \eta k \mathcal{L}_S(\Theta^*) + \frac{3 }{5  }  \big\| \Theta(0)-\Theta^*\big\|_2^2   +   \frac{ \eta}{5}  \mathcal{L}_S(\Theta(0)) ,\]
\[\|\Theta(k)-\Theta^*\|_2 \le  \sqrt{2}   \big\| \Theta(0)-\Theta^*\big\|_2 \quad \text{ and } \quad \big\| \Theta(k) - \Theta(0)\big\|_2  \le 3   \big\|\Theta(0) - \Theta^*\big\|_2.\] Note that the statements hold with $k=0$.
Here, we take the conventional notation $\sum_{s=1}^0 =0$.
We assume that the above statements hold for all $t\in[k]$ and will prove that they also hold for $t=k+1$.   

For any fixed $\Theta^* = (\ba^*, \bc^*)$, $t\in[k] $ and $ \alpha\in[0,1]$, we denote $\Theta_{\alpha, t }=\alpha\Theta(t ) + (1-\alpha)\Theta^*$ and $\bc_{\alpha,t} = \alpha\bc(t) + (1-\alpha)\bc^*$ as the convex combination of the reference point and the output of the $t$-th iteration.
According to the induction assumption, we know $R_{\bc}=\max_{i\in[m]}\|\bc_i(0)-[\bc_{\alpha,t }]_i\|_2 \le \max\{ \|\bc(t)-\bc(0)\|_2, \|\bc^*-\bc(0)\|_2\}\le \max\{ \|\Theta(t)-\Theta(0)\|_2, \|\Theta^*-\Theta(0)\|_2\}\le 3   \big\|\Theta(0) - \Theta^*\big\|_2$.
Then, from Proposition~\ref{pro:smooth} we know the least eigenvalue of $\nabla^2\mathcal{L}_S( \Theta_{\alpha, t })$ can be controlled as follows.
  \begin{align}\label{eq:lambdamin}
      \lambda_{\min}\big(\nabla^2 \mathcal{L}_S( \Theta_{\alpha, t })\big)  \ge  - \frac{c_{\max} }{\sqrt{m}} \mathcal{L}_S(\Theta_{\alpha,t}).    
  \end{align} 
 where $c_{\max}= C_{\sigma, b} \, p^{\frac{3}{2}} \big(\sqrt{\log(\frac{m }{\delta})} + \sqrt{p} + 3 \big\|\Theta(0) - \Theta^*\big\|_2 \big)$.

Then from Taylor's theorem and \eqref{eq:lambdamin}, we know that there exists a $\alpha\in[0,1]$ and the associated $\Theta_{\alpha, t }$ such that with probability at least $1-\delta$, it holds that
    \begin{align*}
        \mathcal{L}_S(\Theta^*) &= \mathcal{L}_S(\Theta(t )) + \big\langle \nabla \mathcal{L}_S(\Theta(t )), \Theta^* - \Theta(t ) \big\rangle +  \frac{1}{2} \big(  \Theta(t)-\Theta^*\big)^\top \nabla^2 \mathcal{L}_S( \Theta_{\alpha, t }) \big(  \Theta(t)-\Theta^*\big)\\
        &\ge \mathcal{L}_S(\Theta(t )) + \big\langle \nabla \mathcal{L}_S(\Theta(t )), \Theta^* - \Theta(t ) \big\rangle -  \frac{c_{\max} }{2\sqrt{m}} \mathcal{L}_S(\Theta_{\alpha,t }) \big\|  \Theta(t)-\Theta^*\big\|_2^2. 
        %&\ge  \mathcal{L}_S(\Theta(k)) + \big\langle \nabla \mathcal{L}_S(\Theta(k)), \Theta^* - \Theta(k) \big\rangle - \frac{1}{2}\max_{\alpha \in [0,1]}\lambda_{\min}(\alpha, k )  \big\| \Theta^* - \Theta(k)\big\|_2^2. 
    \end{align*}
   It then follows that
   \begin{align}\label{eq:L_S}
       \mathcal{L}_S(\Theta(t ))  &\le \mathcal{L}_S(\Theta^*) -     \big\langle \nabla \mathcal{L}_S(\Theta(t )), \Theta^* - \Theta(t) \big\rangle + \frac{c_{\max} }{2\sqrt{m}} \max_{
        \alpha \in [0,1]}    \mathcal{L}_S(\Theta_{\alpha,t }) \big\|  \Theta(t)-\Theta^*\big\|_2^2.%\nonumber\\  &\le \mathcal{L}_S(\Theta^*) -     \big\langle \nabla \mathcal{L}_S(\Theta(k)), \Theta^* - \Theta(k) \big\rangle + . 
   \end{align} 
   
Now, we turn to estimate the largest eigenvalue of $\nabla^2\mathcal{L}_S( \Theta_{\alpha, t })$.
For all $t\in[k]$, from Proposition~\ref{pro:smooth}, Lemma~\ref{lem:update-ca} and the induction assumption on the upper bound of $\eta \sum \mathcal{L}_S(\Theta(t))$, it holds that
   \begin{align*}
       & \lambda_{\max}\big(\nabla^2  \mathcal{L}_S( \Theta(t+1) )\big) \nonumber\\
       &\le
     C_{\sigma,b } \, p^2 \Big(p  + \max_{i\in [m]}\{\bc_i(0)-\bc_i(t+1)\}^2\Big)\nonumber\\
     &\le   C_{\sigma,b } \, p^2 \Big(p  + \frac{B^2_b p}{m} \Big(\eta\sum_{s=0}^t \mathcal{L}_S(\Theta(s))\Big)^2\Big) \nonumber\\ 
     &\le C_{\sigma,b } \, p^2 \Big(p + \frac{B^2_b p}{m}\Big( \frac{7}{5} \eta k\mathcal{L}_S(\Theta^*) + \frac{3}{5}  \big\|   \Theta(0)-\Theta^*\big\|_2^2   +  \frac{  \eta}{5} \mathcal{L}_S(\Theta(0))  \Big)^2 \Big).
   \end{align*}
   % where we define $\rho'_\ell= C_{\sigma,b } \, p^2 \big(p + \frac{B^2_b p}{m}\big( \frac{7}{5} \eta T\mathcal{L}_S(\Theta^*) + \frac{7}{10}  \big\|   \Theta(0)-\Theta^*\big\|_2^2   +  \frac{  7\eta}{30} \mathcal{L}_S(\Theta(0))  \big)^2 \big)$.
   
   %Let $\rho_\ell=C_{\sigma,b } \, p^2d \,  (p + \frac{\eta B_b\sqrt{p}}{\sqrt{m}} \sum_{t=0}^{T-1} L_S(\Theta(t))  )$. 
   Note that the condition $ \|\Theta(0) - \Theta^{*} \|_2^2 \ge  4 \max  \{  \eta T \mathcal{L}_S(\Theta^*),  \eta \mathcal{L}_S(\Theta(0)) \} $  implies the estimate
   \(\frac{7}{5} \eta T\mathcal{L}_S(\Theta^*) + \frac{7}{10}  \big\|   \Theta(0)-\Theta^*\big\|_2^2   +  \frac{  6\eta}{5} \mathcal{L}_S(\Theta(0))  \le \|\Theta(0) - \Theta^{*} \|_2^2. \)
   Then, the above inequality can be further bounded as follows
   \[\lambda_{\max}\big(\nabla^2  \mathcal{L}_S( \Theta(t+1) )\big) \le C_{\sigma,b } \, p^2 \big(p + \frac{B^2_b p}{m}  \|\Theta(0) - \Theta^{*} \|_2^4  \big) \le C_{\sigma, b}p^3=:\rho_\ell,\]
   where the last inequality used the condition $m\gtrsim\|\Theta(0) - \Theta^*\|_2^4$.
    Then we know $\ell(\Theta(t+1))$ is $\rho_\ell$-strongly smooth through the trajectory of $\{\Theta(t)\}_{t\in[k]}$. 
    From the descent lemma (see e.g., \cite{taheri2024generalization}), we know if $\eta \le 1/\rho_\ell$, it holds for any $t\in[k]$ that
   \begin{align}\label{eq:descent}
       \mathcal{L}_S(\Theta(t+1 ) )\le \mathcal{L}_S(\Theta(t  )  )  - \frac{\eta}{2} \big\|\nabla \mathcal{L}_S(\Theta(t )) \big\|_2^2. 
   \end{align}
   
   Combining the above inequality and \eqref{eq:L_S}, we know
    \begin{align}\label{eq:k+1-k}
        &\mathcal{L}_S(\Theta(t+1 ) ) \nonumber\\
        &\le \mathcal{L}_S(\Theta^*) -    \big\langle \nabla \mathcal{L}_S(\Theta(t  )), \Theta^* \!- \Theta(t ) \big\rangle - \frac{\eta}{2} \big\|\nabla \mathcal{L}_S(\Theta(t )) \big\|_2^2 +   \frac{c_{\max} }{2\sqrt{m}} \max_{\alpha\in[0,1]}    \mathcal{L}_S(\Theta_{\alpha,t })\big\| \Theta(t)-\Theta^*\big\|_2^2\nonumber\\
        &= \mathcal{L}_S(\Theta^*) + \frac{1}{2\eta} \big( \big\|\Theta(t)-\Theta^*\big\|_2^2 - \big\| \Theta(t+1)-\Theta^*\big\|_2^2\big)  +  \frac{c_{\max} }{2\sqrt{m}} \max_{\alpha\in[0,1]}    \mathcal{L}_S(\Theta_{\alpha,t })\big\|\Theta(t)-\Theta^*\big\|_2^2,  
    \end{align}
    where the last equality used the fact $\nabla \mathcal{L}_S(\Theta(t ))= \frac{1}{\eta} \big(\Theta(t ) -\Theta(t+1)\big)$ and the identity $2\langle \bu, \bv\rangle = \|\bu\|_2^2 + \|\bv\|_2^2 - \|\bu-\bv\|_2^2$.
    
    Now, we turn to estimate $\max_{\alpha\in[0,1]}    \mathcal{L}_S(\Theta_{\alpha,t })$. 
    From the the induction assumption $\|\Theta(t)-\Theta^*\|_2 \le  \sqrt{2}   \big\|\Theta(0) - \Theta^*\big\|_2$ for all $t=0,\ldots, k$, the condition $m\ge  {64c^2_{\max} }\|\Theta(0)-\Theta^*\|_2^4$ (implied by $m\gtrsim p^3 ( \log(m/
    \delta) + p + \|\Theta(0)-\Theta^*\|_2^2 )\|\Theta(0)-\Theta^*\|_2^4$), we know 
\begin{align}\label{eq:cmax}
     \big\|   \Theta(t)-\Theta^*\big\|_2^2\le  \frac{ \sqrt{m}}{4c_{\max} } .
\end{align}
Therefore, Lemma~\ref{lem:quasi-convexity} with $ D^2=  \frac{ \sqrt{m}}{4c_{\max} }$ and $\kappa=\frac{c_{\max}}{\sqrt{m}}$ (hence $D^2\kappa\le 1/4$) implies, for all $t=0,\ldots, k$
\[\max_{\alpha\in[0,1]  }  \mathcal{L}_S(\Theta_{\alpha,t})  \le \frac{4}{3}  \max\big\{\mathcal{L}_S(\Theta(t)), \mathcal{L}_S(\Theta^*)  \big\}\le  \frac{4}{3}  \big(\mathcal{L}_S(\Theta(t))+ \mathcal{L}_S(\Theta^*) \big).   \]
Plugging the above inequality and \eqref{eq:cmax} back into \eqref{eq:k+1-k} yields, for any $t=0,\ldots,k $, we have
  \begin{align}\label{eq:k+1-*}
         \mathcal{L}_S(\Theta(t+1 ) )   
         &\le \mathcal{L}_S(\Theta^*)  +  \frac{\!1}{ 2\eta} \big( \big\| \Theta(t) - \Theta^* \big\|_2^2\! -\! \big\| \Theta(t\!+\!1) - \Theta^* \big\|_2^2\big)  \!+\!  \frac{ 2c_{\max} }{3\sqrt{m}} \big( \mathcal{L}_S(\Theta(t ))\!+\! \mathcal{L}_S(\Theta^*)  \big)\big\| \Theta(t)-\Theta^*\big\|_2^2\nonumber\\
         &\le \mathcal{L}_S(\Theta^*) +  \frac{ 1}{ 2\eta} \big( \big\| \Theta(t) - \Theta^* \big\|_2^2  - \big\| \Theta(t + 1) - \Theta^* \big\|_2^2\big)   +   \frac{ 1}{6} \big( \mathcal{L}_S(\Theta(t )) + \mathcal{L}_S(\Theta^*) \big)
    \end{align}

  Summing over $t\in\{0,\ldots,k\}$ yields
    \begin{align*}
     \sum_{t=1}^{k+1 }  \mathcal{L}_S(\Theta(t ) )  
        &\le  (k+1) \mathcal{L}_S(\Theta^*) + \frac{1}{2\eta   }  \big\| \Theta(0)-\Theta^* \big\|_2^2   + \frac{1}{6 } \sum_{t=0}^{k }\big(\mathcal{L}_S(\Theta(t ))+ \mathcal{L}_S(\Theta^*) \big) \\  
        &= \frac{7}{6} (k+1) \mathcal{L}_S(\Theta^*) + \frac{1}{2\eta  }  \big\| \Theta(0)-\Theta^* \big\|_2^2   + \frac{1}{6} \sum_{t=0}^{k } \mathcal{L}_S(\Theta(t)) .   
    \end{align*}
It then follows that 
     \begin{align*}
    \eta \sum_{t=1}^{k+1 }  \mathcal{L}_S(\Theta(t ) ) 
        &\le \frac{7}{5} \eta(k+1) \mathcal{L}_S(\Theta^*) + \frac{3}{5   }  \big\| \Theta(0)-\Theta^* \big\|_2^2   + \frac{ \eta}{5} \mathcal{L}_S(\Theta(0)),
    \end{align*}
%Hence,\begin{align*}\eta \sum_{t=0}^{k+1 }  \mathcal{L}_S(\Theta(k ) ) &\le \frac{3}{2}  \eta(k+1) \mathcal{L}_S(\Theta^*) + \frac{3}{5   }  \big\| \Theta^* - \Theta(0)\big\|_2^2   + \Big(1+\frac{  \eta}{6} \Big)\mathcal{L}_S(\Theta(0)),\end{align*}
which completes the first part of the induction. 

It remains to show $\|\Theta(k+1) - \Theta^*\|_2  \le \sqrt{2}\|\Theta(0)-\Theta^*\|_2  $ and $ \|\Theta(k+1) - \Theta(0)\|_2  \le 3\|\Theta(0)-\Theta^*\|_2 $.
Note we have the estimate
 \begin{align} 
       \big\| \Theta(k+1)-\Theta^* \big\|_2^2    
          &\le   \big\|\Theta(k) -\Theta^* \big\|_2^2 + \frac{7}{3}\eta  \mathcal{L}_S(\Theta^*) +\frac{\eta}{3}  \mathcal{L}_S(\Theta(k ) )  \nonumber\\
          &\le \big\| \Theta(0)-\Theta^* \big\|_2^2 + \frac{7}{3}\eta k \mathcal{L}_S(\Theta^*)   + \frac{\eta}{3} \sum_{t=0}^k \mathcal{L}_S(\Theta(t))\nonumber\\
          &\le   \frac{6}{5   }\big\| \Theta(0)-\Theta^* \big\|_2^2 +  \frac{14}{5   }\eta k \mathcal{L}_S(\Theta^*)      + \frac{\eta}{15} \mathcal{L}_S(\Theta(0)) \nonumber\\
          &\le 2 \big\| \Theta(0)-\Theta^* \big\|_2^2,
    \end{align}
where the first inequality used \eqref{eq:k+1-*} with $t=k$, and in the second inequality we used the first inequality recursively, and the last second inequality is according to the induction assumption $\eta \sum_{s=1}^{k }  \mathcal{L}_S(\Theta(s ) ) 
         \le \frac{7}{5} \eta k \mathcal{L}_S(\Theta^*) + \frac{3}{5 }   \|\Theta(0)-\Theta^* \|_2^2   +   \frac{  \eta}{5} \mathcal{L}_S(\Theta(0))$, and in the last inequality we have used $\big\|\Theta(0) - \Theta^{*}\big\|_2^2 \ge  4 \max \big\{  \eta T \mathcal{L}_S(\Theta^*),  \eta \mathcal{L}_S(\Theta(0)) \big\} $ and $\eta \le 1$. 

The second part of the induction is proved by noting that
\begin{align*}
    \|\Theta(k+1) - \Theta(0)\|_2 \le    \|\Theta(k+1) - \Theta^*\|_2  +    \|\Theta^* - \Theta(0)\|_2  \le 3\|\Theta(0)-\Theta^*\|_2 .
\end{align*}

Finally, note  \eqref{eq:descent} implies $\mathcal{L}_S(\Theta(t+1 ) )\le \mathcal{L}_S(\Theta(t  )  ) $ holds for all $t\in [T-1]$, then
\begin{align}\label{eq:sum-LS}
    \mathcal{L}_S(\Theta(T ) ) 
       & \le  \frac{1}{T} \sum_{k=1}^{T }  \mathcal{L}_S(\Theta(k ) )\le  \frac{6}{5}   \mathcal{L}_S(\Theta^*) + \frac{3}{5\eta T  }  \big\|  \Theta(0)-\Theta^*\big\|_2^2   +  \frac{1}{5 T} \mathcal{L}_S(\Theta(0))\\
       &\le  2 \mathcal{L}_S(\Theta^*) + \frac{1}{\eta T}  \big\|  \Theta(0)-\Theta^*\big\|_2^2,\nonumber%\le   \frac{7}{5\eta T}  \big\|  \Theta(0)-\Theta^*\big\|_2^2 
\end{align} 
which in the last inequality we have used the condition $ \|\Theta(0) - \Theta^{*} \|_2^2 \ge  4 \eta \mathcal{L}_S(\Theta(0))$. The proof is completed. 
\end{proof}
\begin{assumption}[Weak Realizability\label{ass:stand-realizability}]
For any $\epsilon>0$, there exists $\Theta^\epsilon $ such that $\L_S(\Theta^\epsilon ) \le \epsilon$.
\end{assumption}
\begin{lemma}\label{lem:reali}
Suppose Assumption \ref{ass:stand-realizability} holds. 
For any $\epsilon>0$, one can define a non-increasing function
$g:\mathbb{R}_+\to\mathbb{R}_+$ such that there exists $\Theta^\epsilon$ satisfying
\[
\mathcal{L}_S(\Theta^\epsilon)\le \epsilon
\quad\text{and}\quad
\|\Theta^\epsilon-\Theta(0)\|_2 = g(\epsilon).
\]
As a corollary, Assumption~\ref{ass:stand-realizability} implies  Assumption~\ref{ass:realizability}.
\end{lemma}
\begin{proof}
    Define $\Omega_n:= \{\Theta: \L_S(\Theta) \le 1/n\}$ and $\Omega_0 := \R^{mp(d+1)}$.
    Then, $\Omega_{n} \subset \Omega_{n-1}$ for all $n\in\mathbb{N}$.
    Further, by the continuity of $\L_S(\cdot)$, we know each $\Omega_n$ is closed.
    Then, there exists $\Theta_n$ satisfying $\Theta_n = \arg\min_{\Theta\in\Omega_n}\|\Theta - \Theta(0)\|_2$.
    For any $\epsilon$, there exists $n\in\mathbb{N}$ such that $\epsilon\in[\frac{1}{n}, \frac{1}{n-1})$ (we take $\frac{1}{0} = \infty$).
    Setting $\Theta^\epsilon = \Theta_n$, we know $\L_S(\Theta^\epsilon) = \L_S(\Theta_n) \le \frac{1}{n} \le \epsilon$.
    Then, for any $\epsilon$ and the associated $n$, we know 
    \begin{align*}
        g(\epsilon) = \big\|\Theta_n - \Theta(0)\big\|_2 = \min_{\Theta\in\Omega_n}\big\|\Theta - \Theta(0)\big\|_2.
    \end{align*}
    Now we show $g(\epsilon)$ is non-increasing.
    Specifically, for any $0<\epsilon_1\le\epsilon_2$.
    For the case $\epsilon_1,\epsilon_2\in(\frac{1}{n+1},\frac{1}{n}]$ for some $n$, we know $g(\epsilon_1) = g(\epsilon_2) = \min_{\Theta\in\Omega_n}\|\Theta - \Theta(0)\|_2$.
    For the case $\frac{1}{n_1 + 1} < \epsilon_1 \le \frac{1}{n_1} \le \frac{1}{n_2 + 1} < \epsilon_2 \le \frac{1}{n_2}$ for some $n_1 > n_2$, we know 
    $$g(\epsilon_1) = \min_{\Theta\in\Omega_{n_1}}\big\|\Theta - \Theta(0)\big\|_2 \ge \min_{\Theta\in\Omega_{n_2}}\big\|\Theta - \Theta(0)\big\|_2 = g(\epsilon_2),$$
    where the inequality follows from $\Omega_{n_1} \subset \Omega_{n_2}$.
    This completes the proof of the first part of the lemma.

    Note that the condition $\|\Theta^\epsilon-\Theta(0)\|_2 = g(\epsilon)$ implies $\|\Theta^\epsilon-\Theta(0)\|_2 \le g(\epsilon)$.
    Then, Assumption \ref{ass:realizability} holds.
\end{proof}

Now, we give the proof of Theorem~\ref{thm:risk-reali} as follows.
Observe that the only source of randomness in Theorem~\ref{thm:risk-reali} comes from the initialization, and the conclusion holds on the event~\eqref{eq:bound_c0}. 
Therefore, we may equivalently restate Theorem~\ref{thm:risk-reali} on the event~\eqref{eq:bound_c0}, replacing the statement ``with probability at least $1-\delta$” by the condition~\eqref{eq:bound_c0}.
\begin{theorem}[Restatement of Theorem~\ref{thm:risk-reali}]
Let \eqref{eq:bound_c0} and Assumptions  \ref{ass:sigma}, \ref{ass:loss} and \ref{ass:realizability}  hold.
Assume $\eta \le  \min\big\{ g^2(1),g^2(1) ( \mathcal{L}_S(\Theta(0)))^{-1}\big\}$,  $m\gtrsim  { \big( \log( {m }/{\delta}) +  g^2(\frac{1}{T})  \big)} g^4(\frac{1}{T})$. 
Then, it holds that $\|\Theta(k)-\Theta^{\frac{1}{T}}\|_2^2\leq 4g^2(\frac{1}{T})$ and
\[ \mathcal{L}_S(\Theta(T ) )  
       \le \frac{1}{T}\sum_{k=1}^T\L_S(\Theta(k)) \lesssim  \frac{2\eta + g^2\big(\frac{1}{T}\big)}{\eta T}. \]
\end{theorem}
\begin{proof}
    Assumption~\ref{ass:realizability} with $\epsilon = 1/T$ implies that there exists a non-increasing function $g:\rbb_+\mapsto\rbb_+$ such that there exists $\Theta^{1/T} $ with
    \[
        \mathcal{L}_S(\Theta^{1/T})\leq \frac{1}{T} \quad\text{and}\quad \|\Theta(0) - \Theta^{1/T}\|_2= g\Big(\frac{1}{T}\Big) \ge g(1).
    \] 
Recall that we denote $\Lambda_{\Theta^{1/T}} = \|\Theta(0)- \Theta^{1/T}\|_2$ and $\mathfrak{C}_S(\Theta^{1/T}) = 2\eta T + \|\Theta(0) - \Theta^{1/T}\|_2^2$.
Then, the conditions $m \gtrsim (\log(m/\delta) + g^2(1/T))g^4(1/T)$ and $\eta \lesssim \min\big\{g^2(1),g^2(1)\big(\mathcal{L}_S(\Theta(0))\big)^{-1}\big\}$ ensure that all the conditions in Theorem~\ref{thm:opt-error} are satisfied.
Applying Theorem~\ref{thm:opt-error} with $\Theta^* = \Theta^{1/T}$ implies that
\[ \mathcal{L}_S(\Theta(T ) )  
        \lesssim \frac{2\eta T\L_S(\Theta^{1/T}) + \Lambda_{\Theta^{1/T}}}{\eta T} \le  \frac{2\eta + g^2\big(\frac{1}{T}\big)}{\eta T}. \]
This completes the proof of the Theorem. 
\end{proof}

% \subsection{Proofs under NTK Separability}\label{sec:appen-ntk}
The proof of Theorem~\ref{thm:ntk} is given as follows.
\begin{proof}[Proof of Theorem~\ref{thm:ntk}]
    We first bound the values of $f_{\Theta(0)}(\cdot)$ over the training dataset $S$.
    Recall that we defined $\bh(v) = [b_1(v), \ldots, b_{p}(v)]^\top\in\R^{p}$ for any scalar $v\in\R$  and  $\bh(\bu) = [\bh(u_1)^\top,\ldots,\bh(u_s)^\top]^\top\in\R^{sp}$ for vector $\bu = [u_1,\ldots,u_s]^\top \in \R^{s}$.
    Note that 
    $$f_{\Theta(0)}(\bx) = \frac{1}{\sqrt{m}} \bc(0)^\top\bh\Big(\sigma\big(\frac{1}{\sqrt{d}}\bA\bh(\bx)\big)\Big) = \frac{1}{\sqrt{m}} \sum_{i=1}^m \bc_i(0)^\top \bh\big(\sigma(\ba_i^\top\bh(\bx))/\sqrt{d}\big).$$
    Denote $\bh_{i,j} = \bh\big(\sigma(\ba_i^\top\bh(\bx_j))/\sqrt{d}\big)$ for all $j\in[n]$.
    Then, from the fact $\bc_i(0)\sim \N(0, \bfI_p)$ we know
    \begin{align*}
        f_{\Theta(0)}(\bx_j) = \frac{1}{\sqrt{m}} \sum_{i=1}^m \bc_i(0)^\top \bh_{i,j} \sim \N\Big(0, \frac{\sum_{i=1}^m\|\bh_{i,j}\|_2^2}{m}\Big).
    \end{align*}
    Then, $f_{\Theta(0)}(\bx_j)/(\sum_{i=1}^m\|\bh_{i,j}\|_2^2/m)^{1/2} \sim \N(0,1)$ for all $j\in[n]$.
    According to $(2.3)$ in \cite{vershynin2018high}, we know with probability at least $1-\delta/2$ over $\bc(0)$, it holds for all $j\in[n]$ that
    \begin{align}\label{eq:bound-f-init}
        |f_{\Theta(0)}(\bx_j)| \le \sqrt{2\log\Big(\frac{2n}{\delta}\Big)} \bigg(\frac{\sum_{i=1}^m\|\bh_{i,j}\|_2^2}{m}\bigg)^{\frac{1}{2}} \le B_b \sqrt{2p\log\Big(\frac{2n}{\delta}\Big)},
    \end{align}
    where the last inequality used Assumption \ref{ass:sigma}.
    Let $\Theta_0\in\R^{m(d+1)p}$ be the parameter that satisfies the NTK separable assumption (see Assumption \ref{ass:ntk}).
    Take $\tau \asymp \big(\log(T) + \sqrt{\log(n/\delta)}\big)/\gamma$, we denote $\Theta_\tau = \Theta(0) + \tau\Theta_0 = [\bc_\tau^\top, \ba_\tau^\top]^\top$.
    For any $j\in[n]$, from Taylor's theorem we know there is a $\Theta_\tau' \in[\Theta(0) , \Theta_\tau]$, it holds that
    \begin{align*}
        y_jf_{\Theta_\tau}(\bx_j) &= y_jf_{\Theta(0)}(\bx_j) + y_j\big\langle \nabla f_{\Theta(0)}(\bx_j), \Theta_\tau - \Theta(0) \big\rangle_2 + \frac{1}{2} y_j\big(\Theta_\tau - \Theta(0)\big)^\top \nabla^2 f_{\Theta_\tau'}(\bx_j)\big(\Theta_\tau - \Theta(0)\big)\\
        &\ge -|f_{\Theta(0)}(\bx_j)| + \tau \ y_j\big\langle \nabla f_{\Theta(0)}(\bx_j), \Theta_0 \big\rangle_2 - \frac{\tau^2}{2}\big\|\nabla^2 f_{\Theta_\tau'}(\bx_j)\big\|_2\\
        &\ge \tau\gamma - B_b \sqrt{2p\log\Big(\frac{n}{\delta}\Big)} - \frac{\tau^2}{2}\big\|\nabla^2 f_{\Theta_\tau'}(\bx_j)\big\|_2\\
        &\ge \tau\gamma - B_b \sqrt{2p\log\Big(\frac{n}{\delta}\Big)} - \frac{\tau^2}{2} \frac{ C_{\sigma, b} \, p^{\frac{ 3}{2}}\big(\sqrt{p} +  \sqrt{\log({m}/{\delta})} + \tau  \big)}{\sqrt{m}}\\
        &\ge \frac{\tau\gamma}{2} - \frac{\tau^2}{2} \frac{ C_{\sigma, b} \, p^{\frac{ 3}{2}}\big(\sqrt{p} +  \sqrt{\log({m}/{\delta})} + \tau  \big)}{\sqrt{m}}\\
        & \ge \frac{\tau\gamma}{4} \gtrsim \log(T),
    \end{align*}
    where the first inequality used Assumption \ref{ass:ntk} with $\|\Theta_0\|_2 = 1$, the second inequality again used Assumption \ref{ass:ntk} and the estimate of $|f_{\Theta(0)}(\bx_j)|$ from above, the third inequality used \eqref{eq:Hessian_loss} with $|\ell'(\cdot)| \le 1$ and the fact $\max_{i\in[m]}\|(\bc_\tau)_i - \bc_i(0)\|_2 \le \|\Theta_\tau - \Theta(0)\|_2 = \tau \|\Theta_0\|_2$, and the fourth inequality used the fact $\tau \asymp \big(\log(T) + \sqrt{\log(n/\delta)}\big)/\gamma$, and in the last second inequality we have used the condition $m \gtrsim \log(m/\delta)\big(\log^6(T) + \log^3(n/\delta)\big)/\gamma^6 \gtrsim \big(\sqrt{\log(m/\delta)} + \tau)^2\tau^2/\gamma^2$.
    
    Then, for all $j\in[n]$, it holds that
    \begin{align*}
        \ell\big(y_jf_{\Theta_\tau}(\bx_j)\big) = \log\big(1 + \exp\big(-y_jf_{\Theta_\tau}(\bx_j)\big)\big) \le \exp\big(-y_jf_{\Theta_\tau}(\bx_j)\big) \le \exp(-\log(T)) = \frac{1}{T}.
    \end{align*}
    Then, we know $\L_S(\Theta_\tau) \le \frac{1}{T}$ and the corresponding $g(1/T) = \|\Theta_\tau - \Theta(0)\|_2 = \tau \asymp \big(\log(T) + \sqrt{\log(n/\delta)}\big)/\gamma$.
    Then, we know Assumption \ref{ass:realizability} holds with $\epsilon = 1/T$.
    Plugging the estimate of $g(1/T)$ back into Theorem \ref{thm:risk-reali}, and note that \eqref{eq:bound_c0} with $\delta$ replaced by $\delta/2$ holds with probability at least $1-\delta/2$ over initialization $\bc(0)$, we can obtain the desired results.
\end{proof}

\section{Proofs for Generalization}\label{sec:appe-gen}
To establish generalization results of GD, we introduce the concept of algorithmic stability. 
For a randomized algorithm $\A$, let $\A(S)\in\R^{mp(d+1)}$ be the output of $\A$ based on dataset $S$. The on-average argument stability measures the on-average sensitivity of the output up to the perturbation of the dataset.
\begin{definition}[On-average argument stability \cite{lei2020fine}]\label{def:stability}
Let $S=\{z_1,\ldots,z_n\}$ and $\widetilde{S}=\{z'_1,\ldots,z'_n\}$ be drawn independently from $\mathcal{P}$. For any $i\in[n]$, define $S^{ i}=\{ z_1,\ldots,z_{i-1},z_i', z_{i+1},\ldots,z_n\}$.
Let $\mathcal{A}(S)$ and $\mathcal{A}(S^{ i})$ be produced by an randomized algorithm $\mathcal{A}$ based on $S$ and $S^{ i}$ respectively. We say $\A$ is on-average argument $\epsilon$-stable if \[\E_{S,\widetilde{S},\A}\Big[\frac{1}{n}\sum_{i=1}^n\|\A(S)-\A(S^{ i})\|_2 \Big] \le \epsilon .\]
\end{definition}

We consider using the connection between the on-average argument stability and generalization error bounds \cite{lei2020fine}. 
\begin{lemma}[\cite{lei2020fine}]\label{lem:connection}
If $\A$ is on-average argument $\epsilon$-stable and the loss $\ell$ is $L$-Lipschitz with respect to $\A(S)$, then
% \junyu{It seems that the condition is not correct? The Lipschitz constant should be with respect to $\A(S)$ instead of the $f_{\A(S)}$.}
\[\E_{S,\mathcal{A}} \big[ \mathcal{L}(\mathcal{A}(S)) - \mathcal{L}_S(\mathcal{A}(S)) \big]\le  2L \epsilon. 
\]
\end{lemma}
We will first estimate stability bound  of GD and then using the above connection to obtain the generalization error bounds. 
Our stability analysis requires the following lemma.  
\begin{lemma}[Expansiveness of GD]\label{lem:exansiveness}
    Suppose Assumptions~\ref{ass:sigma} and  \ref{ass:loss} hold.
    Let $\alpha\in[0,1]$.
    For all $\Theta, \Theta'\in\R^{mp(d+1)}$, we denote $\Theta_\alpha=\alpha\Theta +(1-\alpha) \Theta'$.
    Then, it holds for any training dataset $S$ that
    \[\big\| \Theta -\eta \nabla \mathcal{L}_S(\Theta )    - \big(\Theta'-\eta \nabla \mathcal{L}_S(\Theta' )      \big)\big\|_2 \le \max_{\alpha\in[0,1]}\big\{ 1 - \eta \lambda_{\min}(\nabla^2 \mathcal{L}_S(\Theta_\alpha)), \eta \lambda_{\max}(\nabla^2 \mathcal{L}_S(\Theta_\alpha))   \big\} \big\| \Theta - \Theta' \big\|_2.\]
\end{lemma}
\begin{proof}
    The proof can be directly obtained by Lemma B.1 in \citep{taheri2024generalization}. Specifically, similar to $(32)$ of their Lemma B.1, one can show that 
    \[\big\| \Theta -\eta \nabla \mathcal{L}_S(\Theta )    - \big(\Theta'-\eta \nabla \mathcal{L}_S(\Theta' )     \big)\big\|_2 \le \max_{\alpha\in[0,1]}\big\| \mathbf{I} - \eta \nabla^2 \mathcal{L}_S(\Theta_\alpha)  \big\|_{op} \big\| \Theta - \Theta' \big\|_2.\]
    Note that
    \begin{align*}
        \max_{\alpha\in[0,1]}\big\| \mathbf{I} - \eta \nabla^2 \mathcal{L}_S(\Theta_\alpha)  \big\|_{op}&\le \max \big\{ \big|1-\eta \lambda_{\min}(\nabla^2 \mathcal{L}_S(\Theta_\alpha))\big|,  \big|1-\eta \lambda_{\max}(\nabla^2 \mathcal{L}_S(\Theta_\alpha))\big| \big\}\nonumber\\
        &\le \max\big\{ 1 - \eta \lambda_{\min}(\nabla^2 \mathcal{L}_S(\Theta_\alpha)), \eta \lambda_{\max}(\nabla^2 \mathcal{L}_S(\Theta_\alpha))   \big\}.
    \end{align*}
   Combining the above two observations and noting that \eqref{eq:bound_c0} holds with probability at least $1-\delta$ over initialization $\bc(0)$ complete the proof of the theorem.
\end{proof}
 
Recall that $c_{\max}=C_{\sigma, b} \, p^{\frac{3}{2}} \big(\sqrt{\log(\frac{m }{\delta})} + \sqrt{p} + 3\big\|\Theta(0) - \Theta^*\big\|_2  \big)$ and $\rho_\ell = C_{\sigma, b}p^2(p + 3 \|\Theta(0) - \Theta^*\|_2^2  )$. The following lemma gives the on-average argument stability bound  of GD. In this section, we assume the reference model $\Theta^*$ is independent of $S$.

Observe that the only source of randomness in Theorem~\ref{thm:generalization} comes from the initialization, and the conclusion holds on the event~\eqref{eq:bound_c0}. 
Therefore, we may equivalently restate Theorem~\ref{thm:generalization} on the event~\eqref{eq:bound_c0}, replacing the statement ``with probability at least $1-\delta$” by the condition~\eqref{eq:bound_c0}.
\begin{lemma}[Restatement of Theorem~\ref{thm:generalization}]\label{thm:restate_generalization}
    Let $\delta\in(0,1)$.
    Suppose \eqref{eq:bound_c0}, \ref{ass:sigma} and \ref{ass:loss} hold.
    For any reference point $\Theta^* \in \R^{mp(d+1)}$ satisfies $ \|\Theta(0) - \Theta^{*} \|_2^2 \ge 8 \max  \big\{  \eta T \big(\mathcal{L}_S(\Theta^*) + \L_{\widetilde{S}}(\Theta^*)\big),  \eta \big(\mathcal{L}_S(\Theta(0)) + \L_{\widetilde{S}}(\Theta(0))\big)\big\} $, suppose $\eta \le \min\{1/2\rho_\ell, 1\}$, and
    \[m\gtrsim  p^3(\log(\frac{m}{\delta})+p + \|\Theta(0) - \Theta^* \|_2^2 )  \max\big\{ \|  \Theta(0) - \Theta^* \|_2^2,   \|\Theta(0)-\Theta^*\|_2^4 \big\}.\]
    Then, it holds that
   \begin{align*}
    \E_{S} \big[ \mathcal{L}(\Theta(T)) - \mathcal{L}_S(\Theta(T)) \big]\le  C_{\sigma, b}p^2\big(p + p\|\Theta^* - \Theta(0)\|_2^2\big)\frac{  \eta  }{n} \E_{S} \Big[\sum_{t=0}^T \mathcal{L}_S\big(\Theta(t) \big) \Big].
\end{align*}
\end{lemma}
\begin{proof}
    For any $i\in[n]$, let $\{\Theta^{i}(k)\}$ and $\{\Theta^{-i}(k)\}$ be produced by GD based $S^{ i}$ and $S^{ -i}$, respectively.
    Observe that the condition $ \|\Theta(0) - \Theta^{*} \|_2^2 \ge 8 \max  \{  \eta T \big(\mathcal{L}_S(\Theta^*) + \L_{\widetilde{S}}(\Theta^*)\big),  \eta \big(\mathcal{L}_S(\Theta(0)) + \L_{\widetilde{S}}(\Theta(0))\big)\} $ implies $ \|\Theta(0) - \Theta^{*} \|_2^2 \ge 4 \max  \{  \eta T \mathcal{L}_{S^i}(\Theta^*) , \eta T \mathcal{L}_{S^{-i}}(\Theta^*), \eta \mathcal{L}_{S^i}(\Theta(0)), \eta \mathcal{L}_{S^{-i}}(\Theta(0))\}$ for all $i\in[n]$.
    Then, Theorem \ref{thm:restate_opt-error} can be applied to any $S^i$ and $S^{-i}$ and the associated GD outputs $\Theta^i(k)$ and $\Theta^{-i}(k)$.

  For any $i\in[n]$, define $S^{-i}=\{z_1,\ldots,z_{i-1},z_{i+1},\ldots,z_n\} = S\backslash\{z_i\}$.  From the update rule of GD and the self-bounding property of $\ell$, for any $k=1,\ldots,T-1$, it holds
    \begin{align}\label{eq:theta-thetai}
        & \big\|  \Theta(k+1) - \Theta^{ i}(k+1)\big\|_2 \nonumber\\
        &=\big\|  \Theta(k) - \eta \nabla \mathcal{L}_S(\Theta(k)) - \Theta^{ i}(k) + \eta \nabla \mathcal{L}_{S^{ i}}(\Theta^{ i}(k))\big\|_2\nonumber\\
        &\le \big\|  \Theta(k) - \frac{\eta (n-1)}{n} \nabla \mathcal{L}_{S^{-i}}(\Theta(k)) - \Theta^{ i}(k) + \frac{\eta (n-1)}{n} \nabla \mathcal{L}_{S^{-i}}(\Theta^{ i}(k))\big\|_2 + \frac{\eta}{n}\big\| \nabla \ell\big(y_i f_{\Theta(k)}(\bx_i)\big)-\nabla \ell\big(y'_i f_{\Theta^i(k)}(\bx'_i)\big)\big\|_2\nonumber\\
         &\le G^i_\alpha(k) \big\|\Theta(k) -  \Theta^{i}(k)  \big\|_2  + \frac{\eta}{n}\big[  \ell\big(y_i f_{\Theta(k)}(\bx_i)\big) \big\|\nabla f_{\Theta(k)}(\bx_i)\big\|_2 +  \ell\big(y'_i f_{\Theta^i(k)}(\bx'_i)\big)\big\|\nabla f_{\Theta^i(k)}(\bx'_i)\big\|_2\big],
    \end{align}
where in the last inequality we have used Lemma~\ref{lem:exansiveness} with $S=S^{-i}$.
Here, $G^i_\alpha(k)= \max_{\alpha\in[0,1]}\big\{1 - \eta \lambda_{\min}(\nabla^2 \mathcal{L}_{S^{-i}}(\Theta_\alpha)), \eta \lambda_{\max}(\nabla^2 \mathcal{L}_{S^{-i}}(\Theta_\alpha(k)))   \big\}$ with $\Theta_\alpha(k) =\alpha \Theta(k) + (1-\alpha) \Theta^{ i}(k)$.
From Proposition~\ref{pro:smooth} with $S=S^{-i}$, we have the following estimates
\[ \lambda_{\min}\big(\nabla^2 \mathcal{L}_{S^{-i}}(\Theta_\alpha(k))\big) \ge -\frac{ C_{\sigma , b} \, p^{\frac{3}{2}}  \,\big(\!\sqrt{\log(\frac{m }{\delta})}  + \sqrt{p}  + \max_{i\in m}\|\bc_i(0) - [\bc_{\alpha}(k)]_i\|_2\big)}{\sqrt{m}}  \mathcal{L}_{S^{-i}}( \Theta_\alpha(k) ),\]
\[\lambda_{\max}\big(\nabla^2  \mathcal{L}_{S^{-i}}( \Theta_\alpha(k) )\big) \le
     C_{\sigma,b } \, p^2 \Big(p  + \max_{i\in m}\|\bc_i(0) - [\bc_{\alpha}(k)]_i\|_2^2\Big).\]

To control $G^i_\alpha(k)$, we need to estimate $\lambda_{\min}(\nabla^2  \mathcal{L}_{S^{-i}}( \Theta_\alpha(k) ))$ and $\lambda_{\max}(\nabla^2  \mathcal{L}_{S^{-i}}( \Theta_\alpha(k) ))$, respectively. 
Since $\Theta(k)$ and $\Theta^{ i}(k)$ are the outputs of GD, from Theorem~\ref{thm:restate_opt-error} we know 
\[\max_{i\in m}\|\bc_i(0) - [\bc_{\alpha}(k)]_i\|_2 \le \|  \Theta(0) - \Theta_\alpha(k) \|_2\le  \max \{ \|  \Theta(0) - \Theta(k) \|_2,  \|\Theta(0) - \Theta^{ i}(k) \|_2  \} \le  3 \|  \Theta(0) - \Theta^* \|_2.\]
Recall that we defined $c_{\max}=C_{\sigma, b} \, p^{\frac{3}{2}} \big(\sqrt{\log(\frac{m }{\delta})} + \sqrt{p} + 3\big\|\Theta(0) - \Theta^*\big\|_2  \big)$.
It then follows that
\begin{align}\label{eq:gamma_alpha}
    \lambda_{\min}\big(\nabla^2 \mathcal{L}_{S^{-i}}(\Theta_\alpha(k))\big) \ge -\frac{ c_{\max}}{\sqrt{m}}  \mathcal{L}_{S^{-i}}( \Theta_\alpha(k) ), 
\end{align}
\begin{align}\label{eq:lambdamax_alpha}
     \lambda_{\max}\big(\nabla^2\mathcal{L}_{S^{-i}} (\Theta_\alpha(k))\big)  \le     C_{\sigma,b } \, p^2 \Big(p + 3 \|  \Theta(0) - \Theta^* \|_2^2\Big).   
\end{align}
Applying Theorem~\ref{thm:restate_opt-error} again, we have
\[  \big\|  \Theta(k) - \Theta^{ i}(k)\big\|_2\le   \big\|  \Theta(k) - \Theta(0)\big\|_2+ \big\|\Theta(0) - \Theta^{ i}(k)\big\|_2 \le  6\big\|  \Theta(0) - \Theta^*\big\|_2, \]
Define $D=6 \|\Theta(0)-\Theta^*\|_2 $ and $\kappa= \frac{c_{\max} }{\sqrt{m}} $. 
The condition $m\ge 36 c_{\max}^2 \|\Theta(0)-\Theta^*\|_2^4$ implies $  D^2 \kappa \le 1$. 
Hence, Lemma~\ref{lem:quasi-convexity} with $\tau=2$ yields
\begin{align}\label{eq:L_alpha}
    \max_{\alpha\in[0,1]} \mathcal{L}_{S^{-i}}( \Theta_\alpha(k) ) &\le 2 \max \big\{ \mathcal{L}_{S^{-i}}( \Theta (k) ), \mathcal{L}_{S^{-i}}( \Theta^{ i}(k) )  \big\}\nonumber\\
    &= \frac{2}{n-1}\max \big\{ (n-1)\mathcal{L}_{S^{-i}}( \Theta (k) ), (n-1)\mathcal{L}_{S^{-i}}( \Theta^{ i}(k) )\big\}\nonumber\\
    &\le \frac{2}{n-1}\max \big\{ n\mathcal{L}_{S}( \Theta (k) ), n\mathcal{L}_{S^{i}}( \Theta^{ i}(k) ) \big\}\nonumber\\
    &\le 4\max\big\{ \mathcal{L}_{S}( \Theta (k) ), \mathcal{L}_{S^{i}}( \Theta^{ i}(k) ) \big\}.
\end{align}
Plugging \eqref{eq:L_alpha} back into \eqref{eq:gamma_alpha} yields
    \begin{align}\label{eq:gamma_alpha-bound}
         \lambda_{\min}\big(\nabla^2 \mathcal{L}_S(\Theta_\alpha(k))\big)%&=\frac{\!C_{\sigma\!, b} \, p^{\frac{3}{2}}d \,\big(\!\sqrt{\log(\frac{m}{\delta})} \!+\!\! \sqrt{p} \!+\! \max_{i\in m}\|\bc_i(0) - [\bc_{\alpha}(k)]_i\|_2\big)}{\sqrt{m}}  \mathcal{L}_S( \Theta_\alpha(k) )\nonumber\\
         & \ge -\frac{4 c_{\max} }{\sqrt{m}}   \max \big\{ \mathcal{L}_{S}( \Theta (k) ), \mathcal{L}_{S^{i}}( \Theta^{ i}(k) )  \big\}. 
    \end{align} 
Note that $\eta\le 1/ \rho_{\ell} \le \big(C_{\sigma,b } \,  p^2   (  p + 3 \|\Theta(0) - \Theta^*\|_2^2  )  \big)^{-1}$ implies $\eta \lambda_{\max}(\nabla^2\mathcal{L}_{S^{-i}}(\Theta_\alpha(k)))\le 1$.
Then, we have
\begin{align*}
    G^i_\alpha(k)  \le   1+  \frac{ 4\eta c_{\max} }{\sqrt{m  }}  \max \big\{ \mathcal{L}_{S}( \Theta (k) ), \mathcal{L}_{S^{i}}( \Theta^{ i}(k) ) \big\}=: 1+ M^i(k).
\end{align*}
Further, from Theorem~\ref{thm:restate_opt-error} and Lemma \ref{lem:hessian} and the condition $m \gtrsim \log(1/\delta)$, it holds for any $k\in[T-1]$ and any $\bx\in\X$ that
\begin{align}\label{eq:bound_gradient}
    \big\|\nabla f_{\Theta(k)}(\bx)\big\|_2 &\le C_{\sigma, b}p\big(\sqrt{p} + \max_{i\in[m]}\|\bc_i(k) - \bc_i(0)\|_2\big) \le C_{\sigma, b}p\big(\sqrt{p} + \|\Theta(k) - \Theta(0)\|_2\big)\nonumber\\
    &\le C_{\sigma, b}p\big(\sqrt{p} + 3\|\Theta^* - \Theta(0)\|_2\big).
\end{align}
Plugging the estimates of $G^i_\alpha(k)$ and $\|\nabla f_{\Theta(k)}(\bx)\|_2$ back into \eqref{eq:theta-thetai}, we get 
\begin{align}\label{eq:theta-diff}
        & \big\|  \Theta(k+1) - \Theta^{ i}(k+1)\big\|_2 \nonumber\\
         &\le \big(1+M^i(k) \big) \big\|\Theta(k) -  \Theta^{ i}(k)  \big\|_2  + C_{\sigma, b}p\big(\sqrt{p} + 3\|\Theta^* - \Theta(0)\|_2\big)\frac{\eta}{n} \big[  \ell\big(y_i f_{\Theta(k)}(\bx_i)\big) +  \ell\big(y'_i f_{\Theta^i(k)}(\bx'_i)\big)\big]\nonumber\\
         &\le C_{\sigma, b}p\big(\sqrt{p} + 3\|\Theta^* - \Theta(0)\|_2\big)\frac{\eta}{n} \sum_{t=0}^k\Big( \prod_{s=t+1}^{k} \big(1+M^i(s) \big) \Big) \big[  \ell\big(y_i f_{\Theta(t)}(\bx_i)\big) +  \ell\big(y'_i f_{\Theta^i(t)}(\bx'_i)\big)\big] \nonumber\\
         &\le  C_{\sigma, b}p\big(\sqrt{p} + 3\|\Theta^* - \Theta(0)\|_2\big)\frac{\eta}{n} \sum_{t=0}^k\Big( \exp\big( \sum_{s=t+1}^{k} M^i(s)  \big) \Big) \big[  \ell\big(y_i f_{\Theta(t)}(\bx_i)\big) +  \ell\big(y'_i f_{\Theta^i(t)}(\bx'_i)\big)\big] \nonumber\\
         &\le C_{\sigma, b}p\big(\sqrt{p} + 3\|\Theta^* - \Theta(0)\|_2\big)\frac{\eta}{n} \exp\big( \max_{i\in[n]} \sum_{s=1}^{k}  M^i(s) \big) \sum_{t=0}^k \big[  \ell\big(y_i f_{\Theta(t)}(\bx_i)\big) +  \ell\big(y'_i f_{\Theta^i(t)}(\bx'_i)\big)\big],
\end{align}
where the third inequality uses the inequality $(1+x)\le e^x$ for all $x\ge 0$.

Since $\sqrt{m}\ge 2 \eta c_{\max} \max_{i\in[n]} \big\{ \sum_{s=1}^{T} \mathcal{L}_{S }( \Theta (s) ), \sum_{s=1}^{T} \mathcal{L}_{S^{ i}}( \Theta^{ i}(s) ) \big\} $, then
According to Theorem~\ref{thm:restate_opt-error} with training dataset $S$ and $S^{i}$, respectively,
it holds that
\begin{align*}
    \max_{i\in[n]} \sum_{s=1}^{k}  M^i(s) &\le \frac{ 2\eta c_{\max}}{\sqrt{m }} \max_{i\in[n]} \Big\{ \sum_{s=1}^{k} \mathcal{L}_{S }( \Theta (s) ), \sum_{s=1}^{k} \mathcal{L}_{S^{i}}( \Theta^{ i}(s) ) \Big\}\\
    &\le \frac{ 2\eta c_{\max}}{\sqrt{m }} \max_{i\in[n]}\Big\{2T\L_S( \Theta^*) + 2T\mathcal{L}_{S^{i}}( \Theta^* ) + \frac{1}{\eta}\big\|\Theta(0) - \Theta^*\big\|_2^2\Big\}\\
    &\le \frac{ 2\eta c_{\max}}{\sqrt{m }} \max_{i\in[n]}\Big\{\frac{1}{4\eta}\big\|\Theta(0) - \Theta^*\big\|_2^2 + \frac{1}{\eta}\big\|\Theta(0) - \Theta^*\big\|_2^2\Big\}\\
    &\le 1,
\end{align*}
where in the last second inequality used $ \|\Theta(0) - \Theta^{*} \|_2^2 \ge 8 \max  \{  \eta T \big(\mathcal{L}_S(\Theta^*) + \L_{\widetilde{S}}(\Theta^*)\big)\} \ge 8\eta T \max  \{  \mathcal{L}_S(\Theta^*), \L_{\widetilde{S}}(\Theta^*)\}$ due to $\L_S(\Theta^*) + \L_{\widetilde{S}}(\Theta^*) \ge \L_{S^i}(\Theta^*)$, and in the last inequality we have used $m \gtrsim c_{\max}^2 \asymp C_{\sigma, b}p^3(\log(\frac{m}{\delta})+p + \|\Theta(0) - \Theta^* \|_2^2 )\textbf{}$.

Putting the above observation  into \eqref{eq:theta-diff}
 yields
 \begin{align*} 
      \big\|  \Theta(k+1) - \Theta^{i}(k+1)\big\|_2  \le C_{\sigma, b}p\big(\sqrt{p} + 3\|\Theta^* - \Theta(0)\|_2\big)\frac{\eta}{n} \sum_{t=0}^k\big[  \ell\big(y_i f_{\Theta(k)}(\bx_i)\big) +  \ell\big(y'_i f_{\Theta^i(k)}(\bx'_i)\big)\big].
\end{align*}
Summing over $i\in[n]$ and taking an average, it holds
 \begin{align}\label{eq:on-average}
  \E_{S,\widetilde{S}}\Big[   \frac{1}{n}\sum_{i=1}^n  \big\|  \Theta(k+1) - \Theta^{i}(k+1)\big\|_2 \Big] &\le C_{\sigma, b}p\big(\sqrt{p} + 3\|\Theta^* - \Theta(0)\|_2\big)\frac{\eta}{n} \sum_{t=0}^k \E_{S}\big[\mathcal{L}_S\big(\Theta(t) \big)\big]
\end{align}
where we have used $\E_{S,\widetilde{S}}\big[ \mathcal{L}_{S^i}\big(\Theta^i(t) \big)\big]=\E_{S }  \big[\mathcal{L}_S\big(\Theta(t) \big)\big].$
The proof of the first inequality of the lemma is complete.

From \eqref{eq:bound_gradient} we know for all $k\in[T-1]$ and $(\bx,y)\in\X\times\Y$, the Lipschitz constant of $\ell$ with respect to $\Theta(k)$ can be controlled as $\|\nabla \ell(yf_{\Theta(k)}(\bx))\|_2 \le \|\nabla f_{\Theta(k)}(\bx)\|_2 \le C_{\sigma, b}p(\sqrt{p} + 3\|\Theta^* - \Theta(0)\|_2)$.
Then, from Lemma \ref{lem:connection} and the on-average stability inequality \eqref{eq:on-average}, it holds that
\begin{align*}
    \E_{S} \big[ \mathcal{L}(\Theta(T)) - \mathcal{L}_S(\Theta(T)) \big] \le  C_{\sigma, b}p^2\big(p + p\|\Theta^* - \Theta(0)\|_2^2\big)\frac{  \eta  }{n} \E_{S} \Big[\sum_{t=0}^T \mathcal{L}_S\big(\Theta(t) \big) \Big].
\end{align*}
By further noting that \eqref{eq:bound_c0} holds with probability at least $1-\delta$ over initialization $\bc(0)$, this completes the proof of the theorem.
\end{proof}

Combining Theorems~\ref{thm:opt-error} and~\ref{thm:generalization}, the proof of Theorem~\ref{thm:risk} follows immediately.

\begin{proof}[Proof of Theorem~\ref{thm:risk}]
Eq.~\eqref{eq:sum-LS} together with the condition for $m$ and the condition $\Lambda_{\Theta^*}^2 = \|\Theta(0) - \Theta^*\|_2^2 \ge  8 \max  \{  \eta T (\mathcal{L}_S(\Theta^*) + \mathcal{L}_{\widetilde{S}}(\Theta^*)),  \eta (\mathcal{L}_S(\Theta(0)) + \mathcal{L}_{\widetilde{S}}(\Theta(0))) \}$ show that
\begin{align*}
    \frac{1}{T} \sum_{k=0}^{T }  \mathcal{L}_S(\Theta(k ) ) &\le \frac{1}{T}\L_S(\Theta(0)) + 2\L_S(\Theta^*) + \frac{1}{\eta T}\big\|\Theta(0) - \Theta^*\big\|_2^2\\
    &\le \frac{2\eta T\L_S(\Theta^*) + \frac{9}{8}\big\|\Theta(0) - \Theta^*\big\|_2^2}{\eta T}\le \frac{2\mathfrak{C}_S(\Theta^*)}{\eta T},
\end{align*}
where $\mathfrak{C}_S(\Theta^*) = 2\eta T\L_S(\Theta^*) + \|\Theta(0) - \Theta^*\|_2^2$.
Combining this with Theorems~\ref{thm:restate_opt-error} and~\ref{thm:restate_generalization} yields
\[ \E_{S} \big[ \mathcal{L}(\Theta(T))  \big] \lesssim   \frac{\eta \Lambda_{\Theta^*}^2 }{n} \E_{S} \Big[\sum_{t=0}^T \mathcal{L}_S\big(\Theta(t) \big)\Big]+\ebb_S\big[\L_S(\Theta(T))\big]  \lesssim \big(\frac{\Lambda_{\Theta^*}^2 }{n} + \frac{1}{ \eta T}\big) \mathfrak{C}(\Theta^*), 
  \]  
 where $\mathfrak{C}(\Theta^*)=\E_S\big[\mathfrak{C}_S(\Theta^*)\big].$ 
By further noting that \eqref{eq:bound_c0} holds with probability at least $1-\delta$ over initialization $\bc(0)$, this completes the proof of the theorem.
\end{proof}

\begin{theorem}[Generalization under Realizability]\label{thm:risk-reali-gen}
Suppose Assumption~\ref{ass:realizability} holds.
Let $\eta$ be a constant step size satisfying 
$\eta \lesssim \min\big\{ g^2(1),\ g^2(1) (\mathcal{L}_S(\Theta(0)) )^{-1},\  (p^2 (p+g^2( {1}/{T}) ) )^{-1}\big\}.$ 
Assume $\eta T \gtrsim n$ and
$m \gtrsim \big(\log\!\big( {m}/{\delta}\big)+g^2( {1}/{T})\big)\,g^4( {1}/{T}).$
Then, with probability at least $1-\delta$ over the random initialization, it holds that
\[
\ebb_S\big[\mathcal{L}(\Theta(T))\big]
\ \lesssim\ \frac{\big(\eta+g^2(\tfrac{1}{T})\big)\,g^2(\tfrac{1}{T})}{n}.
\]
\end{theorem}
\begin{proof}
    The proof follows directly from Theorem~\ref{thm:risk}. 
    It suffices to verify that the width and step size conditions on $m$ and $\eta$ required by Theorem~\ref{thm:risk} are satisfied, which can be shown by arguments similar to those used in the proof of Theorem~\ref{thm:risk-reali}.
\end{proof}

\begin{proof}[Proof of Theorem~\ref{thm:ntk-gen}]
    Plugging the estimates of $\L_S(\Theta(0))$ and $g(1/T)$ (see the proof of Theorem \ref{thm:ntk}) into Theorem~\ref{thm:risk-reali-gen} completes the proof of the theorem.
\end{proof}

\section{Proofs for DP-GD}\label{sec:appen-dp}
\subsection{DP-GD Algorithm}
The detailed differentially private GD algorithm given in Section~\ref{sec:DP-GD} is described in Algorithm~\ref{alg1}. 
\begin{algorithm}[h]
\begin{algorithmic}[1]
\caption{Differentially Private GD for KANs}\label{alg1}
\STATE{\bf Inputs:}  Dataset $S= \{z_{i}\in \Z\}_{i=1}^n$, step size $ \eta>0$, iteration number $T>0$, \\
privacy parameters $\gep$, $\gd >0$.
\STATE{ Generate $\ba(0)$ and $\bc(0)$ via \eqref{eq:init-W}  }
\STATE Set $\widetilde{\Theta}(0)=(\ba(0),\bc(0))$.
\FOR { $k=0$ to $T-1$ } 
\STATE{Update $ \ba(k + 1) $ via  \eqref{eq:update-A}  }
\STATE{Update $ \bc(k+1) $ via   \eqref{eq:update-c} }
\STATE{Set $\widetilde{\Theta}(k+1)=( \ba(k + 1) , \bc(k+1))$}
\ENDFOR
\STATE  {\bf return:}   $\Theta_{\priv}:= \widetilde{\Theta}(T)=(\ba(T) ,\bc(T))$ 
\end{algorithmic}
\end{algorithm} 
\subsection{Privacy Guarantee}\label{subsec:appen-DP}
In this subsection, we provide proofs for privacy guarantee (i.e., Theorem \ref{thm:privacy}) of our DP-GD algorithm.
The proofs are based on Rényi differential privacy (RDP), a relaxation of classical differential privacy that allows for a more refined analysis of privacy loss.
We call two training 
\begin{definition}[RDP \cite{mironov2017renyi}]\label{def:RDP}
For $\lambda > 1$, $\rho > 0$, a randomized algorithm  $\A$ satisfies $(\lambda, \rho)$-RDP, if,  for all neighboring datasets $S$ and $S'$, we have 
    \begin{align*}       D_{\lambda}\big(\A(S)\parallel \A(S')\big):= \frac{1}{\lambda-1}\log \int  \Big( \frac{ P_{\A(S)}(\theta) }{ P_{\A(S')}(\theta) }  \Big)^\lambda   P_{\A(S')}(\theta) \, d\theta \le \rho,
    \end{align*} 
    where $P_{\A(S)}(\theta)$ and $P_{\A(S')}(\theta)$ are the density of $\A(S) $ and $\A(S')$, respectively. 
\end{definition}
To achieve DP, we need the concept of $\ell_2$-sensitivity defined as follows. 
\begin{definition}[$\ell_2$-sensitivity]\label{def:sensitivity}
The $\ell_2$-sensitivity of a function (mechanism) $\mathcal{M}:\mathcal{Z}^n \rightarrow \mathcal{W}$ is defined as 
$
\Delta  = \sup_{S, S'} \|\mathcal{M}(S) - \mathcal{M}(S')\|_2,
$ where $S$ and $S'$ are neighboring datasets.
\end{definition}
A basic mechanism to obtain RDP is Gaussian mechanism. 
\begin{lemma}[Gaussian mechanism \cite{mironov2017renyi}]\label{lem:gaussian-rdp} Consider a function $\mathcal{M}:  \mathcal{Z}^n\rightarrow \mathcal{R}^d$ with the $\ell_2$-sensitivity parameter $\Delta$,  and a dataset $S\subset\mathcal{Z}^n$.  
{The Gaussian mechanism $\mathcal{G}(S,\sigma)=\mathcal{M}(S)+\mathbf{b}$, where $\mathbf{b}\sim \mathcal{N}(0,\sigma^2\mathbf{I}_d)$, } satisfies $(\lambda,\frac{\lambda \Delta^2}{2\sigma^2})$-RDP. 
\end{lemma}

A connection $(\epsilon,\delta)$-DP and RDP is established in the following lemma. 
\begin{lemma}[From RDP to $(\epsilon,\delta)$-DP \cite{mironov2017renyi}]\label{lemma:RDP_to_DP}
	If a randomized algorithm  $\mathcal{A}$ satisfies $(\lambda,\rho)$-RDP with $\lambda > 1$, then $\mathcal{A}$ satisfies $(\rho+\log(1/\delta)/(\lambda-1),\delta)$-DP for all $\delta\in(0,1)$.
\end{lemma}

The following post-processing property enables flexible use of private data outputs while preserving rigorous privacy guarantees.
\begin{lemma}[Post-processing \cite{mironov2017renyi}]\label{lemma:post-processing}
Let  $\A: \mathcal{Z}^n \rightarrow \mathcal{W}_1 $  satisfy $(\lambda, \rho)$-RDP  and $f: \mathcal{W}_1 \rightarrow \mathcal{W}_2$ be an arbitrary function. Then $f \circ \A : \mathcal{Z}^n \rightarrow \mathcal{W}_2$ satisfies $(\lambda, \rho)$-RDP.     
\end{lemma}
 
 The following RDP composition theorem characterizes the privacy of a composition of parallel or adaptive mechanisms in terms of the privacy guarantees of the individual mechanisms.
 \begin{lemma}[Composition of RDP {\cite{mironov2017renyi}}]\label{lem:composition_RDP}
Fix an order $\lambda>1$. For each $i\in[k]$, let $\A_i:\mathcal Z^n\to\mathcal W_i$
be a randomized mechanism satisfying $(\lambda,\rho_i)$-RDP.
Then the following statements hold.
\begin{enumerate}[label=({\alph*}),leftmargin=*]
\item \textit{Joint (simultaneous) release.} 
Let $\A(S)=(\A_1(S),\ldots,\A_k(S))$.
Suppose $\{\A_i\}_{i=1}^k$ are independent.
Then $\A$ satisfies
$(\lambda,\sum_{i=1}^k\rho_i)$-RDP.

\item \textit{Adaptive composition.}
Suppose $\A_1,\ldots,\A_k$ are applied sequentially, and for each $i\in[k]$,
$\A_i$ may depend on the previous outputs
$\A_1(S),\ldots,\A_{i-1}(S)$.
If for every fixed realization $w_{<i}:= (\A_1(S),\ldots,\A_{i-1}(S))$ of previous outputs,
the conditional mechanism $\A_i(\cdot\,;w_{<i})$ satisfies $(\lambda,\rho_i)$-RDP,
then the overall mechanism
\[
\A(S)=\big(\A_1(S),\A_2(S;\A_1(S)),\ldots,\A_k(S;\A_1(S),\ldots,\A_{k-1}(S))\big)
\]
satisfies $(\lambda,\sum_{i=1}^k\rho_i)$-RDP.
\end{enumerate}
\end{lemma}

Based on the above lemmas, we now prove Theorem~\ref{thm:privacy}. 
\begin{proof}[Proof of Theorem \ref{thm:privacy}]
 Without loss of generality, we assume that the neighboring datasets $S$ and $S'$  differ in the first data point, i.e., $z_1$ and $z_1'$. We begin by estimating the $\ell_2$-sensitivity of $\partial_{\ba } \mathcal{L}_S(\widetilde{\Theta}(k))$ and $\partial_{\bc } \mathcal{L}_S(\widetilde{\Theta}(k))$ at each iteration $k$. %Note $\widetilde{\Theta}(k)=[\ba(k)^\top,\bc(k)^\top]^\top$, we compute the $\ell_2$-sensitivity of $\ba(k) $ and $\bc(k) $ at iteration $k$. 
 Note Lemma~\ref{cor:c} implies $\|\bc(0)\|_2 \le 4\sqrt{pm} +  2\sqrt{\log( {2}/{\delta})}$ with probability at least $1 - \delta/2$. In the following proof, we assume the event $\{\bc(0): \|\bc(0)\|_2 \le 4\sqrt{pm} +  2\sqrt{\log( {2}/{\delta})}\}$ holds and aim to show Algorithm~\ref{alg1} satisfies $(\epsilon, {\delta}/{2})$-DP.  
 
Noting that $S$ and $S'$  differ in the first data point. From the update rule of $\ba(k) $  we know 
\begin{align*}
     & \big\| \partial_{\ba } \mathcal{L}_S(\widetilde{\Theta}(k)) - \partial_{\ba } \mathcal{L}_{S'}(\widetilde{\Theta}(k)) \big\|_2\nonumber\\
     &=\frac{1}{n}\big\|\partial_{\ba }  \ell(\widetilde{\Theta}(k),z_1) -  \partial_{\ba }  \ell(\widetilde{\Theta}(k),z'_1) \big\|_2\le \frac{1}{n}\big(\big\|\partial_{\ba }  \ell(\widetilde{\Theta}(k),z_1) \big\|_2 + \big\|  \partial_{\ba}  \ell(\widetilde{\Theta}(k),z'_1) \big\|_2\big) \nonumber\\
     &\le \frac{2B'_{\ell}B'_\sigma B'_b B_b p  }{n \sqrt{m} }  \, \big(\|\bc(0)\|_2 +  \| \bc (0)-\bc \|_2\big)  
     \le  \frac{2B'_{\ell}B'_\sigma B'_b B_b  p }{n }\Big(4\sqrt{p}  + \frac{\log(2/\delta) + R_2}{\sqrt{m}}\Big),
\end{align*}
  where in the last second inequality we have used \eqref{eq:partial_a_norm_1} together with Assumption~\ref{ass:loss}, and the last inequality follows from  $\bc (k) \in \Omega_{\bc}= \mathcal{B}(\bc(0),R_2)$ and the condition for $m$.  

  Hence, the $\ell_2$-sensitivity of the gradient $\partial_{\ba } \mathcal{L}_S(\widetilde{\Theta}(k))$ is $\Delta_{\ba}=\frac{2B'_{\ell}B'_\sigma B'_b B_b  }{n }p\big(4\sqrt{p} + \frac{\log(2/\delta) + R_2}{\sqrt{m}}\big)$.
  Let $C_1 =  8(B'_{\ell}B'_\sigma B'_b B_b)^2 p^2\big(4\sqrt{p} + \frac{ \log(2/\delta) + R_2}{\sqrt{m}}\big)^2$.
  We set  
  \begin{align*}
      \sigma_1^2=\frac{2T (1+\frac{\log(2T/\delta)}{\epsilon}) \Delta_{\ba}^2 }{ \epsilon}=\frac{ C_1 T (1+\frac{\log(2T/\delta)}{\epsilon})     }{n^2 \epsilon}.
  \end{align*}
   Lemma~\ref{lem:gaussian-rdp} implies that the mechanism  $\partial_{\ba } \mathcal{L}_S(\widetilde{\Theta}(k)) + \bb_{1 }(k)$ satisfies  $(\lambda,\frac{ \epsilon}{4T })$-RDP with  $\lambda=1+\frac{2\log(2 /\delta)}{\epsilon}$.
   Applying Lemma \ref{lemma:post-processing} (post-processing property) we know $\ba (k+1)=\ba (k )-\eta \big(\partial_{\ba } \mathcal{L}_S(\widetilde{\Theta}(k)) + \bb_{1 }(k)\big)$ also satisfies $(\lambda,\frac{ \epsilon}{4T })$-RDP. %Furthermore, from part (a) in Lemma~\ref{lem:composition_RDP}, we know that $\ba(k+1)$ satisfies  $(\lambda,\frac{ \epsilon}{4T })$-RDP.

  Similarly, combining Assumption~\ref{ass:loss} with \eqref{eq:partial_c_norm}, we know
  \begin{align*}
      & \big\| \partial_{\bc } \mathcal{L}_S(\widetilde{\Theta}(k)) - \partial_{\bc } \mathcal{L}_{S'}(\widetilde{\Theta}(k)) \big\|_2 \le \frac{1}{n}\big(\big\|\partial_{\bc }  \ell(\widetilde{\Theta}(k),z_1) \big\|_2 + \big\|  \partial_{\bc }  \ell(\widetilde{\Theta}(k),z'_1) \big\|_2\big)   \le  \frac{2B'_\ell B_b\sqrt{p}}{n   }.
\end{align*}
Then, the $\ell_2$-sensitivity of  $\partial_{\bc } \mathcal{L}_S(\widetilde{\Theta}(k))$ is $\Delta_{\bc}= \frac{2B'_\ell B_b\sqrt{p}}{n }$. Let $C_{2} =8(B'_\ell B_b)^2p.$ Set 
  \[ \sigma_2^2=\frac{2T (1+\frac{\log(2T/\delta)}{\epsilon}) \Delta_{\bc}^2 }{ \epsilon}=\frac{ C_2T (1+\frac{\log(2T/\delta)}{\epsilon})   }{n^2 \epsilon}.  \]
Hence, by applying Lemmas \ref{lem:gaussian-rdp} and \ref{lemma:post-processing} again, we know both the mechanism $\partial_{\bc } \mathcal{L}_S(\Theta(k))+ \bb_{2 }(k)$ and $\bc(k+1)$ satisfy $(\lambda,\frac{ \epsilon}{4T })$-RDP.  

According to part (a) in Lemma~\ref{lem:composition_RDP}, for each $k=0,\ldots,T-1$ and condition on the fixed $(\widetilde{\Theta}(1),\ldots,\widetilde{\Theta}(k))$, we know $\widetilde{\Theta}(k+1)=[\ba(k+1),\bc(k+1)]$ satisfies  $(\lambda,\frac{ \epsilon}{2T })$-RDP.
Combining all iterations, from part (b) in Lemma~\ref{lem:composition_RDP}, we know $\widetilde{\Theta}(T)$ is  $(\lambda,\frac{ \epsilon}{2 })$-RDP where $\lambda=1+\frac{2\log(2 /\delta)}{\epsilon}$.  Finally, Lemma~\ref{lemma:RDP_to_DP} implies $\widetilde{\Theta}(T)$ satisfies $(\epsilon,\frac{\delta}{2})$-DP. Combining this observation with the event $\{\bc(0): \|\bc(0)\|_2 \le 4\sqrt{pm} +  2\sqrt{\log( {2}/{\delta})}\}$ holds with probability at least $1 - \delta/2$  completes the proof. 
\end{proof}

\subsection{Utility Guarantee}\label{subsec:appen-DP-utility}
Now, we present the proofs for Theorems~\ref{thm:utility}. 
Throughout this subsection, we condition on a fixed initialization satisfying~\eqref{eq:bound_c0}. 
Thus all estimates below are deterministic with respect to the initialization. 
The expectation \(\E_{\A}\) is taken only over the Gaussian perturbations generated after initialization, and \(\E_{S,\A}\) is taken over the sample and these Gaussian perturbations, conditional on the initialization. 
At the end, since~\eqref{eq:bound_c0} holds with probability at least \(1-\delta\) over the initialization, the resulting utility bounds hold with the same probability over the initialization.

We require the following two properties of the training loss to estimate the optimization error of DP-GD.
\begin{lemma}\label{lem:self-square}
    Let $\delta\in(0,1)$.
    Suppose \eqref{eq:bound_c0}, $m \gtrsim \log(m/\delta)$ and Assumptions~\ref{ass:sigma} and \ref{ass:loss} hold.
    Let $\bar{\rho} = C_{\sigma,b}\, p^2   \big(p  + R_2^2 \big) $ and $\bar{\kappa}=  C_{\sigma,b }\, p  \big( 1+ \frac{ \sqrt{p  } ( \sqrt{p} +R_2 )}{\sqrt{m}}\big)  $. 
    For any $S$ and any $ \widetilde{\Theta} = (\ba, \bc) $ with $\bc\in \Omega_\bc$, it holds that
    \begin{align*}
        \lambda_{\max}\big(\nabla^2  \L_S(\widetilde{\Theta})\big) \le  \bar{\rho}, \qquad \lambda_{\min}\big(\nabla^2  \L_S(\widetilde{\Theta})\big)\ge  -\frac{\bar{\kappa}}{\sqrt{m}} \, \L_S(\widetilde{\Theta}),
    \end{align*}
    and
    \begin{align*}
        \big\| \nabla \mathcal{L}_S(\widetilde{\Theta})\big\|_2^2 \le 2\bar{\rho} \mathcal{L}_S(\widetilde{\Theta}).
    \end{align*}
\end{lemma}
\begin{proof}
    Since $\bc\in\Omega_\bc$, we know $\max_{i\in[m]}\|\bc_i-\bc_i(0)\|_2 \le R_2$.
    Applying Proposition \ref{pro:smooth} with $\Theta = \widetilde{\Theta}$, the fact $\max_{i\in[m]}\|\bc_i-\bc_i(0)\|_2 \le R_2$ and the condition $m \gtrsim \log(m/\delta)$, we obtain the estimates $\lambda_{\max}(\nabla^2  \mathcal{L}_S(\widetilde{\Theta})) \le \bar{\rho}$ and $\lambda_{\min}(\nabla^2  \mathcal{L}_S(\widetilde{\Theta})) \ge -\bar{\kappa}\ell(yf_{\widetilde{\Theta}}(\bx))/\sqrt{m}$ immediately.

    Note that in the proof of Proposition \ref{pro:smooth}, $\bar{\rho}$ is an upper bound for the term $\sup_{\bx\in\X} \|\nabla f_{\Theta}(\bx)\|_2^2 + \|\nabla^2 f_{\Theta}(\bx)\|_2$.
    Then, it holds that
    \begin{align*}
        \big\| \nabla \mathcal{L}_S(\widetilde{\Theta})\big\|_2^2 &= \Big\|\frac{1}{n}\sum_{i=1}^n \ell'(y_if_{\widetilde{\Theta}}(\bx_i))y_i\nabla f_{\widetilde{\Theta}}(\bx_i)\Big\|_2^2 \le \Big|\frac{1}{n}\sum_{i=1}^n \ell'(y_if_{\widetilde{\Theta}}(\bx_i)) \Big|^2 \sup_{i\in[n]} \big\|\nabla f_{\Theta}(\bx_i)\big\|_2^2\\
        &\le \bar{\rho} B_\ell'\Big|\frac{1}{n}\sum_{i=1}^n \ell(y_if_{\widetilde{\Theta}}(\bx_i)) \Big| = \bar{\rho} B_\ell' \L_S(\widetilde{\Theta}) \le 2\bar{\rho}\L_S(\widetilde{\Theta}),
    \end{align*}
    where the second inequality used the self-boundedness property and $|\ell'(\cdot)| \le B_\ell'$ (see Assumption \ref{ass:loss}).
    This completes the proof of the lemma.
\end{proof}

\begin{lemma}\label{lem:dp-descent}
    Let $\delta\in(0,1)$.
    Suppose \eqref{eq:bound_c0}, Assumptions~\ref{ass:sigma} and \ref{ass:loss} hold.
    Let $\{\widetilde{\Theta}(k)\}$ be produced by Algorithm~\ref{alg1} based on $S$ with  $T$ iterations.
    Let $\bB(k)=(\bb_1(k),\bb_2(k))$,  where $\bb_1(k),\bb_2(k)$ are independent Gaussian noise vectors added at iteration $k$.
    For $k=0,\ldots,T-1$, it holds that
\[
\mathcal{L}_S(\widetilde{\Theta}(k+1))
\le
\mathcal{L}_S(\widetilde{\Theta}(k))
-\Big(\frac{1}{2\eta}-\frac{\bar{\rho}}{2}\Big)\|\widetilde{\Theta}(k)-\widetilde{\Theta}(k+1)\|_2^2
+\frac{\eta}{2}\|\bB(k)\|_2^2. 
\]
Furthermore, if we assume $\eta\le 1/\bar{\rho}$, then
\[
\mathcal{L}_S(\widetilde{\Theta}(k+1))
\le
\mathcal{L}_S(\widetilde{\Theta}(k))
+\frac{\eta}{2}\|\bB(k)\|_2^2. 
\]
\end{lemma}
\begin{proof}
Denote $\Omega = \Omega_\ba\times\Omega_\bc = \mathcal{B}(\ba(0), R_1)\times \mathcal{B}(\bc(0), R_2)$.
From the update rule of DP-GD, for any $k=0,\ldots,T-1$, we know
\[ \widetilde{\Theta}(k+1) = \proj_{\Omega}\big( \widetilde{\Theta}(k)  - \eta \big(\nabla  \mathcal{L}_S(\widetilde{\Theta}(k ))+ \bB(k) \big)\big).   \]
% \textcolor{blue}{Although the notation here is slightly ambiguous, we adopt it for notational convenience.
% In Algorithm~\ref{alg1}, we first perform projections on $\ba$ and $\bc$ respectively, and then let $\widetilde{\Theta}=[\ba^\top,\bc^\top]^\top$. }

Let $W(k)= \widetilde{\Theta}(k)  - \eta \big(\nabla  \mathcal{L}_S(\widetilde{\Theta}(k ))+ \bB(k) \big)$.
Note that $\Omega$ is a closed convex set.
By the projection theorem, we know
\[ \big\langle W(k) -\widetilde{\Theta}(k+1), \widetilde{\Theta}(k )-\widetilde{\Theta}(k+1)  \big\rangle \le 0.\]
Substituting the definition of $W(k)$ and dividing $\eta$, we have
\begin{align}\label{eq:W}
    \frac{1}{\eta}\big\|\widetilde{\Theta}(k )-\widetilde{\Theta}(k+1)   \big\|_2^2 \le  \big\langle \nabla  \mathcal{L}_S(\widetilde{\Theta}(k ))+ \bB(k),\widetilde{\Theta}(k )-\widetilde{\Theta}(k+1)  \big\rangle.
\end{align}

On the other hand, according to Lemma~\ref{lem:self-square}, we know that  $ \lambda_{\max}\big(\nabla^2 \mathcal{L}_S( \widetilde{\Theta}(k) )\big) \le  \bar{\rho} $ for all $k\in[T]$ with probability at least $1-\delta$. It then follows
\begin{align*}
  \mathcal{L}_S(\widetilde{\Theta}(k+1))& \le\mathcal{L}_S(\widetilde{\Theta}(k))+ \big\langle \nabla \mathcal{L}_S(\widetilde{\Theta}(k)), \widetilde{\Theta}(k+1)-\widetilde{\Theta}(k)\big\rangle + \frac{\bar{\rho}}{2} \big\| \widetilde{\Theta}(k+1)-\widetilde{\Theta}(k) \big\|_2^2 \nonumber\\
  &\le \mathcal{L}_S(\widetilde{\Theta}(k))- \big(  \frac{1}{\eta}- \frac{\bar{\rho}}{2}\big)\big\|\widetilde{\Theta}(k )-\widetilde{\Theta}(k+1)   \big\|_2^2+ \big\langle    \bB(k),\widetilde{\Theta}(k )-\widetilde{\Theta}(k+1)  \big\rangle \nonumber\\
  &\le \mathcal{L}_S(\widetilde{\Theta}(k))- \big(  \frac{1}{2\eta}- \frac{\bar{\rho}}{2}\big)\big\|\widetilde{\Theta}(k )-\widetilde{\Theta}(k+1)   \big\|_2^2+ \frac{\eta}{2}\big\|\bB(k)\big\|_2^2,
\end{align*}
where in the second inequality we have used \eqref{eq:W} and the last inequality is due to $ab\le \frac{\eta a^2}{2} + \frac{b^2}{2\eta }$. 
The proof is complete.   
\end{proof}

Our proofs also need the following mirror descent inequality.  
 \begin{lemma}\label{lem:MD}
Let $\Omega\subseteq \R^{(m+1)dp}$ be a nonempty closed convex set, and let $\proj_\Omega(\cdot)$ denote the Euclidean projection onto $\Omega$.
Fix any $\eta>0$ and any vector $g\in\R^{(m+1)dp}$. Define
\begin{equation*} 
    \Theta^{+} = \proj_\Omega\!\big(\Theta - \eta g\big).
\end{equation*}
Then, for any comparator $\Theta^*\in \Omega$, the following inequality holds
\begin{equation*} 
    \big\langle g,\ \Theta-\Theta^*\big\rangle
     \le 
    \frac{1}{2\eta}\Big(\|\Theta-\Theta^*\|_2^2-\|\Theta^{+}-\Theta^*\|_2^2\Big)
    +\frac{\eta}{2}\|g\|_2^2.
\end{equation*}
\end{lemma}

\begin{proof}
% For completeness, we present a proof here.
Let $z:=\Theta-\eta g$ such that $\Theta^{+}=\proj_\Omega(z)$. From the projection theorem we know for all $u\in \Omega$,
$\big\langle z-\Theta^{+},\ u-\Theta^{+}\big\rangle \le 0.$
Choosing $u=\Theta^*\in \Omega$ and substituting $z=\Theta-\eta g$ yield
 \begin{equation}\label{eq:proj_opt3}
    \eta\big\langle g,\ \Theta^*-\Theta^{+}\big\rangle
     \ge 
    \big\langle \Theta-\Theta^{+},\ \Theta^*-\Theta^{+}\big\rangle .
\end{equation}

Note that for any vectors $\ba,\bb,\bc\in\R^{(m+1)dp}$, $2\langle \ba-\bb,\ \bc-\bb\rangle = \|\bb-\bc\|_2^2 + \|\ba-\bb\|_2^2 - \|\ba-\bc\|_2^2.$
Applying this  with $(\ba,\bb,\bc)=(\Theta,\Theta^{+},\Theta^*)$ gives
\begin{equation*} 
    \big\langle \Theta-\Theta^{+},\ \Theta^*-\Theta^{+}\big\rangle
    = 
    \frac12\Big(\|\Theta^{+}-\Theta^*\|_2^2 + \|\Theta-\Theta^{+}\|_2^2 - \|\Theta-\Theta^*\|_2^2\Big).
\end{equation*}
Combining \eqref{eq:proj_opt3} and the above equality and dividing by $\eta$ yield 
\begin{equation*} 
    \big\langle g,\ \Theta^*-\Theta^{+}\big\rangle
    \ge 
    \frac{1}{2\eta}\Big(\|\Theta^{+}-\Theta^*\|_2^2 + \|\Theta-\Theta^{+}\|_2^2 - \|\Theta-\Theta^*\|_2^2\Big).
\end{equation*}
Therefore, 
\begin{align*} 
    \big\langle g,\ \Theta-\Theta^*\big\rangle
    &=
    \big\langle g,\ \Theta-\Theta^{+}\big\rangle
    -
    \big\langle g,\ \Theta^* - \Theta^{+} \big\rangle 
    \le
    \big\langle g,\ \Theta-\Theta^{+}\big\rangle
    +
    \frac{1}{2\eta}\Big(\|\Theta-\Theta^*\|_2^2-\|\Theta^{+}-\Theta^*\|_2^2-\|\Theta-\Theta^{+}\|_2^2\Big)\\
    &\le   \frac{\eta}{2}\|g\|_2^2+\frac{1}{2\eta}\|\Theta-\Theta^{+}\|_2^2 +   
    \frac{1}{2\eta}\Big(\|\Theta-\Theta^*\|_2^2-\|\Theta^{+}-\Theta^*\|_2^2-\|\Theta-\Theta^{+}\|_2^2\Big)\\
    &  \le
    \frac{1}{2\eta}\Big(\|\Theta-\Theta^*\|_2^2-\|\Theta^{+}-\Theta^*\|_2^2\Big)
    +\frac{\eta}{2}\|g\|_2^2,
\end{align*}
where we have used Cauchy-Schwarz inequality and the basic inequality $2ab\le a^2+b^2$ for all $a,b\in\R$.
This completes the proof of the lemma.
\end{proof}

Now we give the optimization risk bounds. For convenience, we restate it below. Let $R=R_1+R_2$.  
\begin{theorem}\label{thm:opt-DP}
     Suppose \eqref{eq:bound_c0}, Assumptions~\ref{ass:sigma} and \ref{ass:loss} hold. Let $\{\widetilde{\Theta}(k)\}$ be produced by Algorithm~\ref{alg1} with $\eta \le \min\{1/3\bar{\rho},1\}$ and $T$ iterations. For any reference point $\Theta^*=( \ba^* , \bc^*)$ with $\ba^* \in \Omega_{\ba}$ and $\bc^* \in \Omega_{\bc}$, assume  $m\gtrsim p^2 (\log(m/\delta)+R^2)(R^4+ \| \Theta^* -  \widetilde{\Theta}(0) \|_2^4)$. 
     Then, it holds that
   \begin{align*} 
          \frac{1}{T}\sum_{k=1}^{T}\E_{\A} \big[\mathcal{L}_S(\widetilde{\Theta}(k ))\big] \le 4\mathcal{L}_S(\Theta^*) + \frac{3}{2 \eta T}  \|\widetilde{\Theta}(0 )-\Theta^*\|_2^2  +     \frac{ mp^4\eta T  d   \log(2T/\delta) }{n^2\epsilon^2} +   2\mathcal{L}_S(\widetilde{\Theta}(0)).
    \end{align*}
    %Here, the expectation is taken with respect to the randomness of the algorithm, i.e., the randomness of the Gaussian noises. 
    Here, the expectation is taken only over the Gaussian perturbations generated after the fixed initialization.
    If we further assume  $ \widetilde{\Lambda}_{\Theta^*}^2 \ge  \eta \mathcal{L}_S(\widetilde{\Theta}(0))$, then
    \begin{align*} 
         \frac{1}{T}\sum_{k=1}^{T}\E_{\A} \big[\mathcal{L}_S(\widetilde{\Theta}(k ))\big] \le   4\mathcal{L}_S(\Theta^*) + \frac{4}{ \eta T}  \|\widetilde{\Theta}(0 )-\Theta^*\|_2^2  +     \frac{ mp^4\eta T  d   \log(2T/\delta) }{n^2\epsilon^2}.
    \end{align*}
\end{theorem}
\begin{proof}
    Note that $\|\bc(k)\|_2\in\Omega_\bc$ for all iterations $k$.
    Then, Lemma~\ref{lem:self-square} shows that $   \lambda_{\min}\big(\nabla^2  \ell(y f_{\widetilde{\Theta}(k)}(\bx))\big) \ge -\frac{\bar{\kappa}}{\sqrt{m}} \ell(y f_{\widetilde{\Theta}(k)}(\bx)).$
    Similar to \eqref{eq:L_S}, for any fixed $\Theta^*$, according to Taylor expansion and the estimate for $\lambda_{\min}\big(\nabla^2  \ell(y f_{\widetilde{\Theta}(k)}(\bx))\big)$, we have
    \begin{align}\label{eq:dp-L}
       \mathcal{L}_S(\widetilde{\Theta}(k ))  &\le \mathcal{L}_S(\Theta^*) -     \big\langle \nabla  \mathcal{L}_S(\widetilde{\Theta}(k )), \Theta^* -   \widetilde{\Theta}(k ) \big\rangle + \frac{\bar{\kappa}}{2\sqrt{m}}\max_{\alpha\in[0,1]}  \mathcal{L}_S(\widetilde{\Theta}_{\alpha,k })  \big\| \Theta^* -   \widetilde{\Theta}(k )\big\|_2^2. 
   \end{align} 
     where $\widetilde{\Theta}_{\alpha,k }=\alpha\widetilde{\Theta}(k ) + (1-\alpha)\Theta^*$.
    Note that $\rho\le 1/(2\bar{\rho})$.
    From Lemma~\ref{lem:dp-descent} we know that
\begin{align*}
    \mathcal{L}_S(\widetilde{\Theta}(k+1)) \le \mathcal{L}_S(\widetilde{\Theta}(k)) + \frac{\eta}{2}\big\|\bB(k)\big\|_2^2,
\end{align*}
where $\bB(k)=(\bb_1(k),\bb_2(k))$ with $\bb_1(k),\bb_2(k)$ being the independent Gaussian noise vectors added at iteration $k$.
Plugging \eqref{eq:dp-L} into the above  inequality yields
    \begin{align}\label{eq:dp-k+1} 
         \mathcal{L}_S(\widetilde{\Theta}(k+1)) &\le \mathcal{L}_S(\Theta^*) -     \big\langle \nabla  \mathcal{L}_S(\widetilde{\Theta}(k )), \Theta^* -   \widetilde{\Theta}(k ) \big\rangle  + \frac{\eta}{2}\big\|\bB(k)\big\|_2^2   + \frac{\bar{\kappa}}{2\sqrt{m}} \max_{\alpha\in[0,1]}  \mathcal{L}_S(\widetilde{\Theta}_{\alpha,k })  \big\| \Theta^* -   \widetilde{\Theta}(k )\big\|_2^2.
    \end{align}
On the other hand, Lemma~\ref{lem:MD} with $g=\nabla  \mathcal{L}_S(\widetilde{\Theta}(k ))+\bB(k)$, $\Theta=\widetilde{\Theta}(k )$ and $\Theta^+=\widetilde{\Theta}(k +1)$ implies
\[   \big\langle \nabla  \mathcal{L}_S(\widetilde{\Theta}(k ))+\bB(k),\widetilde{\Theta}(k )-\Theta^*\big\rangle
     \le 
    \frac{1}{2\eta}\Big(\|\widetilde{\Theta}(k )-\Theta^*\|_2^2-\|\widetilde{\Theta}(k +1)-\Theta^*\|_2^2\Big)
    +\frac{\eta}{2}\|\nabla  \mathcal{L}_S(\widetilde{\Theta}(k ))+\bB(k)\|_2^2.\]
It then follows that
\[ - \big\langle \nabla  \mathcal{L}_S(\widetilde{\Theta}(k )),\Theta^*-\widetilde{\Theta}(k )\big\rangle
     \le 
    \frac{1}{2\eta}\Big(\|\widetilde{\Theta}(k )-\Theta^*\|_2^2-\|\widetilde{\Theta}(k +1)-\Theta^*\|_2^2\Big)
    +\frac{\eta}{2}\|\nabla  \mathcal{L}_S(\widetilde{\Theta}(k ))+\bB(k)\|_2^2 +  \big\langle \bB(k),\Theta^*-\widetilde{\Theta}(k )\big\rangle.  \]
Putting the above inequality into \eqref{eq:dp-k+1}, we know
      \begin{align} \label{eq:dp-L_S}
         \mathcal{L}_S(\widetilde{\Theta}(k+1)) 
        %&\le \mathcal{L}_S(\Theta^*) + \frac{1}{2\eta}\Big(\|\widetilde{\Theta}(k )-\Theta^*\|_2^2-\|\widetilde{\Theta}(k +1)-\Theta^*\|_2^2\Big)  +\frac{\eta}{2}\|\nabla  \mathcal{L}_S(\widetilde{\Theta}(k ))+\bB(k)\|_2^2 +  \big\langle \bB(k),\Theta^*-\widetilde{\Theta}(k )\big\rangle \nonumber\\&\quad  + \frac{\eta}{2}\big\|\bB(k)\big\|_2^2 + \frac{c_{\min}}{2\sqrt{m}} \max_{\alpha\in[0,1]}  \mathcal{L}_S(\widetilde{\Theta}_{\alpha,k })  \big\| \Theta^* -   \widetilde{\Theta}(k )\big\|_2^2\nonumber\\
   \le \mathcal{L}_S(\Theta^*)  & + \frac{1}{2\eta}\Big(\|\widetilde{\Theta}(k )-\Theta^*\|_2^2-\|\widetilde{\Theta}(k +1)-\Theta^*\|_2^2\Big)
    +\frac{\eta}{2}\|\nabla  \mathcal{L}_S(\widetilde{\Theta}(k ))+\bB(k)\|_2^2\nonumber\\&\quad  +  \big\langle \bB(k),\Theta^*-\widetilde{\Theta}(k )\big\rangle   + \frac{\eta}{2}\big\|\bB(k)\big\|_2^2 + \frac{1}{4} \max_{\alpha\in[0,1]}  \mathcal{L}_S(\widetilde{\Theta}_{\alpha,k }) , 
    \end{align} 
  where the inequality follows from $ \frac{\bar{\kappa}}{2\sqrt{m}} \big\| \Theta^* -   \widetilde{\Theta}(k )\big\|_2^2 \le \frac{1}{4} $, which is ensured by the fact $ \| \Theta^* -   \widetilde{\Theta}(k ) \|_2 \le   \| \Theta^* -  \widetilde{\Theta}(0) \|_2  +  \| \widetilde{\Theta}(0) - \widetilde{\Theta}(k ) \|_2 \le    \| \Theta^* -  {\Theta}(0) \|_2 + R $ and the condition $m\gtrsim p^2  (\log(m/\delta)+R^2)( R^4+ \| \Theta^* -  \widetilde{\Theta}(0) \|_2^4)$. 

Applying Lemma~\ref{lem:quasi-convexity} with $ D= \big\| \Theta^* -  \widetilde{\Theta}(k)\big\|_2$ and $\kappa=\frac{\bar{\kappa}}{ \sqrt{m}}$, and $\tau = (1 - \kappa D/2)^{-1}$, we know
\[\max_{\alpha\in[0,1]  }  \mathcal{L}_S(\widetilde{\Theta}_{\alpha,k})  \le \frac{4}{3}  \max\big\{\mathcal{L}_S(\widetilde{\Theta}(k)), \mathcal{L}_S(\Theta^*)  \big\}\le  \frac{4}{3}  \big(\mathcal{L}_S(\widetilde{\Theta}(k))+ \mathcal{L}_S(\Theta^*) \big).   \]
Plugging the above inequality back into \eqref{eq:dp-L_S} and taking expectation over the randomness of $\bB(k)$, it holds that
      \begin{align} 
         \E_{\bB(k)} \big[\mathcal{L}_S(\widetilde{\Theta}(k+1))\big]  
        &\le \mathcal{L}_S(\Theta^*) + \frac{1}{2\eta}\big(\|\widetilde{\Theta}(k )-\Theta^*\|_2^2-\E_{\bB(k)}  \big[\|\widetilde{\Theta}(k +1)-\Theta^*\|_2^2\big]\big)
    +\frac{\eta}{2} \|\nabla  \mathcal{L}_S(\widetilde{\Theta}(k ))\|_2^2  \nonumber\\&\quad +   \frac{3 mp\eta T (1+\frac{\log(2T/\delta)}{\epsilon})}{2n^2\epsilon}\big( C_1d +C_2  \big)    + \frac{1}{3 } \big(  \mathcal{L}_S(\widetilde{\Theta}(k)) + \mathcal{L}_S(\Theta^*) \big)\nonumber\\
    &\le \mathcal{L}_S(\Theta^*) + \frac{1}{2\eta}\big(\|\widetilde{\Theta}(k )-\Theta^*\|_2^2-\E_{\bB(k)}  \big[\|\widetilde{\Theta}(k +1)-\Theta^*\|_2^2\big]\big)
    + \eta \bar{\rho}\mathcal{L}_S(\widetilde{\Theta}(k ))  \nonumber\\&\quad  + \frac{ 3 mp\eta  T (1+\frac{\log(2T/\delta)}{\epsilon})}{2n^2\epsilon}\big( C_1d +C_2  \big)   + \frac{1}{3 }  \big( \mathcal{L}_S(\widetilde{\Theta}(k))+ \mathcal{L}_S(\Theta^*) \big),
    \end{align}
    where in the first inequality we have used $\E_{\bB(k)}\big[\big\langle \bB(k),\Theta^*-\widetilde{\Theta}(k )\big\rangle\big]=\E_{\bB(k)}\big[\big\langle \bB(k),\nabla  \mathcal{L}_S(\widetilde{\Theta}(k ))\big\rangle\big] = 0 $ by noting that $\E[\bB(k)]=0$ and $\Theta^*, \widetilde{\Theta}(k )$ and $S$ are independent of $\bB(k)$, and  $\E[\|\bb\|_2^2]=\sigma^2 d$ if $\bb \sim \N(0, \sigma^2\mathbf{I}_{d})$,
    and the second inequality is due to  $ \| \nabla \mathcal{L}_S(\widetilde{\Theta}(k)) \|_2^2 \le 2\bar{\rho} \mathcal{L}_S(\widetilde{\Theta}(k))$  implied by Lemma~\ref{lem:self-square}. 
   % \[   \|\nabla  \mathcal{L}_S(\widetilde{\Theta}(k ))\|_2\le \sqrt{C_{\sigma,b} d}pR  \mathcal{L}_S(\widetilde{\Theta}(k )) \text{ and }  \|\nabla  \mathcal{L}_S(\widetilde{\Theta}(k ))\|_2\le |\ell'| \|\nabla f_{\widetilde{\Theta}(k)}\|_2\le  \sqrt{C_{\sigma,b} d}pR  \] Here, the first step is follows from  the self-boundness assumption and \eqref{eq:dp-bound-gradient}, and in the second inequality we have used $|\ell'(u)|\le B'_\ell$ and \eqref{eq:dp-bound-gradient}. 

Taking the expectation over all $\{\bB(t)\}_{t=1}^{T-1}$ on both sides of the above inequality and applying it recursively, we have
   \begin{align*} 
          \sum_{k=1}^{T}\!\E_{\A } \big[\mathcal{L}_S(\widetilde{\Theta}(k ))\big] \!\le & \frac{4T}{3} \mathcal{L}_S(\Theta^*) \!+\! \frac{1}{2\eta}  \|\widetilde{\Theta}(0 )\!-\!\Theta^*\|_2^2 \!+\!
     \big(\eta \bar{\rho}\!+\!\frac{1}{3 }\big)\sum_{k=0}^{T-1}\! \E_{\A } \big[\mathcal{L}_S(\widetilde{\Theta}(k )) \big] \!+\!  \frac{3mp \eta T^2  (1\!+\!\frac{\log(2T/\delta)}{\epsilon})}{2n^2\epsilon}\big( C_1d \!+\!C_2  \big).
    \end{align*}
Noting that $\eta \le \frac{1}{3\bar{\rho}} $ implies $\eta \bar{\rho} + \frac{1}{3} \le \frac{2}{3}$. It then follows that
  \begin{align}\label{eq:dp-opt}
         \sum_{k=1}^{T}\E_{\A} \big[\mathcal{L}_S(\widetilde{\Theta}(k ))\big] \le  4T \mathcal{L}_S(\Theta^*) + \frac{3}{2\eta}  \|\widetilde{\Theta}(0 )-\Theta^*\|_2^2  +  \frac{9mp\eta T^2 (1+\frac{\log(2T/\delta)}{\epsilon})}{2n^2\epsilon}\big( C_1d +C_2  \big)   + 2\mathcal{L}_S(\widetilde{\Theta}(0)) .
    \end{align}
Therefore,
\begin{align*} 
         \frac{1}{T}\sum_{k=1}^{T}\E_{\A} \big[\mathcal{L}_S(\widetilde{\Theta}(k ))\big] \le 4\mathcal{L}_S(\Theta^*) + \frac{3}{2 \eta T}  \|\widetilde{\Theta}(0 )-\Theta^*\|_2^2  +     \frac{ mp^4\eta T  d   \log(2T/\delta) }{n^2\epsilon^2} +   \frac{2}{T}\mathcal{L}_S(\widetilde{\Theta}(0)).
    \end{align*}
By further using $ \widetilde{\Lambda}_{\Theta^*}^2 \ge  \eta \mathcal{L}_S(\widetilde{\Theta}(0))$, it holds
\begin{align*} 
         \frac{1}{T}\sum_{k=1}^{T}\E_{\A} \big[\mathcal{L}_S(\widetilde{\Theta}(k ))\big] \le 4\mathcal{L}_S(\Theta^*) + \frac{4}{ \eta T}  \|\widetilde{\Theta}(0 )-\Theta^*\|_2^2  +     \frac{ mp^4\eta T  d   \log(2T/\delta) }{n^2\epsilon^2} .
    \end{align*}
    The proof is completed. 
\end{proof}

Now, we turn to prove our generalization bounds for DP-GD. 
Recall that we define $S=\{z_1,\ldots,z_n\}$ and $\widetilde{S}=\{z'_1,\ldots,z'_n\}$ be drawn independently from $\mathcal{P}$. For any $i\in[n]$, define $S^{ i}=\{ z_1,\ldots,z_{i-1},z_i', z_{i+1},\ldots,z_n\}$.
We need high probability version of optimization risk bound.
\begin{lemma}[High-probability optimization bound for DP-GD]
\label{lem:dp-opt-hp}
Suppose \eqref{eq:bound_c0} and Assumptions~\ref{ass:sigma} and~\ref{ass:loss} hold.
Let \(\{\widetilde\Theta(k)\}_{k=0}^T\) be produced by Algorithm~\ref{alg1}.
Assume the conditions of Theorem~\ref{thm:opt-DP} hold and $\|\widetilde\Theta(0)-\Theta^*\|_2^2
\ge C
\eta\L_S(\widetilde\Theta(0))$ with $C>0$ a constant.
Then, conditional on \eqref{eq:bound_c0}, with probability at least \(1-\delta\) over the Gaussian perturbations,
\begin{align*}
\frac1T\sum_{k=1}^T \L_S(\widetilde\Theta(k))
\lesssim 
\L_S(\Theta^*)
+
\frac{4\|\widetilde\Theta(0)-\Theta^*\|_2^2}{\eta T}   +
\frac{mp^4\eta T d\log(T/\delta)}{n^2\epsilon^2}
+
\frac{R p^{3/2}\log(T/\delta)}{n\epsilon}
.
\end{align*}
\end{lemma}
\begin{proof}
The proof follows the proof of Theorem~\ref{thm:opt-DP}, except that we do not take expectation over the Gaussian perturbations.

After summing the one-step inequality, the only random terms are
\[
\sum_{k=0}^{T-1}\|\bB(k)\|_2^2
\quad\text{and}\quad
\sum_{k=0}^{T-1}
\langle \bB(k),\Theta^*-\widetilde\Theta(k)\rangle .
\]
The first term is controlled by a standard chi-square concentration bound for Gaussian vectors:
\[
\sum_{k=0}^{T-1}\|\bB(k)\|_2^2
\lesssim
\frac{mp^4 d T^2\log(T/\delta )}{n^2\epsilon^2}
\]
with probability at least \(1-\delta /2\).

For the second term, since \(\widetilde\Theta(k)\) is measurable with respect to the past randomness and \(\bB(k)\) is an independent centered Gaussian vector, the sequence $\langle \bB(k),\Theta^*-\widetilde\Theta(k)\rangle$ is a martingale difference sequence. Moreover, on the projected set,
$\|\Theta^*-\widetilde\Theta(k)\|_2\le 2R$.
Therefore, by conditional Gaussian concentration, with probability at least \(1-\delta /2\),
\[
\sum_{k=0}^{T-1}
\langle \bB(k),\Theta^*-\widetilde\Theta(k)\rangle
\lesssim
\frac{R p^{3/2}T\log(T/\delta )}{n\epsilon}.
\]
Substituting these two bounds into the summed optimization recursion gives the claim.
\end{proof}
The following lemma provides on-average argument stability bounds for DP-GD algorithm. 
\begin{lemma}[DP-GD stability]\label{lem:dp-stability}
    Let $\delta\in(0,1)$.
    Suppose \eqref{eq:bound_c0}, Assumptions~\ref{ass:sigma} and~\ref{ass:loss} hold. 
    For any $i\in[n]$, let $\{\widetilde{\Theta}(k)\}_{k=0}^T$ and 
    $\{\widetilde{\Theta}^{i}(k)\}_{k=0}^T$ be produced by Algorithm~\ref{alg1} with 
    $\eta \le \min\{1/3\bar{\rho},1\}$ and $T$ iterations based on $S$ and $S^{i}$, respectively. For any reference point $\Theta^*=(\ba^*,\bc^*)$ with $\ba^*\in\Omega_{\ba}$ and $\bc^*\in\Omega_{\bc}$, assume
\[
m\gtrsim
p^2\big(\log (\tfrac{m}{\delta})+R^2\big)
\cdot
\max\Big\{
R^4+\|\Theta^*-\widetilde{\Theta}(0)\|_2^4,\,
\big(
\eta T(\mathcal L_S(\Theta^*)+\mathcal L_{\widetilde S}(\Theta^*))
+\eta(\mathcal L_S(\widetilde{\Theta}(0))
+\mathcal L_{\widetilde S}(\widetilde{\Theta}(0)))
\big)^2,\]
\[\big(
\frac{Rp^{3/2}\eta T\log(T/\delta)}{n\epsilon}
\big)^2
\Big\},
\]
and
\[
m\lesssim
\frac{(n\epsilon)^4}
{p^{10}d^2(\log(m/\delta)+R^2)(\eta T)^4\log^2(T/\delta)}.
\]
    Then, conditional on~\eqref{eq:bound_c0}, for any $k=0,\ldots,T-1$, it holds that
   \begin{align*} 
    \E_{S,\widetilde S,\A}\Big[
    \frac{1}{n}\sum_{i=1}^n  
    \big\| \widetilde{\Theta}(k+1)-\widetilde{\Theta}^{i}(k+1)\big\|_2
    \Big]
    \le
    \frac{2e\eta C_{\sigma,b}p(\sqrt p+R)}{n}
    \sum_{t=0}^k
    \E_{S,\A}\big[\mathcal{L}_S(\widetilde{\Theta}(t))\big]
    +
    2R\delta .
\end{align*}
\end{lemma}
\begin{proof}
All expectations below are taken conditional on the fixed initialization satisfying~\eqref{eq:bound_c0}. 
For each \(i\in[n]\), the two trajectories 
\(\{\widetilde{\Theta}(k)\}_{k=0}^T\) and 
\(\{\widetilde{\Theta}^i(k)\}_{k=0}^T\) are coupled with the same Gaussian perturbations. 
Thus the noise terms cancel in the stability recursion.

Recall that \(S^{-i}=S\setminus\{z_i\}\). Since the projection is non-expansive, using Lemma~\ref{lem:exansiveness}, the self-bounding property of \(\ell\), and Lemma~\ref{lem:hessian}, we obtain
\begin{align}\label{eq:dp-stab-recursion}
      & \big\|  \widetilde{\Theta}(k+1) - \widetilde{\Theta}^{ i}(k+1)\big\|_2 \le \big\|  \widetilde{\Theta}(k) - \eta \nabla \mathcal{L}_S(\widetilde{\Theta}(k))- \eta \bB(k) - \widetilde{\Theta}^{ i}(k) + \eta \nabla \mathcal{L}_{S^{ i}}(\widetilde{\Theta}^{ i}(k))+ \eta \bB(k)\big\|_2\nonumber\\
        &\le \Big\|  \widetilde{\Theta}(k) - \frac{\eta (n-1)}{n} \nabla \mathcal{L}_{S^{-i}}(\widetilde{\Theta}(k))  - \widetilde{\Theta}^{ i}(k) + \frac{\eta (n-1)}{n} \nabla \mathcal{L}_{S^{-i}}(\widetilde{\Theta}^{ i}(k))\big\|_2 + \frac{\eta}{n}\big\| \nabla \ell\big(y_i f_{\widetilde{\Theta}(k)}(\bx_i)\big)-\nabla \ell\big(y'_i f_{\widetilde{\Theta}^i(k)}(\bx'_i)\big)\Big\|_2\nonumber\\
        &\le G^i_\alpha(k) \big\|\widetilde{\Theta}(k) -  \widetilde{\Theta}^{i}(k)  \big\|_2  + \frac{\eta}{n}\big[  \ell\big(y_i f_{\widetilde{\Theta}(k)}(\bx_i)\big)\|\nabla f_{\widetilde{\Theta}(k)}(\bx_i)\|_2 +  \ell\big(y'_i f_{\widetilde{\Theta}^i(k)}(\bx'_i)\big)\|\nabla f_{\widetilde{\Theta}^i(k)}(\bx_i)\|_2\big]\nonumber\\
        &\le G^i_\alpha(k) \big\|\widetilde{\Theta}(k) -  \widetilde{\Theta}^{i}(k)  \big\|_2  + C_{\sigma, b}p(\sqrt{p} + R)\frac{\eta}{n}\big[  \ell\big(y_i f_{\widetilde{\Theta}(k)}(\bx_i)\big) +  \ell\big(y'_i f_{\widetilde{\Theta}^i(k)}(\bx'_i)\big)\big],
\end{align}
where 
$G_\alpha^i(k)\le 1+M^i(k)$ with $ M^i(k) = \frac{2\eta\bar{\kappa}}{\sqrt m} \max\big\{ \mathcal L_{S^{-i}}(\widetilde{\Theta}(k)), \mathcal L_{S^{-i}}(\widetilde{\Theta}^{i}(k)) \big\}.$
Here \(\bar\kappa\) is the curvature constant in Lemma~\ref{lem:self-square}.

By Lemma~\ref{lem:dp-opt-hp} applied to the two datasets \(S\) and \(S^i\),
with failure probability \(\delta/2\) for each, we obtain that with probability at least \(1-\delta\),
\begin{align*}
\eta\sum_{t=0}^{T-1}\mathcal L_S(\widetilde{\Theta}(t))
&\lesssim
\|\widetilde{\Theta}(0)-\Theta^*\|_2^2
+
\eta T\mathcal L_S(\Theta^*)
+
\eta\mathcal L_S(\widetilde\Theta(0))  +
\frac{mp^4d\,\eta^2T^2\log(T/\delta)}{n^2\epsilon^2}
+
\frac{Rp^{3/2}\eta T\log(T/\delta)}{n\epsilon},\\
\eta\sum_{t=0}^{T-1}\mathcal L_{S^i}(\widetilde{\Theta}^{i}(t))
&\lesssim
\|\widetilde{\Theta}(0)-\Theta^*\|_2^2
+
\eta T\mathcal L_{S^i}(\Theta^*)
+
\eta\mathcal L_{S^i}(\widetilde\Theta(0))   +
\frac{mp^4d\,\eta^2T^2\log(T/\delta)}{n^2\epsilon^2}
+
\frac{Rp^{3/2}\eta T\log(T/\delta)}{n\epsilon}.
\end{align*}
Denote this high-probability event by \(\mathcal E_{\rm opt}^{(i)}\). 
On \(\mathcal E_{\rm opt}^{(i)}\), we use $\mathcal L_{S^{-i}}(\widetilde{\Theta})
\le
\frac{n}{n-1}\mathcal L_S(\widetilde{\Theta})
\le
2\mathcal L_S(\widetilde{\Theta})$, 
and similarly $\mathcal L_{S^{-i}}(\widetilde{\Theta}^{i})
\le
2\mathcal L_{S^i}(\widetilde{\Theta}^{i})$. 
Therefore,
\begin{align*}
\sum_{s=0}^{T-1}M^i(s)
&\le
\frac{4\eta\bar\kappa}{\sqrt m}
\max\Big\{
\sum_{s=0}^{T-1}\mathcal L_S(\widetilde{\Theta}(s)),
\sum_{s=0}^{T-1}\mathcal L_{S^i}(\widetilde{\Theta}^{i}(s))
\Big\} \\
&\lesssim
\frac{p}{\sqrt m}
\Big(
\|\widetilde{\Theta}(0)-\Theta^*\|_2^2
+
\frac{mp^4d\,\eta^2T^2\log(T/\delta)}{n^2\epsilon^2}
+
\frac{Rp^{3/2}\eta T\log(T/\delta)}{n\epsilon}
\Big) \\
&\le 1,
\end{align*}
where the last inequality follows from the  condition of $m$. 

Hence, on \(\mathcal E_{\rm opt}^{(i)}\), it holds
\[
\prod_{s=t+1}^{k}(1+M^i(s))
\le
\exp\Big(\sum_{s=t+1}^{k}M^i(s)\Big)
\le e .
\]

Unrolling the recursion~\eqref{eq:dp-stab-recursion} gives, on \(\mathcal E_{\rm opt}^{(i)}\),
\begin{align*}
\big\|\widetilde{\Theta}(k+1)-\widetilde{\Theta}^{i}(k+1)\big\|_2
\le
\frac{e\eta C_{\sigma,b}p(\sqrt p+R)}{n}
\sum_{t=0}^{k}
\Big[
\ell\big(y_i f_{\widetilde{\Theta}(t)}(\bx_i)\big)
+
\ell\big(y_i' f_{\widetilde{\Theta}^{i}(t)}(\bx_i')\big)
\Big].
\end{align*}

On the complement \((\mathcal E_{\rm opt}^{(i)})^c\), the projection gives the crude bound $\big\|\widetilde{\Theta}(k+1)-\widetilde{\Theta}^{i}(k+1)\big\|_2
\le 2R$.
Since \(\mathbb P_{\A}((\mathcal E_{\rm opt}^{(i)})^c)\le\delta\), taking expectation over the Gaussian perturbations yields
\begin{align*}
\E_{\A}\Big[
\big\|\widetilde{\Theta}(k+1)-\widetilde{\Theta}^{i}(k+1)\big\|_2
\Big]
&\le
\frac{e\eta C_{\sigma,b}p(\sqrt p+R)}{n}
\sum_{t=0}^{k}
\E_{\A}\Big[
\ell\big(y_i f_{\widetilde{\Theta}(t)}(\bx_i)\big)
+
\ell\big(y_i' f_{\widetilde{\Theta}^{i}(t)}(\bx_i')\big)
\Big] +2R\delta .
\end{align*}

Finally, averaging over \(i\in[n]\) and taking expectation over \(S,\widetilde S\), we get
\begin{align*}
\E_{S,\widetilde S,\A}\Big[
\frac{1}{n}\sum_{i=1}^n
\big\|\widetilde{\Theta}(k+1)-\widetilde{\Theta}^{i}(k+1)\big\|_2
\Big]
&\le
\frac{2e\eta C_{\sigma,b}p(\sqrt p+R)}{n}
\sum_{t=0}^{k}
\E_{S,\A}\big[
\mathcal L_S(\widetilde{\Theta}(t))
\big]
+
2R\delta,
\end{align*}
where we used $\E_{S,\widetilde S,\A}
\big[
\mathcal L_{S^i}(\widetilde{\Theta}^{i}(t))
\big]
=
\E_{S,\A}
\big[
\mathcal L_S(\widetilde{\Theta}(t))
\big].$ 
This completes the proof.
\end{proof}

\begin{theorem}\label{thm:gen-dp}
   Suppose \eqref{eq:bound_c0} and Assumptions~\ref{ass:sigma} and~\ref{ass:loss} hold.
   For any reference point $\Theta^*=(\ba^*,\bc^*)$ with $\ba^*\in\Omega_{\ba}$ and $\bc^*\in\Omega_{\bc}$, assume
   $\|\widetilde{\Theta}(0)-\Theta^*\|_2^2
   \ge
   C\max \{
   \eta T\big(\mathcal L_S(\Theta^*)+\mathcal L_{\widetilde S}(\Theta^*)\big),
   \eta\big(\mathcal L_S(\widetilde{\Theta}(0))
   +\mathcal L_{\widetilde S}(\widetilde{\Theta}(0))\big)
    \}.$
   Suppose $\eta\le \min\{1/3\bar\rho,1\}$,
   $m\gtrsim
   p^2 (\log (\frac{m}{\delta} )+R^2  )
   \max\big\{
   R^4+\|\widetilde{\Theta}(0)-\Theta^*\|_2^4,\,
    (
   \frac{Rp^{3/2}\eta T\log(T/\delta)}{n\epsilon}
   )^2
   \big\} $
   and
   $m\lesssim
   \frac{(n\epsilon)^4}
   {p^{10} (\log(\frac{m}{\delta})+R^2 )d^2(\eta T)^4\log^2(T/\delta)}.$ 
   Then, for any $k\in[T]$, it holds that
   \[
\E_{S,\A}\big[
\mathcal L(\widetilde{\Theta}(k))-\mathcal L_S(\widetilde{\Theta}(k))
\big]
\lesssim
\frac{p^2m^{1/4}\|\widetilde{\Theta}(0)-\Theta^*\|_2^2}{n}
+
\frac{mp^6d\eta^2T^2\log(T/\delta)}{n^3\epsilon^2}
+
p^{\frac{3}{2}}\delta .
\]
\end{theorem}
\begin{proof}
We will use Lemma~\ref{lem:dp-stability} to prove the theorem.
The lower-width condition controls the deterministic localization terms and the martingale fluctuation term in the high-probability optimization bound, while the upper-width condition controls the usual DP noise-square term.

Note that the condition for $m$ implies $C_1
=
8(B'_{\ell}B'_\sigma B'_b B_b)^2p^2
\big(4\sqrt p+\frac{2\sqrt{\log(2/\delta)}+R_2}{\sqrt m}\big)^2
\lesssim p^3$ 
and \(C_2=8(B'_\ell B_b)^2p\lesssim p\).
For any \(S\) and the trajectory produced by Algorithm~\ref{alg1}, \eqref{eq:dp-opt} and the comparator condition imply
\begin{align}
\eta\sum_{k=1}^{T}\E_{\A}
\big[\mathcal L_S(\widetilde{\Theta}(k))\big]
&\lesssim
\|\widetilde{\Theta}(0)-\Theta^*\|_2^2
+
\frac{mp^4d\eta^2T^2\log(T/\delta)}{n^2\epsilon^2}.
\label{eq:dp-on-avg-stab}
\end{align}
The same bound holds with \(S\) replaced by \(S^i\), using $\mathcal L_{S^i}(\Theta^*) \le \mathcal L_S(\Theta^*)+\mathcal L_{\widetilde S}(\Theta^*)$
and similarly for \(\widetilde{\Theta}(0)\).

The assumptions on \(m\) ensure that the width condition required in Lemma~\ref{lem:dp-stability} is satisfied. Therefore Lemma~\ref{lem:dp-stability}, together with \eqref{eq:dp-on-avg-stab}, gives the following on-average argument stability bound 
\[
\epsilon_{\rm stab}
\lesssim
\frac{\|\widetilde{\Theta}(0)-\Theta^*\|_2^2}{n}
+
\frac{mp^4d\eta^2T^2\log(T/\delta)}{n^3\epsilon^2}
+
\delta .
\]

From Lemma~\ref{lem:hessian}, the condition on \(m\), and the fact that the iterates stay in the projected set, the loss is Lipschitz with respect to the parameter, it holds 
\[
\sup_{(\bx,y)\in\X\times\Y}
\big|\ell'(yf_{\widetilde{\Theta}(k)}(\bx))\big|
\big\|\nabla f_{\widetilde{\Theta}(k)}(\bx)\big\|_2
\le
C_{\sigma,b}p(\sqrt p+R).
\]
By Lemma~\ref{lem:connection}, we have
\begin{align*}
\E_{S,\A}\big[
\mathcal L(\widetilde{\Theta}(k))-\mathcal L_S(\widetilde{\Theta}(k))
\big]
&\lesssim
p(\sqrt p+R)
\Big(
\frac{\|\widetilde{\Theta}(0)-\Theta^*\|_2^2}{n}
+
\frac{mp^4d\eta^2T^2\log(T/\delta)}{n^3\epsilon^2}
+
R\delta
\Big)  \\
&\lesssim
\frac{p^2m^{1/4}\|\widetilde{\Theta}(0)-\Theta^*\|_2^2}{n}
+
\frac{mp^6d\eta^2T^2\log(T/\delta)}{n^3\epsilon^2}
+
p^{\frac{3}{2}}\delta ,
\end{align*}
where the last inequality uses $m\gtrsim p^2\big(\log (\frac{m}{\delta} )+R^2\big) \big(R^2+\|\widetilde{\Theta}(0)-\Theta^*\|_2^4\big)$,
which implies \(R\lesssim m^{1/4}\) and \(\sqrt p\lesssim m^{1/4}\).
This completes the proof.
\end{proof}

\begin{theorem}[Restatement of Theorem~\ref{thm:utility}]
Let the sequence $\{\widetilde{\Theta}(k)\}_{k=1}^T$ be produced by Algorithm~\ref{alg1} with step size $\eta>0$. 
Let $\Theta^*$ be a reference point satisfying $\|\widetilde{\Theta}(0)-\Theta^*\|_2^2
\ge
C\max \{
\eta T\big(\mathcal L_S(\Theta^*)+\mathcal L_{\widetilde S}(\Theta^*)\big),
\eta\big(\mathcal L_S(\widetilde{\Theta}(0))
+\mathcal L_{\widetilde S}(\widetilde{\Theta}(0))\big)
 \}$. 
If $\eta\le \min\{1/3\bar\rho,1\}$, $m\gtrsim
p^2\big(\log (\frac{m}{\delta} )+R^2\big)
\max\big\{
R^4+\widetilde{\Lambda}_{\Theta^*}^4,\,
\big( \frac{Rp^{3/2}\eta T\log(T/\delta)}{n\epsilon} \big)^2 \big\} $
and $m\lesssim
\frac{(n\epsilon)^4}
{p^{10}d^2(\log(m/\delta)+R^2)(\eta T)^4\log^2(T/\delta)}.$
Then with probability at least $1-\delta$ over the randomness of the initialization, it holds that
\begin{align*}
\frac{1}{T}\sum_{k=1}^T
\E_{S,\A}\big[\mathcal L(\widetilde{\Theta}(k))\big]
\lesssim 
\big(\frac{1}{\eta T}+\frac{m^{1/4}}{n}\big)
\widetilde{\Lambda}_{\Theta^*}^2 +
\big(1+\frac{\eta T}{n}\big) \frac{m\eta T d\log(T/\delta)}{n^2\epsilon^2} + p^{\frac{3}{2}}\delta .
\end{align*}
Furthermore, by setting $\eta T\asymp \frac{c_0n\epsilon}{\sqrt d\,\log^\alpha(n/\delta)}$ for $\alpha>1$ and $c_0\in(0,1]$, and assuming $\delta \lesssim \frac{\sqrt d}{n\epsilon p^{3/2}}$,
it holds that
\[
\frac{1}{T}\sum_{k=1}^T
\E_{S,\A}\big[\mathcal L(\widetilde{\Theta}(k))\big]
\lesssim
\Big(
\frac{\log^\alpha(\frac n\delta)}{c_0}\widetilde{\Lambda}_{\Theta^*}^2
+
\frac{\log^{3\alpha-1}(\frac n\delta)}{c_0^3R^2}
\Big)
\frac{\sqrt d}{n\epsilon}.
\]
\end{theorem}
\begin{proof}
Let \(\mathcal E_{\rm init}\) denote the initialization event in~\eqref{eq:bound_c0}. 
By Corollary~\ref{cor:c}, we have
\[
\mathbb P(\mathcal E_{\rm init})\ge 1-\delta .
\]
We prove the desired bound on \(\mathcal E_{\rm init}\). Conditional on this event, all expectations \(\E_{\A}\) are taken only over the Gaussian perturbations generated after initialization.

On \(\mathcal E_{\rm init}\), combining Theorems~\ref{thm:opt-DP} and~\ref{thm:gen-dp} gives
\begin{align*}
\frac{1}{T}\sum_{k=1}^{T}
\E_{S,\A}\big[\mathcal L(\widetilde{\Theta}(k))\big]
&=
\frac{1}{T}\sum_{k=1}^{T}
\E_{S,\A}\big[
\mathcal L(\widetilde{\Theta}(k))
-
\mathcal L_S(\widetilde{\Theta}(k))
\big]
+
\frac{1}{T}\sum_{k=1}^{T}
\E_{S,\A}\big[
\mathcal L_S(\widetilde{\Theta}(k))
\big] \\
&\lesssim
\Big( \frac{1}{\eta T} + \frac{m^{1/4}}{n} \Big)
\widetilde{\Lambda}_{\Theta^*}^2 +
\Big( 1+\frac{\eta T}{n} \Big) \frac{m\eta T d\log(T/\delta)}{n^2\epsilon^2} + p^{\frac{3}{2}}\delta .
\end{align*}
This proves the first claim on \(\mathcal E_{\rm init}\). Since
\(\mathbb P(\mathcal E_{\rm init})\ge 1-\delta\), the first claim holds with probability at least \(1-\delta\) over the initialization.

Now set $\eta T\asymp \frac{c_0n\epsilon}{\sqrt d\,\log^\alpha(n/\delta)}.$
Then $\frac{1}{\eta T}
\asymp
\frac{\sqrt d\,\log^\alpha(n/\delta)}{c_0n\epsilon}.$
Moreover, the upper-width condition gives
\[
m
\lesssim
\frac{(n\epsilon)^4}
{p^{10}d^2(\log(m/\delta)+R^2)(\eta T)^4\log^2(T/\delta)}
\lesssim
\frac{\log^{4\alpha-2}(n/\delta)}{c_0^4R^2}.
\]
Hence
\[
\frac{m^{1/4}}{n}
\lesssim
\frac{\log^{\alpha-\frac12}(n/\delta)}{c_0 n R^{1/2}}.
\]
Therefore,
\begin{align*}
\frac{1}{T}\sum_{k=1}^{T}
\E_{S,\A}\big[\mathcal L(\widetilde{\Theta}(k))\big]
&\lesssim
\Big(
\frac{\sqrt d\,\log^\alpha(n/\delta)}{c_0n\epsilon}
+
\frac{\log^{\alpha-\frac12}(n/\delta)}{c_0n}
\Big)
\widetilde{\Lambda}_{\Theta^*}^2 +
\frac{\sqrt d\,\log^{3\alpha-1}(n/\delta)}
{c_0^3R^2n\epsilon}
+
p^{\frac{3}{2}}\delta .
\end{align*}
Under the assumption $ \delta
\lesssim
\frac{\sqrt d}{n\epsilon p^{3/2}}$ 
and using \(\alpha>1\) and \(\widetilde{\Lambda}_{\Theta^*}\ge 1\), the right-hand side is bounded by
\[
\Big(
\frac{\log^\alpha(\frac n\delta)}{c_0}\widetilde{\Lambda}_{\Theta^*}^2
+
\frac{\log^{3\alpha-1}(\frac n\delta)}{c_0^3R^2}
\Big)
\frac{\sqrt d}{n\epsilon}.
\]
This completes the proof.
\end{proof}

\subsection{Proofs under NTK Separability}

\begin{theorem}[Restatement of Theorem~\ref{thm:dp-risk-ntk}]
Let Assumptions~\ref{ass:sigma}, \ref{ass:loss} and~\ref{ass:ntk} hold.
Assume $\eta \lesssim \big(\log(n/\delta)\big)^{-1/2}$ be a constant,  $\eta T \asymp \frac{\gamma^2 n\epsilon}{\sqrt d\,\log^{5/2}(n/\delta)}$ and  $m \asymp \frac{\log^6(n/\delta)}{\gamma^6}$. 
Let  $R\asymp \frac{\log^{1/2}(n/\delta)}{\gamma}$ and $\delta
\lesssim \frac{\sqrt d}{n\epsilon p^{3/2}}$.
Then, with probability at least \(1-\delta\) over the initialization \(\bc(0)\), it holds that
\[
    \frac{1}{T}\sum_{k=1}^T
    \E_{S,\A}\big[ \mathcal{L}(\widetilde{\Theta}(k))\big]
    \lesssim
    \log^6 \big(\frac{n}{\delta}\big)
    \frac{\sqrt d}{\gamma^4 n\epsilon}.
\]
\end{theorem}
\begin{proof}
    Note that in the proof of Theorem~\ref{thm:ntk} we showed that with probability at least $1-\delta$ over $\bc(0)$, it holds that $\L_S(\Theta_\tau) \le \frac{1}{T}$ when we set $\Theta_\tau = \Theta(0) + \tau\Theta_0$ and $\tau \asymp \big(\log(T) + \sqrt{\log(n/\delta)}\big)/\gamma$.
    Here, $\Theta_0$ and $\gamma$ are the parameter and the margin in Assumption \ref{ass:ntk}.

    Setting $R_1 + R_2 = R \asymp \tau \asymp \big(\log(T) + \sqrt{\log(n/\delta)}\big)/\gamma$ and the reference point $\Theta^* = \Theta_\tau \in \mathcal{B}(\Theta(0), R)$.
    Then, $\widetilde{\Lambda}_{\Theta^*} = \|\Theta^* - \widetilde{\Theta}(0)\|_2 = \tau$.

    Now, we show that the  conditions in Theorem~\ref{thm:utility} are satisfied.
    Let $c_0 \asymp \gamma^2\in(0,1]$ and $\alpha = 5/2$.
    
    Then, we know $m \gtrsim (\log(\frac{m}{\delta})+R^2)(R^2+  \widetilde{\Lambda}_{\Theta^*}^6) \asymp \log^6(n/\delta)/\gamma^6$ and $m \lesssim \frac{(n\epsilon)^4}{ p^{10 }(\log(\frac{m}{\delta})+R^2)d^2(\eta T)^4\log^2(T/\delta)} \asymp \log^6(n/\delta)/\gamma^6$.
    These matches the width condition stated in the theorem.

Note $\widetilde{S}$ is an independent copy of $S$, one can also show that $\L_{\widetilde{S}}(\Theta^*) \le \frac{1}{T}$. 
   Since  $\L_S(\Theta^*)+\L_{\widetilde S}(\Theta^*)\le \frac{2}{T}$, we have $\eta T\big(\L_S(\Theta^*)+\L_{\widetilde S}(\Theta^*)\big)\lesssim \eta$. 
Moreover, by the initialization bound used in the proof of Theorem~\ref{thm:ntk}, 
\(\L_S(\widetilde\Theta(0))+\L_{\widetilde S}(\widetilde\Theta(0))\) is at most logarithmic in \(n/\delta\). 
Together with \(\eta\lesssim \log^{-1/2}(n/\delta)\), this term is dominated by $\widetilde\Lambda_{\Theta^*}^2
= \|\Theta^*-\widetilde\Theta(0)\|_2^2 \asymp \frac{\log(n/\delta)}{\gamma^2}$. 
Therefore the comparator condition in Theorem~\ref{thm:utility} is satisfied.

Noting that $\delta \lesssim \frac{\sqrt{d}}{n\epsilon p^{3/2}}$.  
Applying Theorem~\ref{thm:utility} with $\alpha=5/2$, $c_0 = \gamma^{2}$ and $R^2\asymp \tau^2 = \widetilde{\Lambda}_{\Theta^*}^2 \gtrsim \log(n/\delta)/\gamma^2$, we know
    \[
\frac{1}{T}\sum_{k=1}^T
\E_{S,\A}\big[\mathcal L(\widetilde\Theta(k))\big]
\lesssim
\log^6\Big(\frac{n}{\delta}\Big)
\frac{\sqrt d}{\gamma^4 n\epsilon} .
\]
This completes the proof of the theorem.
\end{proof}

\section{Detailed Experiments}\label{sec:appen-experiment}
This section provides the detailed DP-GD algorithm and describes the experimental details and hyperparameters.

\subsection{Datasets}
We consider a synthetic binary classification dataset and MNIST \citep{mnist} for our experiments.

\paragraph{Synthetic Logistic Dataset.}
We generate a challenging dataset $\{(x_i, y_i)\}_{i=1}^n$ with $x_i \in [-1,1]^d$ uniformly sampled.
The label distribution is defined through a logistic model
\[
y_i \sim \mathrm{Bernoulli}\bigl(\sigma(s \cdot h(x_i) + \xi_i)\bigr),
\]
where $\sigma(\cdot)$ denotes the sigmoid, $s>0$ is a signal strength parameter, and $\xi_i \sim \mathcal{N}(0, \sigma_\xi^2)$ is optional label noise.
The latent score $h(x)$ is created in a spline-like manner with random coefficients:
\[
h(x_i) = \sum_{j=1}^d u_j(x_{i,j}),
\qquad
u_j(x_{i,j}) = \sum_{\ell=1}^k \theta_{j\ell} \, b_\ell(x_{i,j}).
\]
The basis functions $b_\ell(\cdot)$ are triangular functions centered at uniformly spaced knots $\{t_\ell\}_{\ell=1}^k \subset [-1,1]$,
\[
b_\ell(x) = \max\!\left(1 - \frac{|x - t_\ell|}{\Delta},\, 0\right),
\]
with $\Delta = \frac{2}{k-1}$ denoting the knot spacing and $\theta_{j\ell} \sim \mathcal{N}(0,1)$ being random coefficients.
For our experiments, we set $s = 4$, $d=10$, $\sigma_\xi^2 = 0.1$, and $k=40$.

\paragraph{MNIST.}
We transform MNIST into a binary classification dataset by restricting it to the first two classes (digit zero and digit one). 

\subsection{Model Hyperparameters}

We study the two-layer KAN described in the main paper. Unless otherwise specified, we set the model width to $m=32$, the number of splines to $p=8$, and the number of full-pass iterations to $T=100$. We use a learning rate of $\eta=0.5$ for MNIST and $\eta=1$ for the synthetic logistic data. The larger step size for the logistic data is chosen to speed up the experiments. We observe similar qualitative trends when using smaller step sizes.

For the synthetic logistic data, we use $n=20{,}000$ training samples and $8{,}000$ test samples. For MNIST, we use the full training and test sets restricted to the two classes under consideration. Unless otherwise specified, we set the gradient ball bounds to $R_1 = R_2 = 1$, %the smoothness-related constants to $B'_{\ell} = B'_\sigma = B'_b = B_b = 1$, 
the privacy budget to $\epsilon = 2.0$, and $\delta = 1/n$.

\subsection{Experiments}

The paper shows results for two types of experiments. 

\paragraph{Loss over $m$.}
Here, we vary the model width $m$ and report the resulting training and test metrics of the model for each $50$ random seeds on the synthetic data and $20$ random seeds on MNIST. 

\paragraph{Loss over $T$.}
Here, we vary the full training dataset iterations $T$ and report the resulting training and test metrics of the model for each $50$ random seeds on the synthetic data and $20$ random seeds on MNIST. For MNIST, we use full-batch GD, which incurs a higher computational cost as the width $m$ increases, hence MNIST experiments are restricted to widths up to $m=32$.

% \paragraph{Optimal width $m^*$ over $N$.}
% Here, we vary the training dataset size $n \in [256, 32768]$ and report, for each $n$, the $m^* \in \{4, 8, 16, 24, 32, 48, 64, 96, 128, 256, 512, 1024\}$ that achieved the best test loss for DP-GD and the $m^*$ that plateus (i.e., where higher $m$ did not surpass $m^*$ by more than $10^{-3}$) for GD. The best loss is determined using again each $50$ random seeds.

\subsection{Loss curves}
Figure~\ref{fig:width m} presents loss curves, analogous to the accuracy curves in the main paper.

\begin{figure*}[!t]
\centering
\setlength{\tabcolsep}{4pt}
\renewcommand{\arraystretch}{1}

% ================= Subfigure: GD =================
\begin{subfigure}[t]{0.43\textwidth}
\centering
\captionsetup{margin={5em,0pt}}
\caption{\text{GD}}

\begin{tabular}{@{}c c c c@{}}
% & \multicolumn{1}{c}{\footnotesize Train/test acc. vs. $m$}
% & \multicolumn{1}{c}{\footnotesize Train/test acc. vs. $T$} \\
% [-0em]

\adjustbox{valign=c}{\rotatebox{90}{\small Synthetic} \hspace{.3em}} &
\multirow{2}{*}{\rotatebox{90}{\small Training and test accuracy}} &
\adjustbox{valign=c}{\hspace{-2mm}\includegraphics[width=0.48\linewidth]{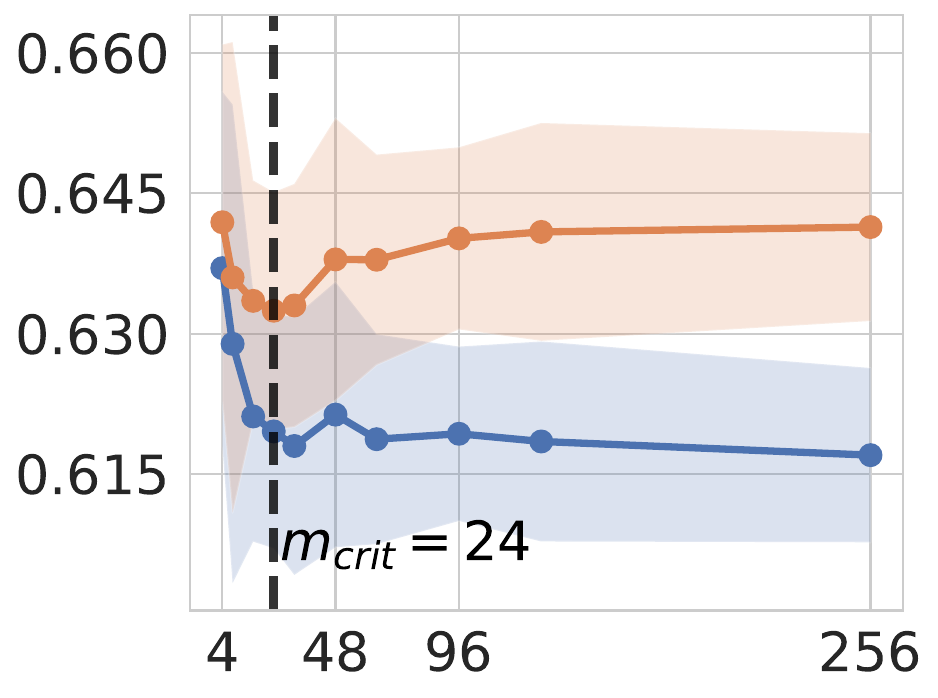}} &
\adjustbox{valign=c}{\includegraphics[width=0.48\linewidth]{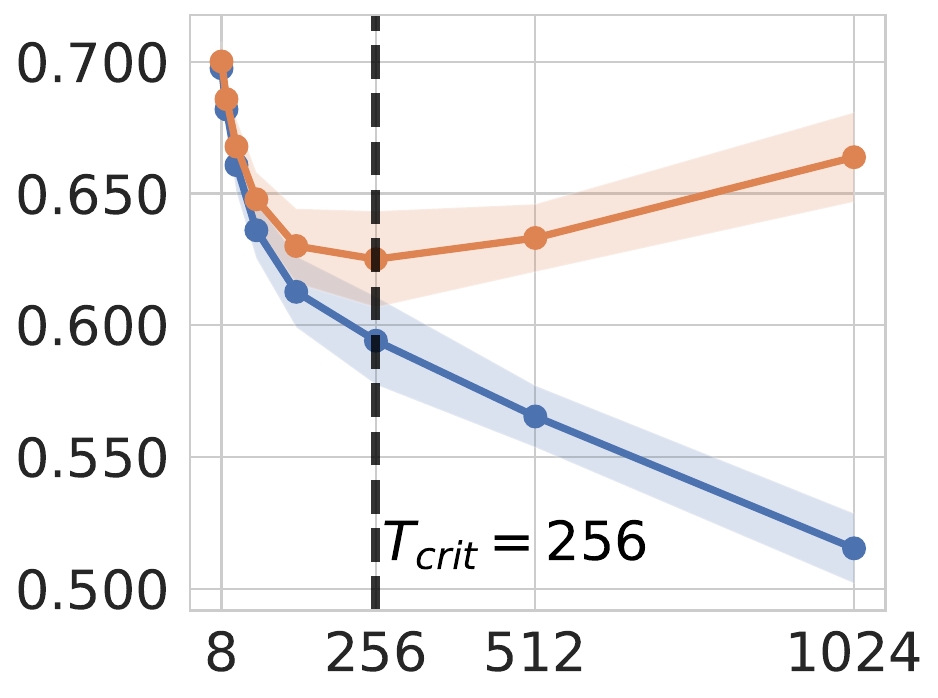}}

\\[4em]

\adjustbox{valign=c}{\rotatebox{90}{\footnotesize MNIST} \hspace{.3em}} &
&
\adjustbox{valign=c}{\hspace{-2mm}\includegraphics[width=0.48\linewidth,height=0.33\linewidth]{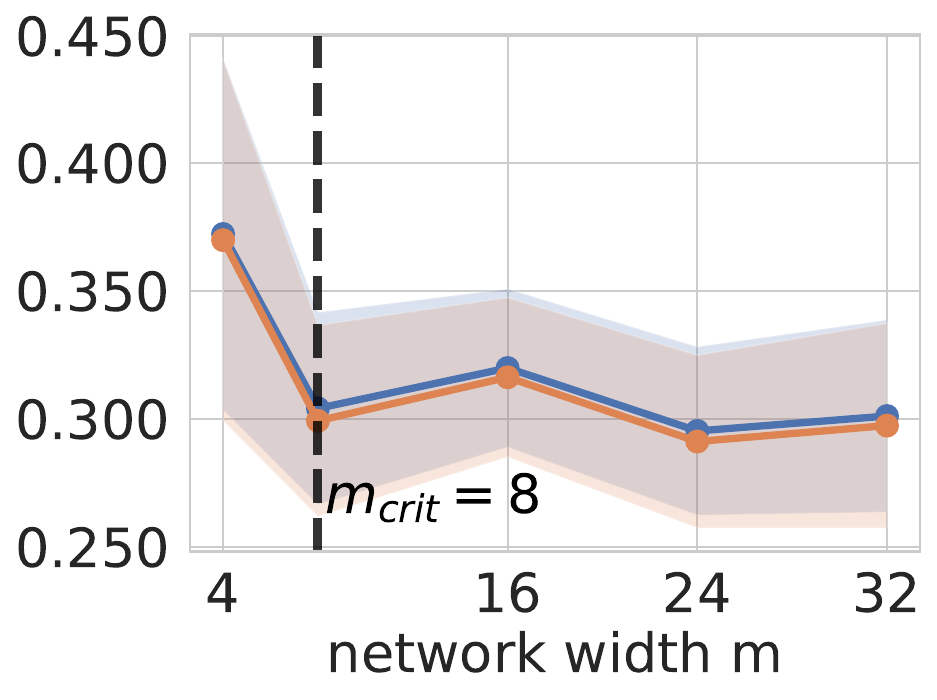}} &
\adjustbox{valign=c}{\includegraphics[width=0.48\linewidth,height=0.33\linewidth]{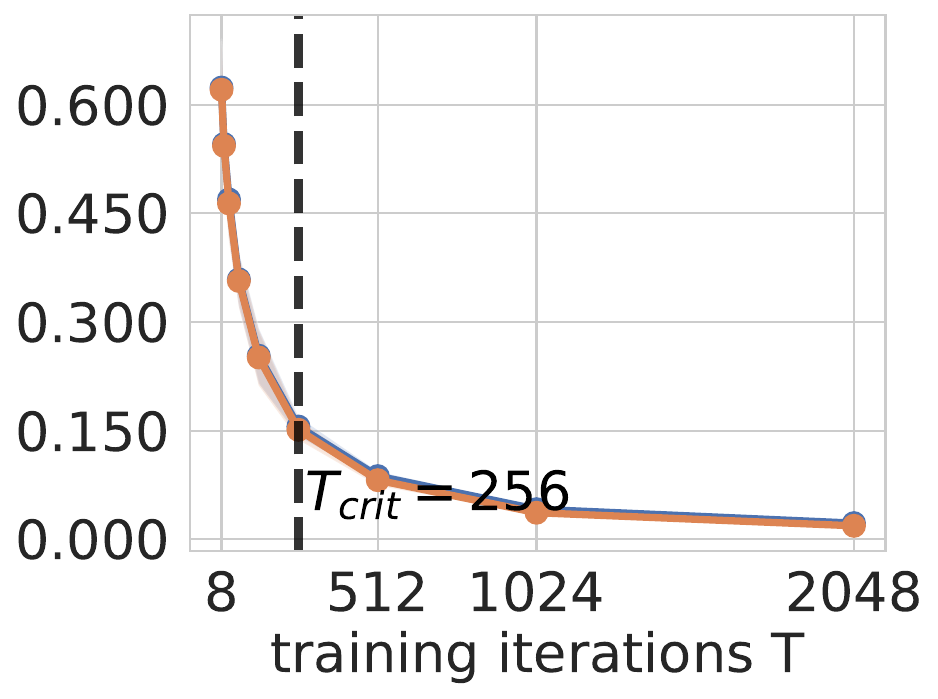}} 

\\[4em]

\multicolumn{4}{c}{\centering \includegraphics[height=1.5em]{figure/gd_legend.pdf}}
\end{tabular}
\end{subfigure}
\hspace{4em}
% ================= Subfigure: DPGD =================
\begin{subfigure}[t]{0.43\textwidth}
\centering
\caption{\text{DP-GD}}

\begin{tabular}{@{}c c c@{}}
%  \multicolumn{1}{c}{\footnotesize Utility vs. $m$}
% & \multicolumn{1}{c}{\footnotesize Utility vs. $T$} \\
% [-0em]

 \multirow{3}{*}{\rotatebox{90}{\small \hspace{-5.5em} Private utility}} \hspace{-1em}
 &
\adjustbox{valign=c}{\includegraphics[width=0.48\linewidth]{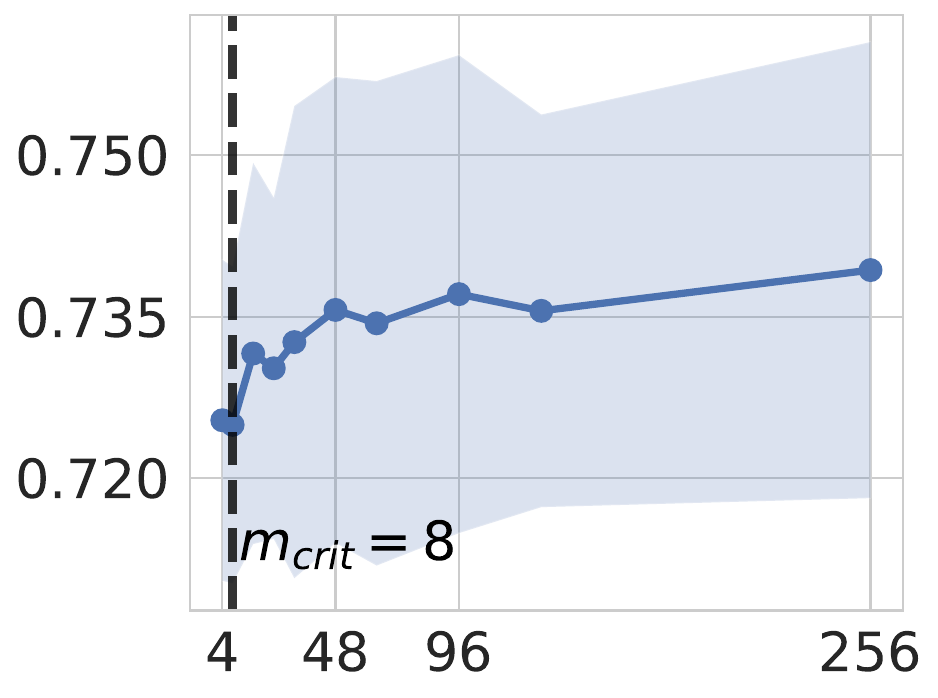}} &
\adjustbox{valign=c}{\includegraphics[width=0.48\linewidth]{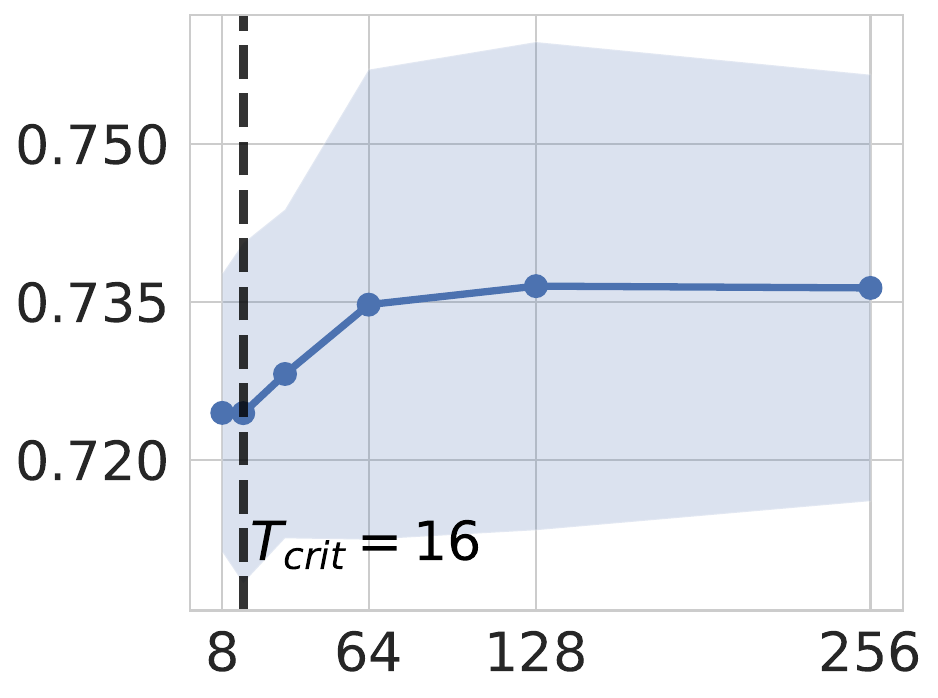}}
\\[4em]

 &
\adjustbox{valign=c}{\includegraphics[width=0.48\linewidth,height=0.33\linewidth]{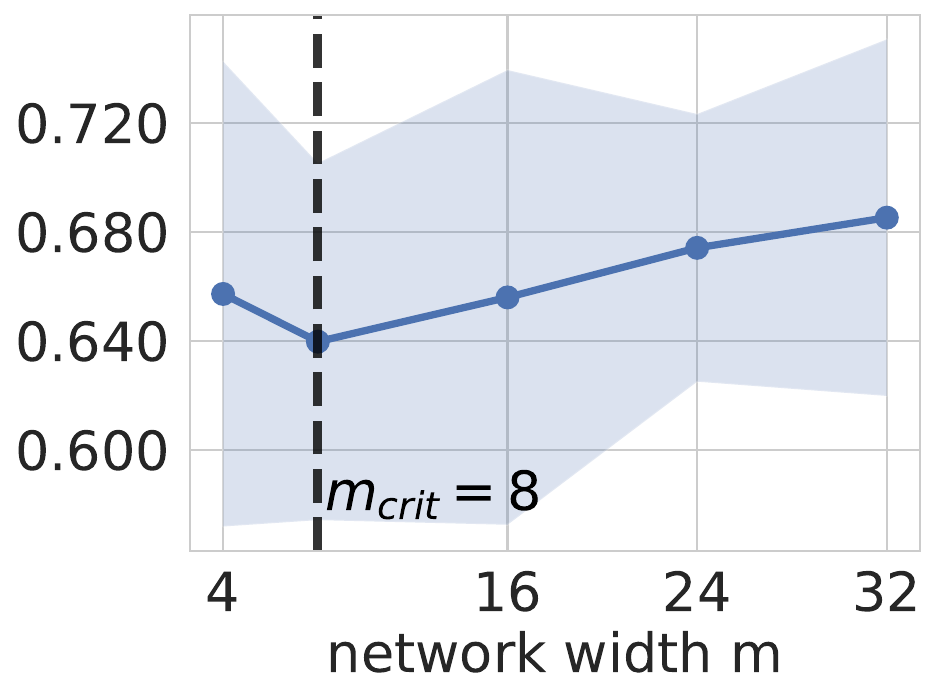}} &
\adjustbox{valign=c}{\includegraphics[width=0.48\linewidth,height=0.33\linewidth]{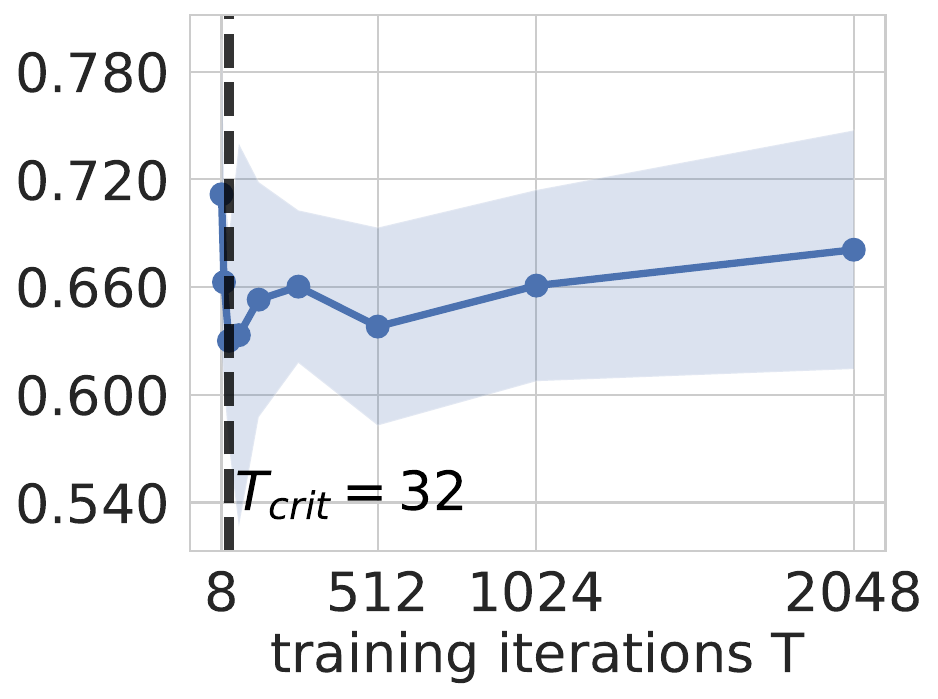}}

\\[4em]

\multicolumn{3}{c}{\centering \includegraphics[height=1.5em]{figure/dpgd_legend.pdf}}
\end{tabular}
\end{subfigure}

\vspace{-1.5mm}
\caption{\footnotesize Training and test losses versus $m$ and $T$.}
\label{fig:width m}
\vspace{-3mm}
\end{figure*}

\end{document}